\documentclass{article} 

\usepackage[final,nonatbib]{nips_2016}

\usepackage[utf8]{inputenc} 
\usepackage[T1]{fontenc}    
\usepackage{url}            
\usepackage{booktabs}       
\usepackage{amsfonts}       
\usepackage{nicefrac}       
\usepackage{microtype}      

\usepackage{color}
\usepackage{amsmath}
\usepackage{amsfonts}
\usepackage{amsthm}
\usepackage{amssymb}
\usepackage{dsfont}
\usepackage{algorithm,algorithmic}

\usepackage{graphicx} 
\usepackage{subfigure}
\usepackage[english]{babel}  

\DeclareMathOperator*{\argmin}{arg\,min}
\DeclareMathOperator*{\argmax}{arg\,max}

\newcommand{\bs}[1]{\boldsymbol{#1}}
\allowdisplaybreaks[2]
\newtheorem{theorem}{Theorem}
\newtheorem{prop}{Proposition}
\newtheorem{lemma}{Lemma}
\newtheorem{corol}{Corollary}

\newcommand{\dts}{\text{\rm D-TS}}
\newcommand{\dtsp}{\text{\rm D-TS$^+$}}
\newcommand{\afterthm}{\vspace{-.3cm}}
\newcommand{\aftersec}{\vspace{-.1cm}}

\title{Double Thompson Sampling for Dueling Bandits}

\author{
Huasen Wu \\
University of California, Davis\\
\texttt{hswu@ucdavis.edu}
\And
Xin Liu \\
University of California, Davis\\
\texttt{xinliu@ucdavis.edu}
}

\begin{document}

\maketitle
\vspace{-.6cm}
\begin{abstract}
In this paper, we propose a Double Thompson Sampling (D-TS) algorithm for dueling bandit problems. As its name suggests, D-TS selects both the first and the second candidates according to Thompson Sampling.
Specifically, D-TS maintains a posterior distribution for the preference matrix, and chooses the pair of arms for comparison according to two sets of samples independently drawn from the posterior distribution.
This simple algorithm applies to general \emph{Copeland dueling bandits}, including \emph{Condorcet dueling bandits} as a special case.  For general Copeland dueling bandits, we show that D-TS achieves $O(K^2 \log T)$ regret. Moreover, using a back substitution argument, we refine the regret to $O(K \log T + K^2 \log \log T)$ in Condorcet dueling bandits and most practical Copeland dueling bandits. In addition, we propose an enhancement of D-TS, referred to as D-TS$^+$, to reduce the regret in practice by carefully breaking ties.   Experiments based on both synthetic and real-world data demonstrate that D-TS and D-TS$^+$ significantly improve the overall performance, in terms of regret and robustness.
\end{abstract}


\section{Introduction} \label{sec:intro}
\aftersec
The dueling bandit problem \cite{Yue2012JCSS:DuelingBandits} is a variant of the classical multi-armed bandit (MAB) problem, where the feedback comes in the form of pairwise comparison. This model has attracted much attention as it can be applied in many systems such as information retrieval (IR) \cite{Yue2009ICML:DuelingBandits,Zoghi2014WSDM:RCS}, where user preferences are  easier to obtain and typically more stable. Most earlier work \cite{Yue2012JCSS:DuelingBandits,Zoghi2014ICML:RUCB,Komiyama2015COLT:DB} focuses on Condorcet dueling bandits, where there exists an arm, referred to as the Condorcet winner, that beats all other arms.
Recent work \cite{Zoghi2015NIPS:CDB,Komiyama2016ICML:CWRMED} turns to a more general and practical case of a Copeland winner(s), which is the arm (or arms) that beats the most other arms.
Existing algorithms are mainly generalized from traditional MAB algorithms along two lines: 1)  UCB (Upper Confidence Bound)-type algorithms, such as RUCB \cite{Zoghi2014ICML:RUCB} and CCB \cite{Zoghi2015NIPS:CDB}; and, 2) MED (Minimum Empirical Divergence)-type algorithms, such as RMED \cite{Komiyama2015COLT:DB} and CW-RMED/ECW-RMED \cite{Komiyama2016ICML:CWRMED}.

In traditional MAB, an alternative effective solution is Thompson Sampling  (TS) \cite{Thompson1933FirstMAB}. Its principle is to choose the optimal action that maximizes the expected reward according to the randomly drawn belief.
TS has been successfully applied in traditional MAB \cite{Chapelle2011NIPS:TS,Agrawal2012COLT:TS,Komiyama2015ICML:MP_TS,Qin&Liu2015IJCAI:TS} and other online learning problems \cite{Gopalan2014ICML:TS,Gopalan2015COLT:TSMDP}.
In particular, empirical studies  in \cite{Chapelle2011NIPS:TS} show that TS not only achieves lower regret than other algorithms in practice, but is also more robust as a randomized algorithm.

In the wake of the success of TS in these online learning problems, a natural question  is  whether  and how TS can be applied to dueling bandits to further improve the performance. However, it is challenging to apply the standard TS framework to dueling bandits, because not all comparisons provide information about the system statistics. Specifically, a good learning algorithm for dueling bandits will eventually compare the winner against itself.  However, comparing one arm against itself does not provide any statistical information, which is critical in TS to update the posterior distribution.
Thus, TS needs to be adjusted so that 1) comparing the winners against themselves is allowed, but, 2) trapping in comparing a non-winner arm against itself is avoided.

In this paper, we propose a Double Thompson Sampling (D-TS) algorithm for dueling bandits, including both Condorcet dueling bandits and general Copeland dueling bandits. As its name suggests,  D-TS typically selects both the first and the second candidates according to samples independently drawn from the posterior distribution.  D-TS also utilizes the idea of
confidence bounds to eliminate the likely non-winner arms, and thus avoids trapping in suboptimal comparisons.
Compared to prior studies on dueling bandits, D-TS has both practical and theoretical advantages.

First, the double sampling structure of D-TS better suits the nature of dueling bandits. Launching two independent rounds of sampling provides us the opportunity to select the same arm in both rounds and thus to compare the winners against themselves. This double sampling structure also leads to more extensive utilization of TS (e.g., compared to  RCS \cite{Zoghi2014WSDM:RCS}), and significantly reduces the regret. In addition, this simple framework applies to general Copeland dueling bandits and achieves lower regret than existing algorithms such as CCB \cite{Zoghi2015NIPS:CDB}. Moreover, as a randomized algorithm, D-TS is more robust in practice.

Second, this double sampling structure enables us to obtain theoretical bounds for the regret of D-TS.  As noted in traditional MAB literature \cite{Agrawal2012COLT:TS,Agrawal2013AISTATS:TS2}, theoretical analysis of TS is usually more difficult than UCB-type algorithms.
The analysis in dueling bandits is even more challenging because the selection of arms involves more factors and the two selected arms may be correlated.
To address this issue, our D-TS algorithm  draws  the two sets of samples independently.
Because their distributions are fully captured by historic comparison results,  when the first candidate is fixed, the comparison between it and all other arms is similar to traditional MAB and thus we can borrow ideas from traditional MAB.  Using the properties of TS and confidence bounds, we show that  D-TS achieves $O(K^2 \log T)$ regret
for a general $K$-armed Copeland dueling bandit. More interestingly, the property that the sample distribution only depends on historic comparing results (but not $t$)  enables us to refine the regret using  a \emph{back substitution} argument, where we show that D-TS achieves $O(K \log T + K^2 \log \log T)$ in Condorcet dueling bandits and many practical Copeland dueling bandits.

Based on the analysis, we further refine the tie-breaking criterion in D-TS and propose its enhancement called D-TS$^+$. D-TS$^+$ achieves the same theoretical regret bound as D-TS, but performs better in practice especially when there are multiple winners.


In summary, the main contributions of this paper are as follows:
\begin{itemize}
\item We propose a D-TS algorithm and its enhancement D-TS$^+$ for general Copeland dueling bandits. The double sampling structure suits the nature of dueling bandits and leads to more extensive usage of TS, which significantly reduces the regret.
\item We obtain theoretical regret bounds for D-TS and D-TS$^+$. For general Copeland dueling bandits, we show that D-TS and D-TS$^+$ achieve $O(K^2 \log T)$ regret. In Condorcet dueling bandits and most practical Copeland dueling bandits, we further refine the regret bound to $O(K \log T + K^2 \log \log T)$ using a back substitution argument.
\item We evaluate the D-TS and D-TS$^+$ algorithms through experiments based on both synthetic and real-world data. The results show that D-TS and D-TS$^+$ significantly improve the overall performance, in terms of regret and robustness, compared to existing algorithms.
\end{itemize}
\aftersec
\aftersec
\section{Related Work}
\aftersec
Early dueling bandit algorithms study finite-horizon settings, using the ``explore-then-exploit'' approaches,
such as IF \cite{Yue2012JCSS:DuelingBandits}, BTM \cite{Yue2011ICML:BTM}, and SAVAGE \cite{Urvoy2013ICML:SAVAGE}.
For infinite horizon settings, recent work has generalized the traditional MAB algorithms to dueling bandits along two lines. First,  RUCB \cite{Zoghi2014ICML:RUCB} and CCB \cite{Zoghi2015NIPS:CDB} are generalizations of UCB for Condorcet and general Copeland dueling bandits, respectively. In addition, \cite{Ailon2014ICML:UBDB} reduces dueling bandits to traditional MAB, which is then solved by UCB-type algorithms, called MutiSBM and Sparring.
Second, \cite{Komiyama2015COLT:DB} and \cite{Komiyama2016ICML:CWRMED} extend the MED algorithm to dueling bandits, where they present the lower bound on the regret and propose the corresponding optimal algorithms, including RMED for Condorcet dueling bandits \cite{Komiyama2015COLT:DB}, CW-RMED  and its computationally efficient version ECW-RMED for general Copeland dueling bandits \cite{Komiyama2016ICML:CWRMED}. Different from such existing work, we study algorithms for dueling bandits from the perspective of TS, which typically achieves lower regret and is more robust in practice.

Dated back to 1933, TS \cite{Thompson1933FirstMAB} is one of the earliest algorithms for exploration/exploitation tradeoff.  Nowadays, it has been applied in many variants of MAB \cite{Komiyama2015ICML:MP_TS,Qin&Liu2015IJCAI:TS,Gopalan2014ICML:TS} and other more complex problems, e.g., \cite{Gopalan2015COLT:TSMDP}, due to its simplicity, good performance, and robustness  \cite{Chapelle2011NIPS:TS}.
Theoretical analysis of TS is much more difficult. Only recently, \cite{Agrawal2012COLT:TS} proposes a logarithmic bound for the standard frequentist expected regret, whose constant factor is further improved in \cite{Agrawal2013AISTATS:TS2}.  Moreover  \cite{Russo2014TR:TS,Russo2014MOR:TS} derive the bounds for its Bayesian expected regret through information-theoretic analysis.

TS has been preliminarily considered for dueling bandits \cite{Zoghi2014WSDM:RCS,Welsh2012LSOLDM:TS}. In particular, recent work \cite{Zoghi2014WSDM:RCS} proposes a Relative Confidence Sampling (RCS) algorithm that combines TS with RUCB \cite{Zoghi2014ICML:RUCB}  for Condorcet dueling bandits. Under RCS, the first arm is selected by TS while the second arm is selected according to their RUCB. Empirical studies demonstrate the performance improvement of using RCS in practice, but no theoretical bounds on the regret are provided.


\section{System Model}

We consider a dueling bandit problem with $K$ ($K \geq 2$) arms, denoted by $\mathcal{A} = \{1, 2, \ldots, K\}$.
At each time-slot $t > 0$, a pair of arms {\small $(a_t^{(1)}, a_t^{(2)})$} is displayed to a user and a noisy comparison outcome $w_t$ is obtained, where $w_t = 1$ if the user prefers {\small $a_t^{(1)}$} to {\small $a_t^{(2)}$}, and $w_t = 2$ otherwise.  
We assume the user preference is stationary over time and the distribution of  comparison outcomes is characterized by the preference matrix  $\bs{P} = [p_{ij}]_{K \times K}$, where $p_{ij}$ is the probability that the user prefers arm $i$ to arm $j$, i.e.,
$p_{ij} = \mathbb{P} \{ i \succ j\},~i,j = 1,2, \ldots, K$.
We assume that the displaying order does not affect the preference, and hence, $p_{ij} + p_{ji} = 1$ and $p_{ii} = 1/2$. We say that arm~$i$ beats arm~$j$ if $p_{ij} > 1/2$.

We study the general Copeland dueling bandits, where the Copeland winner is defined as the arm (or arms) that maximizes the number of other arms it beats \cite{Zoghi2015NIPS:CDB,Komiyama2016ICML:CWRMED}. Specifically, the Copeland score is defined as {\small $\sum_{j \neq i } \mathds{1}(p_{ij} > 1/2)$}, and  the normalized Copeland score is defined as
{\small $\zeta_i = \frac{1}{K - 1}\sum_{j \neq i } \mathds{1}(p_{ij} > 1/2)$},
where $\mathds{1}(\cdot)$ is the indicator function. Let $\zeta^*$ be the highest normalized Copeland score, i.e.,
$\zeta^*  = {\max}_{\small 1 \leq i \leq K}~ \zeta_i$.
Then the Copeland winner is defined as the arm (or arms) with the highest normalized Copeland score, i.e.,
$\mathcal{C}^* = \{i: 1 \leq i \leq K, \zeta_i = \zeta^*\}$.
Note that the Condorcet winner is a special case of Copeland winner with $\zeta^* = 1$.

A dueling bandit algorithm $\Gamma$ decides which pair of arms to compare depending on the historic observations. Specifically, define a filtration $\mathcal{H}_{t-1}$ as the history before $t$, i.e.,
{\small $\mathcal{H}_{t-1} = \{a_{\tau}^{(1)}, a_{\tau}^{(2)}, w_{\tau}, \tau = 1, 2, \ldots, t-1\}$}.
Then a dueling bandit algorithm $\Gamma$ is a function that maps {\small $\mathcal{H}_{t-1} $ to $(a_t^{(1)}, a_t^{(2)})$}, i.e.,
{\small
$
(a_t^{(1)}, a_t^{(2)}) = \Gamma(\mathcal{H}_{t-1}).
$}
The performance of a dueling bandit algorithm $\Gamma$ is measured by its expected cumulative regret, which is defined as
\begin{equation} \label{eq:def_regret}
{\small R_{\Gamma}(T) =  \zeta^* T  - \frac{1}{2} \sum_{t = 1}^T \mathbb{E} \big[ \zeta_{a_t^{(1)}} + \zeta_{a_t^{(2)}}\big].}
\end{equation}
The objective of $\Gamma$ is then to minimize $R_{\Gamma}(T)$.
As pointed out in \cite{Zoghi2015NIPS:CDB}, the results can be adapted to other regret definitions because the above definition bounds the number of suboptimal comparisons.

\section{Double Thompson Sampling} \label{sec:algorithm}


\subsection{D-TS Algorithm}
We present the D-TS algorithm for Copeland dueling bandits, as described in Algorithm~\ref{alg:dts_Copeland} (time index $t$ is omitted in pseudo codes for brevity). As its name suggests, the basic idea of D-TS is to select both the first and the second candidates by TS. For each pair $(i,j)$ with $i\neq j$, we assume a beta prior distribution for its preference probability $p_{ij}$. These distributions are updated according to the comparison results $B_{ij}(t-1)$ and $B_{ji}(t-1)$, where $B_{ij}(t-1)$ (resp.~$B_{ji}(t-1))$ is the number of time-slots when arm $i$ (resp.~$j$)  beats arm $j$ (resp.~$i$)  before $t$. D-TS selects the two candidates by sampling  from the posterior distributions.

\begin{algorithm}[htbp]
\caption{D-TS for Copeland Dueling Bandits}
\label{alg:dts_Copeland}
\renewcommand{\algorithmicrequire}{\textbf{Input:}}
\renewcommand{\algorithmicensure}{\textbf{Output:}}
\renewcommand\algorithmiccomment[1]{%
{//{\it ~{#1}}}%
}
\begin{algorithmic}[1]
\STATE{\bfseries Init:} $\bs{B} \gets {\bs{0}_{K \times K}}$; \COMMENT{$B_{i j}$ is the number of time-slots that the user prefers arm $i$ to $j$.}
\FOR{$t = 1$ {\bfseries to} $T$}
\STATE{\COMMENT{Phase 1: Choose the first candidate $a^{(1)}$} }     \label{alg_line:first_candidate_start}
     \STATE{$\bs{U} := [u_{ij}]$,   $\bs{L} := [l_{ij}]$, where $u_{ij}= \frac{B_{ij}}{B_{ij} + B_{ji}}+ \sqrt{\frac{\alpha \log t}{B_{ij} + B_{ji}}}$, $l_{ij} =  \frac{B_{ij}}{B_{ij}+ B_{ji}} - \sqrt{\frac{\alpha \log t}{B_{ij}+ B_{ji}}}$, if $i \neq j$, and $u_{ii} = l_{ii} = 1/2$, $\forall i$;}   \COMMENT{$\frac{x}{0} := 1$ for any $x$.}
    \STATE{$\hat{\zeta}_i \gets \frac{1}{K - 1}\sum_{j \neq i} \mathds{1} (u_{ij}> 1/2)$;} \COMMENT{Upper bound of the normalized Copeland score.}
    \STATE{$\mathcal{C} \gets \{ i: \hat{\zeta}_i = \max_{j} \hat{\zeta}_j\}$;}
    \FOR{$i, j = 1, \ldots, K$ with $i < j$ }
             \STATE{Sample $\theta^{(1)}_{ij} \sim {\rm Beta}(B_{ij} + 1, B_{ji}+ 1)$};
             \STATE{$\theta^{(1)}_{ji} \gets 1 - \theta^{(1)}_{ij}$};
    \ENDFOR
   \STATE $a^{(1)} \gets \underset{i \in \mathcal{C} }{\argmax} \sum_{j\neq i} \mathds{1}(\theta^{(1)}_{ij} > 1/2)$;    \COMMENT{Choosing from $\mathcal{C}$ to eliminate likely non-winner arms; Ties are broken randomly.}
\STATE{\COMMENT{Phase 2: Choose the second candidate $a^{(2)}$}}    \label{alg_line:sec_candidate_start}
      \STATE{Sample $\theta^{(2)}_{ i a^{(1)} } \sim {\rm Beta}(B_{i a^{(1)} } + 1, B_{a^{(1)}  i} +1)$ for all $i \neq a^{(1)}$, and let $\theta^{(2)}_{a^{(1)}  a^{(1)} } = 1/2$;  }
      \STATE{$a^{(2)}  \gets \underset{i: l_{i a^{(1)}} \leq 1/2}{\argmax}~ \theta^{(2)}_{i a^{(1)} }$;}
      \COMMENT{Choosing only from uncertain pairs.}
      \label{alg_line:sec_candidate_end}
\STATE{\COMMENT{Compare and Update}}
\STATE{Compare pair $(a^{(1)} , a^{(2)} )$ and observe the result $w$;}
\STATE{Update $\bs{B}$: $B_{a^{(1)} a^{(2)}}\gets B_{a^{(1)} a^{(2)}}+ 1$ if $w = 1$, or  $B_{a^{(2)} a^{(1)}} \gets B_{a^{(2)}a^{(1)}} + 1 $ if $w = 2$;}
\ENDFOR
\end{algorithmic}
\end{algorithm}

Specifically, at each time-slot $t$, the D-TS algorithm consists of two phases that select the first and the second candidates, respectively. When choosing the first candidate {\small $a_{t}^{(1)}$}, we first use the RUCB \cite{Zoghi2014ICML:RUCB} of $p_{ij}$ to eliminate the arms that are unlikely to be the Copeland winner, resulting in a candidate set $\mathcal{C}_t$ (Lines~4 to~6). The algorithm then samples {\small $\theta^{(1)}_{ij}(t)$} from the posterior beta distribution, and the first candidate {\small $a_{t}^{(1)}$} is chosen by ``majority voting'', i.e., the arm within $\mathcal{C}_t$ that beats the most arms according to  {\small $\theta^{(1)}_{ij}(t)$} will be selected (Lines~7 to 11). The ties are broken randomly here for simplicity and will be refined later in Section~\ref{subsec:dts_plus}. A similar idea is applied to select the second candidate {\small $a_{t}^{(2)}$}, where new samples {\small $\theta^{(2)}_{i a_{t}^{(1)}}(t)$} are generated and the arm with the largest {\small $\theta^{(2)}_{i a_{t}^{(1)}}(t)$} among all arms with {\small $l_{i a_t^{(1)}} \leq 1/2$}  is selected as the second candidate  (Lines~13 to 14).

The double sampling structure of D-TS is designed based on the nature of dueling bandits,
i.e., at each time-slot, two arms are needed for comparison.
Unlike RCS \cite{Zoghi2014WSDM:RCS}, D-TS selects both candidates using TS. This leads to more extensive utilization of TS and thus achieves much lower regret. Moreover, the two sets of samples are independently distributed, following the same posterior that is only determined by the comparison statistics $B_{ij}(t-1)$ and $B_{ji}(t-1)$. This property enables us to obtain an $O(K^2 \log T)$ regret bound and further refine it by a back substitution argument, as discussed later.

We also note that RUCB-based elimination (Lines~4 to~6) and RLCB (Relative Lower Confidence Bound)-based elimination (Line~14) are essential in D-TS. Without these eliminations, the algorithm may trap in suboptimal comparisons. Consider one extreme case in Condorcet dueling bandits\footnote{A Borda winner may be more appropriate in this special case \cite{Jamieson2015AISTAT:SDB}, and we mainly use it to illustrate the dilemma.}: assume arm $1$ is the Condorcet winner with $p_{1j} = 0.501$ for all $j > 1$, and arm 2 is not the Condorcet winner, but with $p_{2j} = 1$ for all $j > 2$. Then for a larger $K$ (e.g., $K > 4$), without RUCB-based elimination, the algorithm may trap in {\small $a_t^{(1)} = 2$} for a long time, because arm 2 is likely to receive higher score than arm 1.
This issue can be addressed by RUCB-based elimination as follows: when chosen as the first candidate, arm 2 has a great probability to compare with arm 1; after sufficient comparisons with arm 1, arm 2 will have $u_{21}(t) < 1/2$ with high probability; then arm 2 is likely to be eliminated because arm 1 has  {\small $\hat{\zeta}_1(t) = 1 > \hat{\zeta}_2(t)$} with high probability. Similarly, RLCB-based elimination (Line 14, where we restrict to the arms with {\small $l_{i a_t^{(1)}}(t) \leq 1/2$}) is important especially for non-Condorcet dueling bandits. Specifically,  {\small $l_{i a_t^{(1)}}(t) > 1/2$} indicates that arm $i$ beats {\small $a_t^{(1)}$} with high probability. Thus, comparing {\small $a_t^{(1)}$} and arm $i$ brings little information gain and thus should be eliminated to minimize the regret.

%

\subsection{Regret Analysis}


Before conducting the regret analysis, we first introduce certain notations that will be used later.

\textbf{Gap to 1/2:} In dueling bandits, an important benchmark for $p_{ij}$ is 1/2, and thus we let $\Delta_{ij}$ be the gap between $p_{ij}$ and 1/2, i.e.,
$
\Delta_{ij} = |p_{ij} - 1/2|
$.

\textbf{Number of Comparisons:}
Under D-TS, $(i, j)$ can be compared in the form of $(a_t^{(1)}, a_t^{(2)}) = (i, j)$ and $(a_t^{(1)}, a_t^{(2)}) = (j, i)$. We consider these two cases separately and define the following counters:
$
{\small N^{(1)}_{ij}(t) = \sum_{\tau =1 }^{t}\mathds{1} (a_{\tau}^{(1)} = i, a_{\tau}^{(2)} = j)}
$
and
$
{\small N^{(2)}_{ij}(t) = \sum_{\tau =1}^{t}\mathds{1} (a_{\tau}^{(1)} = j, a_{\tau}^{(2)} = i)}
$.
Then the total number of comparisons is
$
{\small N_{ij}(t) = N^{(1)}_{ij}(t)  + N^{(2)}_{ij}(t)}
$ for $i \neq j$, and
$
{\small N_{ii}(t) = N^{(1)}_{ii}(t) = N^{(2)}_{ii}(t)}
$ for $i = j$.


\subsubsection{$O(K^2 \log T)$ Regret}
To obtain theoretical bounds for the regret of D-TS, we make the following assumption:

\textbf{Assumption 1:} The preference probability $p_{ij}\neq 1/2$ for any $i \neq j$.

Under Assumption~1, we present the first result  for D-TS in general Copeland dueling bandits:
\begin{prop} \label{thm:regret_copeland}
When applying D-TS with $\alpha > 0.5$ in a Copeland dueling bandit with a preference matrix $P = [p_{ij}]_{K \times K}$ satisfying Assumption~1, its regret is bounded as:
\begin{equation}
R_{\text{\rm D-TS}}(T) \leq \sum_{i\neq j: p_{ij} <  1/2}\bigg[\frac{4\alpha\log T}{\Delta_{ij}^2} +  (1+ \epsilon) \frac{\log T}{D(p_{ij}||1/2)}\bigg]  +O(\frac{K^2}{\epsilon^2}), \label{eq:regret_copeland}
\end{equation}
where $\epsilon > 0$ is an arbitrary constant, and $D(p || q) = p \log \frac{p}{q} + (1 - p) \log \frac{1-p}{1 - q}$ is the KL divergence.
\end{prop}
The summation operation in Eq.~\eqref{eq:regret_copeland}  is conducted over all pairs $(i,j)$ with $p_{ij} < 1/2$. Thus, Proposition~\ref{thm:regret_copeland} states that D-TS achieves $O(K^2 \log T)$ regret in Copeland dueling bandits. To the best of our knowledge, this is the first theoretical bound for TS in dueling bandits. The scaling behavior of this bound with respect to $T$ is order optimal, since a lower bound $\Omega(\log T)$ has been shown in \cite{Komiyama2016ICML:CWRMED}.
The refinement of the scaling behavior with respect to $K$ will be discussed later.

Proving Proposition~\ref{thm:regret_copeland} needs to  bound the number of comparisons for all pairs $(i,j)$ with $i \notin \mathcal{C}^*$ or $j \notin \mathcal{C}^*$.
When fixing the first candidate as {\small $a_t^{(1)} = i$}, the selection of the second candidate {\small $a_t^{(2)}$} is similar to a traditional $K$-armed bandit problem with expected utilities $p_{ji}$ ($j = 1,2, \ldots, K$).  However, the analysis is more complex here since different arms are eliminated differently depending on the value of $p_{ji}$.
We prove Proposition~\ref{thm:regret_copeland} through Lemmas~\ref{thm:cmp_bounds_less_0p5} to~\ref{thm:cmp_bounds_eq_0p5}, which  bound the number of comparisons for all suboptimal pairs $(i,j)$  under different scenarios, i.e., $p_{ji} < 1/2$, $p_{ji} > 1/2$, and $p_{ji} = 1/2$ ($j = i \notin \mathcal{C}^*$), respectively.

\begin{lemma} \label{thm:cmp_bounds_less_0p5}
Under D-TS, for an arbitrary constant $\epsilon>0$ and one pair $(i,j)$ with {\small $p_{ji}<1/2$}, we have
\begin{equation}
\mathbb{E}[N^{(1)}_{ij}(T)] \leq (1+ \epsilon) \frac{\log T}{D(p_{ji}||1/2)} + O(\frac{1}{\epsilon^2}).
\end{equation}
\end{lemma}
\afterthm
\begin{proof} We can prove this lemma by viewing the comparison between the first candidate arm $i$ and its inferiors as a traditional MAB. In fact, it may be even simpler than that in \cite{Agrawal2013AISTATS:TS2} because under D-TS, arm $j$ with $p_{ji} < 1/2$ is competing with arm $i$ with $p_{ii} = 1/2$, which is known and fixed. Then we can bound $\mathbb{E}[N^{(1)}_{ij}(T)]$ using the techniques in
\cite{Agrawal2013AISTATS:TS2}. Details can be found in  Appendix~\ref{app:proof_cmp_bounds_less_0p5}.
\end{proof}

\begin{lemma} \label{thm:cmp_bounds_larger_0p5}
Under D-TS with $\alpha > 0.5$, for one pair $(i, j)$ with $p_{ji} >1/2$, we have
\begin{equation}
\mathbb{E}[N^{(1)}_{ij}(T)] \leq \frac{4 \alpha \log T}{\Delta_{ji}^2} + O(1).
\end{equation}
\end{lemma}
\afterthm
\begin{proof}
We note that when $a_t^{(1)} = i$, arm $j$ can be selected as $a_t^{(2)}$ only when its RLCB $l_{ji}(t) \leq 1/2$. Then we can bound $\mathbb{E}[N^{(1)}_{ij}(T)]$ by $O(\frac{4\alpha \log T}{\Delta_{ji}^2})$ similarly to the analysis of traditional UCB algorithms \cite{Bubeck2010PhD:bandits}. Details can be found in Appendix~\ref{app:proof_cmp_bounds_larger_0p5}.
\end{proof}

\begin{lemma} \label{thm:cmp_bounds_eq_0p5}
Under D-TS, for any arm  $i \notin \mathcal{C}^*$, we have
\begin{equation}
\mathbb{E}[N_{ii}(T)] \leq O(K) + \sum_{k: p_{ki} > 1/2} \Theta\big( \frac{1}{\Delta_{ki}^2} + \frac{1}{\Delta_{ki}^2 D(1/2|| p_{ki})} + \frac{1}{\Delta_{ki}^4}\big) = O(K).
\end{equation}
\end{lemma}
\afterthm

Before proving Lemma~\ref{thm:cmp_bounds_eq_0p5}, we present an important property for $\hat{\zeta}^*(t) := \max_{1 \leq i \leq K} \hat{\zeta}_i(t)$. Recall that $\zeta^*$ is the maximum normalized Copeland score.
Using the concentration property of RUCB (Lemma~\ref{thm:rucb_properties} in Appendix~\ref{app:preliminary}), the following lemma shows that {\small $\hat{\zeta}^*(t)$}  is indeed a UCB of $\zeta^*$.
\begin{lemma}\label{thm:property_score_ucb}
For any $\alpha > 0.5$ and $t > 0$,
$
\mathbb{P} \{\hat{\zeta}^*(t) \geq  \zeta^*\} \geq 1 -  K \big[\frac{\log t}{\log(\alpha + 1/2)} + 1 \big]{t^{-\frac{2\alpha}{\alpha + 1/2}}}
$.
\end{lemma}
\afterthm

Return to the proof of Lemma~\ref{thm:cmp_bounds_eq_0p5}.
To prove Lemma~\ref{thm:cmp_bounds_eq_0p5}, we consider the cases of $\hat{\zeta}^*(t) < \zeta^*$ and $\hat{\zeta}^*(t) \geq \zeta^*$. The former case $\hat{\zeta}^*(t) < \zeta^*$ can be bounded by Lemma~\ref{thm:property_score_ucb}. For the latter case, we note that when $\hat{\zeta}^*(t) \geq \zeta^*$, the event {\small $(a_t^{(1)}, a_t^{(2)}) = (i,i)$} occurs only if: a) there exists at least one $k \in \mathcal{K}$ with $p_{ki} > 1/2$, such that $l_{ki}(t) \leq 1/2$; and b) $\theta_{ki}^{(2)}(t) \leq 1/2$ for all $k$ with $l_{ki}(t) \leq 1/2$. In this case, we can bound the probability of $(a_t^{(1)}, a_t^{(2)}) = (i,i)$  by that of  $(a_t^{(1)}, a_t^{(2)}) = (i,k)$, for $k$ with $p_{ki} > 1/2$ but $l_{ki}(t) \leq 1/2$, where the coefficient decays exponentially. Then we can bound $\mathbb{E}[N_{ii}(T)]$ by $O(1)$ similar to  \cite{Agrawal2013AISTATS:TS2}. Details of proof can be found in Appendix~\ref{app:proof_cmp_bounds_eq_0p5}.

The conclusion of Proposition~\ref{thm:regret_copeland} then follows by combining Lemmas~\ref{thm:cmp_bounds_less_0p5} to \ref{thm:cmp_bounds_eq_0p5}.

\subsubsection{Regret Bound Refinement} \label{subsec:regret_copeland_refined}

In this section, we refine the regret bound for D-TS and reduce its scaling factor with respect to the number of arms $K$.

We sort the arms for each $i \notin \mathcal{C}^*$ in the descending order of $p_{ji}$, and let  $(\sigma_{i(1)}, \sigma_{i(2)}, \ldots, \sigma_{i(K)})$ be a permutation of $(1,2, \ldots, K)$, such that $p_{\sigma_{i(1)},i} \geq p_{\sigma_{i(2)}, i} \geq \ldots \geq p_{\sigma_{i(K)}, i}$. In addition, for a Copeland winner $i^* \in \mathcal{C}$, let $L_C = \sum_{j=1}^K \mathds{1}(p_{ji^*} > 1/2)$ be the number of arms that beat arm $i^*$. To refine the regret, we introduce an additional no-tie assumption:

\textbf{Assumption~2:} For each arm $i \notin \mathcal{C}^*$, $p_{\sigma_{i(L_C+1)},i} > p_{\sigma_{i(j)},i}$ for all $j > L_C + 1$.

We present a refined regret bound for D-TS as follows:
\begin{theorem} \label{thm:regret_copeland_refined}
When applying D-TS with $\alpha > 0.5$ in a Copeland dueling bandit with a preference matrix $P = [p_{ij}]_{K \times K}$ satisfying Assumptions~1 and 2, its regret is bounded as:
{\small
\begin{eqnarray}\label{eq:regret_copeland_refined}
R_{\dts}(T) &\leq&
 \sum_{i \in \mathcal{C}^*}\bigg[\sum_{j:p_{ji} > 1/2}\frac{4\alpha\log T}{\Delta_{ji}^2} + \sum_{j:p_{ji} < 1/2}(1+ \epsilon)   \frac{\log T}{D(p_{ji}||1/2)}\bigg] + \sum_{i \notin \mathcal{C}^*} \sum_{j=1}^{L_C+1} \frac{4\alpha\log T}{\Delta_{\sigma_{i(j)},i}^2} \nonumber\\
&& +   \beta(1+ \epsilon)^2\sum_{i \notin \mathcal{C}^*} \sum_{j = L_C + 2}^K \frac{\log \log T}{D(p_{\sigma_{i(j)},i}||p_{\sigma_{i(L_C+1)},i})} + O(K^3) +O(\frac{K^2}{\epsilon^2}),
\end{eqnarray}
}
where $\beta > 2$ and $\epsilon > 0$ are constants, and $D(\cdot || \cdot)$ is the KL-divergence.
\end{theorem}
In \eqref{eq:regret_copeland_refined}, the first term corresponds to the regret when the first candidate {\small $a_t^{(1)}$} is a winner, and  is $O(K |\mathcal{C}^*| \log T)$. The second term corresponds to the comparisons between a non-winner arm and its first $L_C+1$ superiors, which is bounded by $O(K(L_C+1) \log T)$.  The remaining terms correspond to the comparisons between a non-winner arm and the remaining arms, and is bounded by $O\big(K^2 \log\log T\big)$. As demonstrated in \cite{Zoghi2015NIPS:CDB}, $L_C$ is relatively small compared to $K$, and can be viewed as a constant.  Thus, the total regret $R_{\dts}(T)$ is bounded as $R_{\dts}(T) = O(K \log T + K^2 \log\log T)$. In particular, this asymptotic trend can be easily seen for Condorcet dueling bandits where $L_C = 0$.

Comparing Eq.~\eqref{eq:regret_copeland_refined} with Eq.~\eqref{eq:regret_copeland}, we can see the difference is the third and fourth terms in \eqref{eq:regret_copeland_refined}, which refine the regret of comparing a suboptimal arm and its last $(K-L_C-1)$ inferiors into $O(\log \log T)$.
Thus,
to prove Theorem~\ref{thm:regret_copeland_refined}, it suffices to show the following additional lemma:
\begin{lemma} \label{thm:regret_refined_singlearm}
Under Assumptions~1 and 2, for any suboptimal arm $i \notin \mathcal{C}^*$ and $j > L_C + 1$, we have
\begin{equation}
\mathbb{E}[N_{i\sigma_{i(j)}}^{(1)}(T)] \leq \frac{\beta(1+\epsilon)^2 \log \log T}{D(p_{\sigma_{i(j)},i}||p_{\sigma_{i(L_C+1)},i})} + O(K) + O(\frac{1}{\epsilon^2}),
\end{equation}
where $\beta > 2$ and $\epsilon > 0$ are constants.
\end{lemma}

\begin{proof} We prove this lemma using a \emph{back substitution} argument. The intuition is that when fixing the first candidate as {\small $a_t^{(1)} = i$}, the comparison between {\small $a_t^{(1)}$} and the other arms is similar to a traditional MAB with expected utilities $p_{ji}$ ($1 \leq j \leq K$). Let {\small $ N_i^{(1)}(T) = \sum_{t = 1}^T \mathds{1}(a_t^{(1)} = i)$} be the number of time-slots when this type of MAB is played. Using the fact that the distribution of the samples only depends on the historic comparison results (but not $t$), we can show {\small $\mathbb{E}[N_{i,\sigma_{i(j)}}^{(1)}(T) | N_i^{(1)}(T)] = O(\log N_i^{(1)}(T))$}, which holds for any $N_i^{(1)}(T)$. We have shown that $\mathbb{E}[N_i^{(1)}(T)] = O(K \log T)$ for any $i\neq \mathcal{C}^*$ when proving Proposition~\ref{thm:regret_copeland}. Then, substituting the bound of $\mathbb{E}[N_i^{(1)}(T)]$ back and using the concavity of the $\log(\cdot)$ function, we have  $\mathbb{E}[N_{i,\sigma_{i(j)}}^{(1)}(T)] = \mathbb{E}\big[\mathbb{E}[N_{i,\sigma_{i(j)}}^{(1)}(T) | N_i^{(1)}(T)]\big] \leq O(\log \mathbb{E}[N_i^{(1)}(T)]) = O(\log \log T + \log K)$.
Details can be found in Appendix~\ref{app:proof_regret_refined_singlearm}
\end{proof}


\subsection{Further Improvement: D-TS$^+$} \label{subsec:dts_plus}
D-TS is a TS framework for dueling bandits, and its performance can be improved by refining certain components of it.
In this section, we propose an enhanced version of D-TS, referred to as D-TS$^+$, that carefully breaks the ties to reduce the regret.

Note that by randomly breaking the ties (Line~11 in Algorithm~\ref{alg:dts_Copeland}), D-TS tends to explore all potential winners. This may be desirable in certain applications such as restaurant recommendation, where users may not want to stick to a single winner. However, because of this, the regret of D-TS scales with the number of winners $|\mathcal{C}^*|$ as shown in Theorem~\ref{thm:regret_copeland_refined}. To further reduce the regret, we can break the ties according to estimated regret.

Specifically, with samples {\small $\theta_{ij}^{(1)}(t)$}, the normalized Copeland score for each arm $i$ can be estimated as
$
{\small \tilde{\zeta}_i(t) = \frac{1}{K-1} \sum_{j \neq i}\mathds{1}(\theta_{ij}^{(1)}(t) > 1/2)}
$.
Then the maximum normalized Copeland score is
$
{\small \tilde{\zeta}^*(t) = \max_i~ \tilde{\zeta}_i(t)}
$,
and the loss of comparing arm $i$ and arm $j$ is
$
{\small \tilde{r}_{ij}(t) =\tilde{\zeta}^*(t) -\frac{1}{2}\big[ \tilde{\zeta}_i(t) + \tilde{\zeta}_j(t)\big]}
$.
For $p_{ij}\neq 1/2$, we need about {\small $\Theta(\frac{\log T}{D(p_{ij} || 1/2)})$} time-slots to distinguish it from 1/2 \cite{Komiyama2015COLT:DB}. Thus, when choosing $i$ as the first candidate, the regret of comparing it with all other arms can be estimated by
$
{\small \tilde{R}^{(1)}_i(t) = \sum_{j: \theta_{ij}^{(1)}(t) \neq 1/2}\tilde{r}_{ij}(t) / D(\theta_{ij}^{(1)}(t) || 1/2)}
$.
We propose the following D-TS$^+$ algorithm that breaks the ties to minimize {\small $\tilde{R}^{(1)}_i(t)$}.

\textbf{D-TS$^+$:} Implement the same operations as D-TS, except for the selection of the first candidate (Line~11 in Algorithm~\ref{alg:dts_Copeland}) is replaced by the following two steps:
{\small
\begin{eqnarray}
&&\mathcal{A}^{(1)} \gets \{i \in \mathcal{C}: \zeta_i = \max_{i\in \mathcal{C}} \sum_{j\neq i} \mathds{1}(\theta^{(1)}_{ij} > 1/2)\}; \nonumber\\
&&a^{(1)} \gets \argmin_{i \in \mathcal{A}^{(1)}} \tilde{R}^{(1)}_i;   \nonumber
\end{eqnarray}
}
D-TS$^+$ only changes the tie-breaking criterion in selecting the first candidate. Thus, the regret bound of D-TS directly applies to D-TS$^+$:
\begin{corol} \label{thm:regret_dts_plus}
The regret of D-TS$^+$, $R_{\dtsp}(T)$, satisfies inequality \eqref{eq:regret_copeland_refined} under Assumptions~1 and 2.
\end{corol}

Corollary~\ref{thm:regret_dts_plus} provides an upper bound for the regret of D-TS$^+$. In practice, however, D-TS$^+$ performs better than D-TS in the scenarios with multiple winners, as we can see in Section~\ref{sec:sim_results} and Appendix~\ref{app:add_sim_results}. Our conjecture is that with this regret-minimization criterion, the D-TS$^+$ algorithm tends to focus on one of the winners (if there is no tie in terms of expected regret), and thus reduces the first term in \eqref{eq:regret_copeland_refined} from $O(K |\mathcal{C}^*|\log T)$ to $O(K \log T)$. The proof of this conjecture requires properties for the evolution of the statistics for all arms and the majority voting results based on the Thompson samples, and is complex. This is left as part of our future work.

In the above D-TS$^+$ algorithm, we only consider the regret of choosing $i$ as the first candidate. From Theorem~\ref{thm:regret_copeland_refined}, we know that comparing other arms with their superiors will also result in  $\Theta(\log T)$ regret. Thus, although the current D-TS$^+$ algorithm performs well in most practical scenarios, one may further improve its performance by taking these additional comparisons into account in {\small $\tilde{R}^{(1)}_i(t)$}.

\section{Experiments} \label{sec:sim_results}

To evaluate the proposed D-TS and \dtsp~algorithms, we run experiments based on synthetic and real-world data. Here we present the results for experiments based on the Microsoft Learning to Rank (MSLR) dataset \cite{Microsoft2010MSLR},
which provides the relevance for queries and ranked documents. Based on this dataset, \cite{Zoghi2015NIPS:CDB} derives a preference matrix for 136 rankers,
where each ranker is a function that maps a user's query to a document ranking and can be viewed as one arm in dueling bandits.
We use the two 5-armed submatrices in  \cite{Zoghi2015NIPS:CDB}, one for Condorcet dueling bandit and the other for non-Condorcet dueling bandit.
More experiments and discussions can be found in Appendix~\ref{app:add_sim_results} \footnote{Source codes are available at \url{https://github.com/HuasenWu/DuelingBandits}.}.

We compare D-TS and \dtsp~with the following algorithms: BTM \cite{Yue2011ICML:BTM}, SAVAGE \cite{Urvoy2013ICML:SAVAGE}, Sparring \cite{Ailon2014ICML:UBDB}, RUCB \cite{Zoghi2014ICML:RUCB},  RCS \cite{Zoghi2014WSDM:RCS}, CCB \cite{Zoghi2015NIPS:CDB}, SCB \cite{Zoghi2015NIPS:CDB}, RMED1 \cite{Komiyama2015COLT:DB}, and ECW-RMED \cite{Komiyama2016ICML:CWRMED}. For BTM, we set the relaxed factor $\gamma = 1.3$ as \cite{Yue2011ICML:BTM}. For algorithms using RUCB and RLCB, including D-TS and D-TS$^+$, we set the scale factor $\alpha = 0.51$. For RMED1, we use the same settings as \cite{Komiyama2015COLT:DB}, and for ECW-RMED, we use the same setting as \cite{Komiyama2016ICML:CWRMED}. For the ``explore-then-exploit'' algorithms, BTM and SAVAGE, each point is obtained by resetting the time horizon as the corresponding value.
The results are averaged over 500 independent experiments, where in each experiment, the arms are randomly shuffled to prevent algorithms from exploiting special structures of the preference matrix.

\begin{figure*}[thbp]
\hspace{-.3cm}
\vspace{-0.25cm}
\begin{minipage}[b]{.66\textwidth}
\begin{center}
\subfigure[$K = 5$, Condorcet]{\includegraphics[angle = 0,width = 0.485\linewidth, height = 0.39\linewidth]{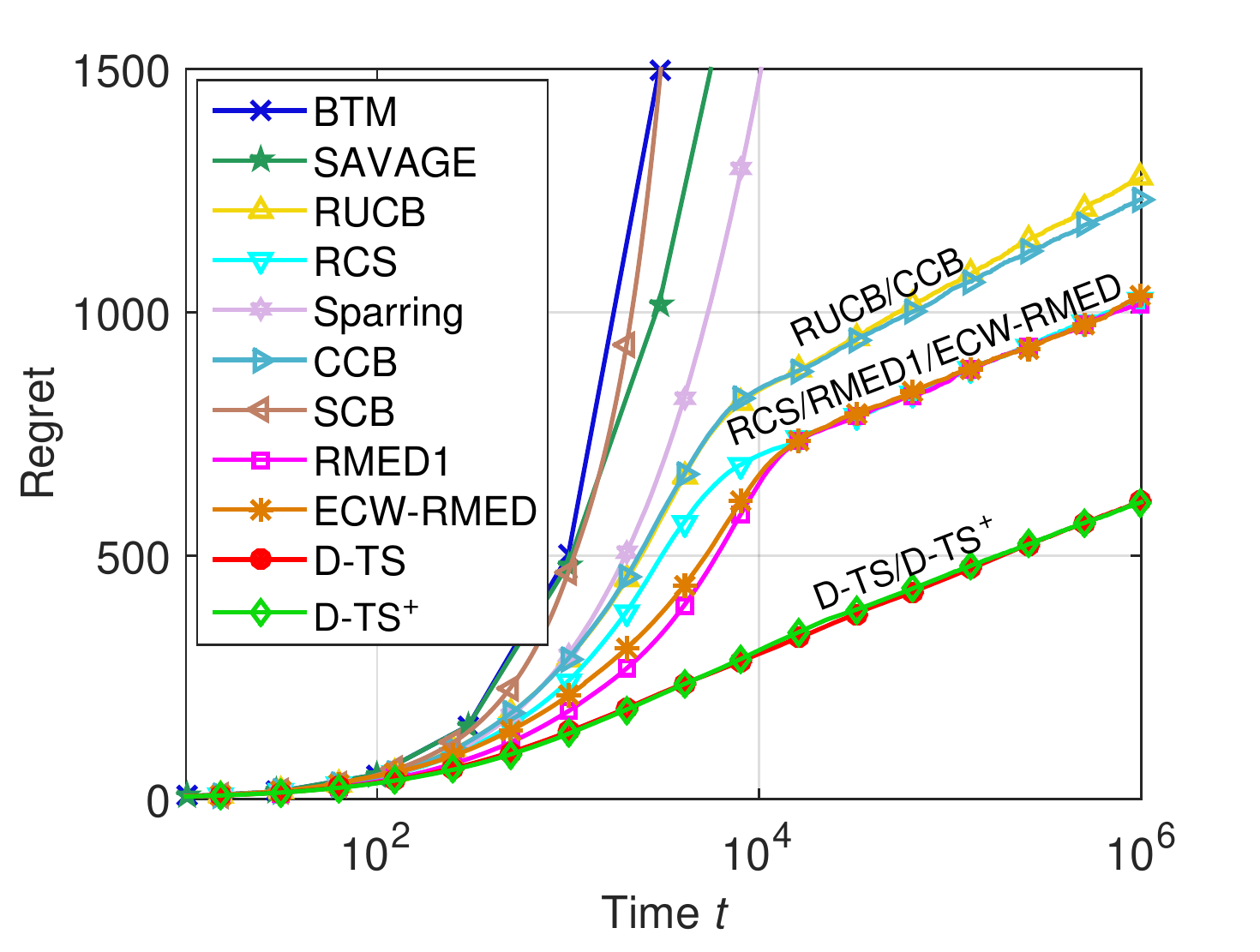}
\label{fig:MSLR_Informational_5_Condorcet}}
\subfigure[$K = 5$, non-Condorcet]{\includegraphics[angle = 0,width = 0.485\linewidth,height = 0.39\linewidth]{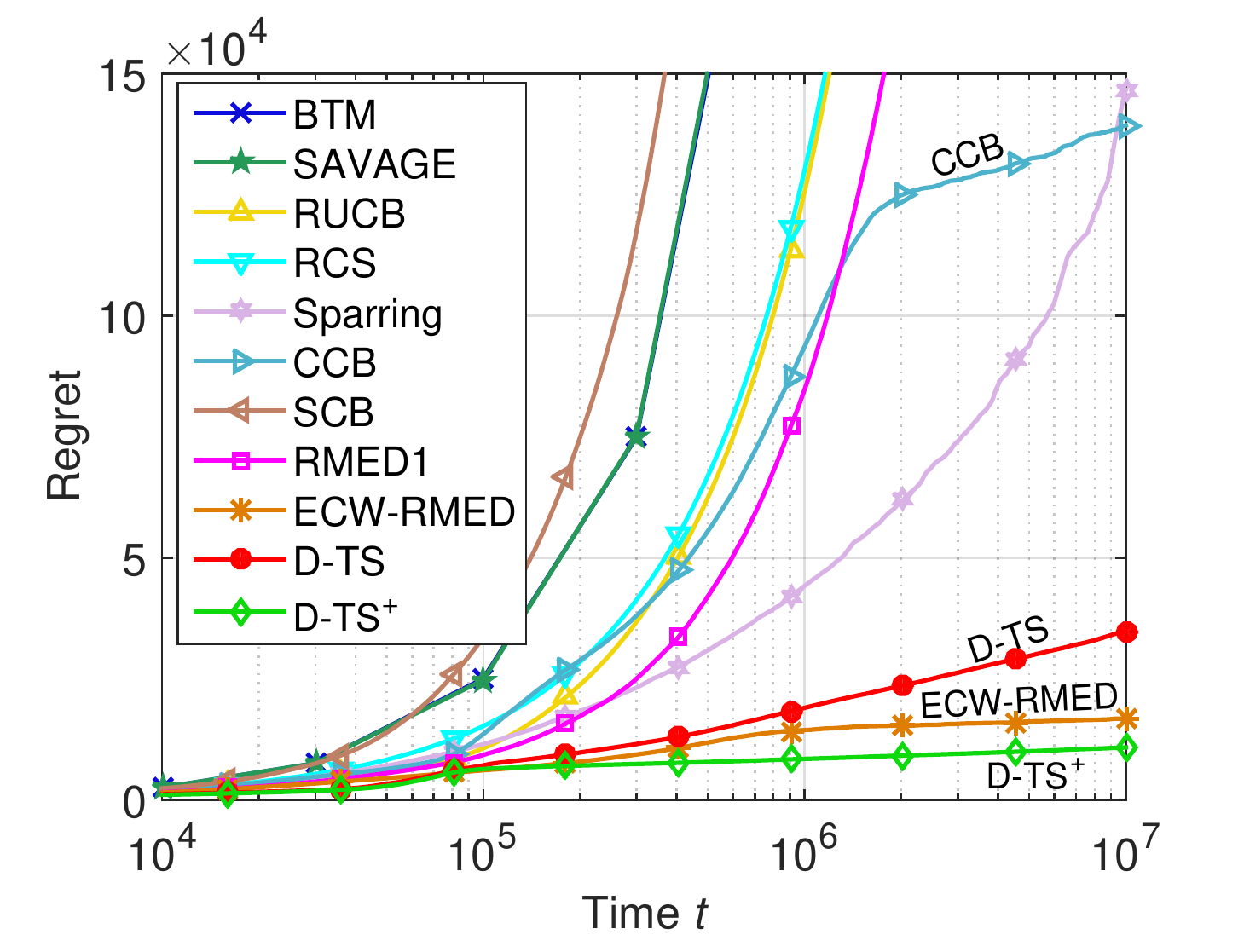}
\label{fig:MSLR_Informational_5_non_Condorcet}}
\vspace{-0.35cm}
\caption{Regret in MSLR dataset. In (b), there are 3 Copeland winners with normalized Copeland score $\zeta^* = 3/4$.}
\label{fig:regret_mslr}
\end{center}
\end{minipage}
\hspace{.25cm}
\begin{minipage}[b]{.34\textwidth}
\begin{center}
{\includegraphics[angle = 0,width = 0.98\linewidth, height = 0.78\linewidth]{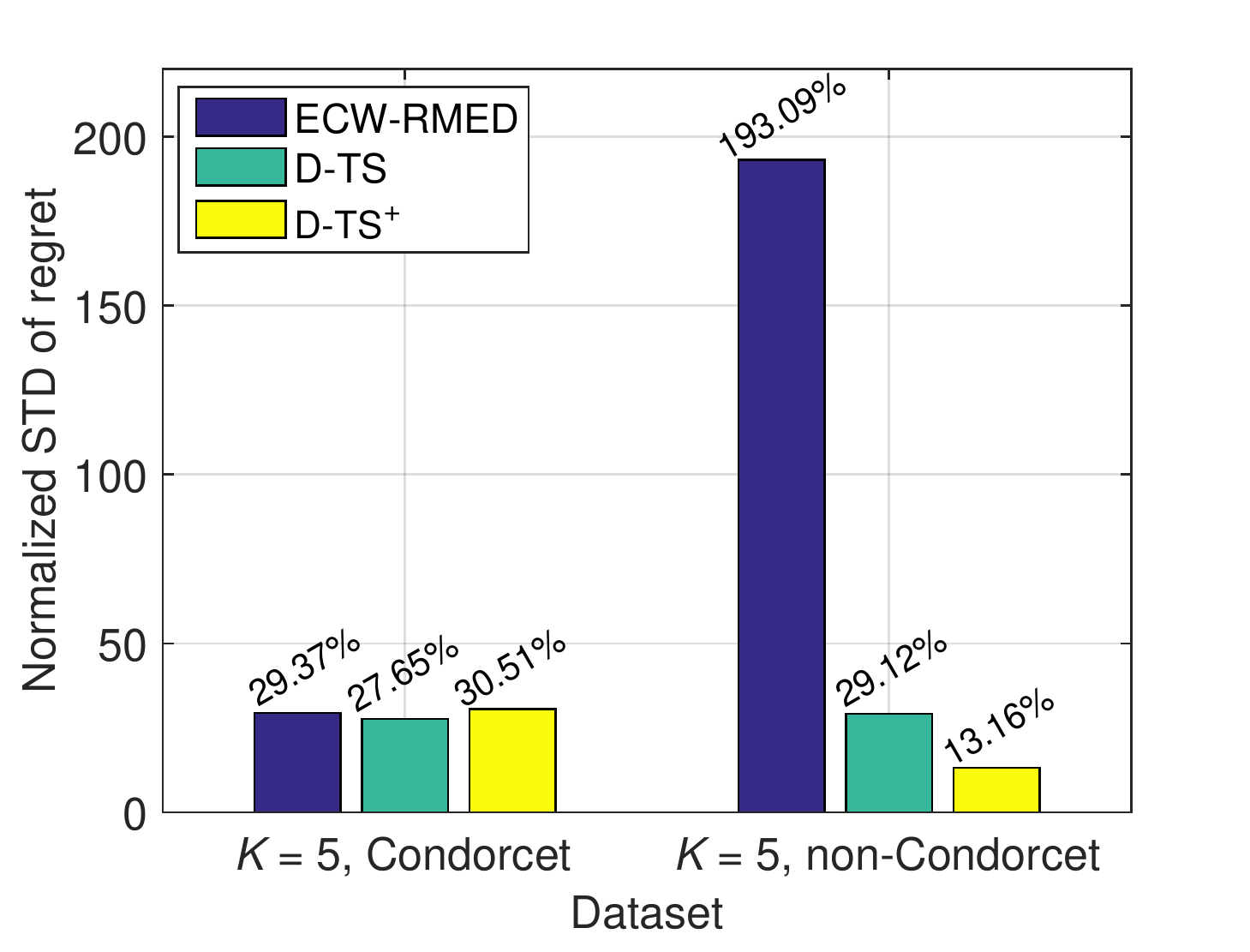}
\vspace{-.25cm}
\caption{Standard deviation (STD) of regret for $T = 10^6$ {\small (normalized by $R_{\rm ECW-RMED}(T)$)}.
\label{fig:regret_deviation}}}
\end{center}
\end{minipage}
\end{figure*}


In Condorcet dueling bandits, our D-TS and D-TS$^+$ algorithms achieve almost the same performance and both perform much better than existing algorithms, as shown in Fig.~\ref{fig:MSLR_Informational_5_Condorcet}.
In particular, compared with RCS, we can see that the full utilization of TS in D-TS and D-TS$^+$ significantly reduces the regret.
Compared with RMED1 and ECW-RMED, our D-TS and D-TS$^+$ algorithms also perform better.
\cite{Komiyama2015COLT:DB} has shown that RMED1 is optimal in Condorcet dueling bandits,
not only in the sense of asymptotic order, but also the coefficients in the regret bound.
The simulation results show that D-TS and D-TS$^+$ not only achieve the similar slope as RMED1/ECW-RMED, but also converge faster to the asymptotic regime and thus achieve much lower regret. This inspires us to further refine the regret bounds for D-TS and D-TS$^+$ in the future.

In non-Condorcet dueling bandits,
as shown in Fig.~\ref{fig:MSLR_Informational_5_non_Condorcet}, D-TS and D-TS$^+$ significantly reduce the regret compared to the UCB-type algorithm, CCB
(e.g., the regret of D-TS$^+$ is less than 10\% of that of CCB).
Compared with ECW-RMED, D-TS achieves higher regret, mainly because it randomly explores all Copeland winners due to the random tie-breaking rule.
With a regret-minimization tie-breaking rule, D-TS$^+$ further reduces the regret, and  outperforms ECW-RMED in this dataset.
Moreover, as randomized algorithms, D-TS and D-TS$^+$ are more robust to the preference probabilities.
As shown in Fig.~\ref{fig:regret_deviation}, D-TS and D-TS$^+$ have much smaller regret STD than
 that of ECW-RMED in the non-Condorcet dataset, where certain preference probabilities (for different arms) are close to 1/2.
In particular, the STD of regret for ECW-RMED is almost 200\% of its mean value,
while it is only 13.16\% for D-TS$^+$.  In addition, as shown in Appendix~\ref{app:robustness}, D-TS and D-TS$^+$ are
also robust to delayed feedback, which is typically batched and provided periodically in practice.

Overall,  D-TS and D-TS$^+$ significantly outperform all existing algorithms, with the exception of ECW-RMED.
Compared to ECW-RMED, D-TS$^+$ achieves much lower regret in the Condorcet case,
lower or comparable regret in the  non-Condorcet case, and much more robustness  in terms of regret STD and delayed feedback. Thus,
the simplicity, good performance, and robustness of D-TS and D-TS$^+$ make them good algorithms in practice.

\section{Conclusions and Future Work} \label{sec:conclusion}
In this paper, we study TS algorithms for dueling bandits.
We propose a D-TS algorithm and its enhanced version D-TS$^+$ for general Copeland dueling bandits, including Condorcet dueling bandits as a special case.
Our study reveals desirable properties of D-TS and D-TS$^+$ from both theoretical and practical perspectives.
Theoretically, we show that the regret of D-TS and D-TS$^+$ is bounded by $O(K^2 \log T)$ in general Copeland dueling bandits, and can be refined to $O(K\log T + K^2 \log \log T)$ in Condorcet dueling bandits and most practical Copeland dueling bandits.
Practically, experimental results demonstrate that these simple algorithms achieve significantly better overall-performance than existing algorithms, i.e., D-TS and D-TS$^+$ typically achieve much lower regret in practice and are robust to many practical factors, such as preference matrix and feedback delay.

Although logarithmic regret bounds have been obtained for D-TS and D-TS$^+$, our analysis relies heavily on the properties of RUCB/RLCB and the regret bounds are likely loose. In fact, we see from experiments that RUCB-based elimination seldom occurs under most practical settings. We will further refine the regret bounds by investigating the properties of TS-based majority-voting. Moreover, results from recent work such as \cite{Komiyama2016ICML:CWRMED} may be leveraged to improve TS algorithms.  Last, it is also an interesting future  direction to study  D-TS type algorithms for dueling bandits with other definition of winners. 

\textbf{Acknowledgements:}
This research was supported in part by NSF Grants CCF-1423542, CNS-1457060, and CNS-1547461.
The authors would like to thank Prof.~R.~Srikant (UIUC), Prof.~Shipra Agrawal (Columbia University),
Masrour Zoghi (University of Amsterdam),  and Dr.~Junpei Komiyama (University of Tokyo)  for their helpful discussions and suggestions.


\bibliography{OnlineOpt,mypublications_0802,math}
\bibliographystyle{unsrt_abbrv}
\clearpage
\newpage
\pagebreak

\appendix

\section*{Appendices}

\section{Preliminary: Concentration of RUCB/RLCB} \label{app:preliminary}

%

We first present the concentration properties of RUCB/RLCB. By relating RUCB/RLCB to UCB/LCB in traditional MAB, we can adjust the results in \cite{Bubeck2010PhD:bandits} for RUCB/RLCB as follows.

\begin{lemma} \label{thm:rucb_properties}

1) When $\alpha > 0.5$, for any $(i,j)$ and $t > 0$,
\begin{equation} \label{eq:rucb_dev}
\mathbb{P} \{p_{ij} \geq u_{ij}(t)\} \leq \big[\frac{\log t}{\log(\alpha + 1/2)} + 1 \big]{t^{-\frac{2\alpha}{\alpha + 1/2}}},
\end{equation}
\begin{equation}\label{eq:rlcb_dev}
\mathbb{P} \{p_{ij} \leq l_{ij}(t)\} \leq \big[\frac{\log t}{\log(\alpha + 1/2)} + 1 \big]{t^{-\frac{2\alpha}{\alpha + 1/2}}}.
\end{equation}
2) For any $\alpha > 1/2$,
\begin{equation}
\sum_{t=1}^T \mathbb{P}\{p_{ij} \geq u_{ij}(t)\} \leq \frac{2}{\log(\alpha + 1/2)[2\alpha/(\alpha + 1/2) - 1]^2} = O(1),
\end{equation}
\begin{equation}
\sum_{t=1}^T \mathbb{P} \{p_{ij} \leq l_{ij}(t)\} \leq \frac{2}{\log(\alpha + 1/2)[2\alpha/(\alpha + 1/2) - 1]^2} = O(1).
\end{equation}
\end{lemma}
\begin{proof}
We prove this lemma using the techniques in the proof of Theorem~2.2 in \cite{Bubeck2010PhD:bandits}.

In fact, RUCB (resp., RLCB) in dueling bandits are essentially the same as UCB (resp., LCB) in traditional MAB. Thus, Part~1) of this lemma can be proved using the peeling argument in \cite{Bubeck2010PhD:bandits}.

For Part~2), the sum can be bounded by the integration $\int_{1}^{\infty} \big[\frac{\log t}{\log(\alpha + 1/2)} + 1 \big]{t^{-\frac{2\alpha}{\alpha + 1/2}}} {\rm d}t$ as in \cite{Bubeck2010PhD:bandits}.
\end{proof}



\section{Regret Analysis: $O(K^2 \log T)$ Regret} \label{app:regret_copeland}

\subsection{Proof of Lemma~\ref{thm:cmp_bounds_less_0p5}} \label{app:proof_cmp_bounds_less_0p5}
For a pair $(i,j)$ with $p_{ji} < 1/2$, let $x_{ji}$ be a number satisfying $p_{ji} < x_{ji} < 1/2$.
Let $\bar{p}_{ji}(t) = \frac{B_{ji}(t-1)}{B_{ji}(t-1) + B_{ij}(t-1)}$ be the empirical estimation for the probability that arm $j$ beats arm $i$. Define the following events:
\begin{equation}
\mathcal{E}^{{p}}_{ji}(t) = \{\bar{p}_{ji}(t) < x_{ji}\}, \nonumber
\end{equation}
\begin{equation}
\mathcal{E}^{\theta}_{ji}(t) = \{\theta^{(2)}_{ji}(t) < 1/2 \}. \nonumber
\end{equation}

For an event $\mathcal{E}$, we let $\urcorner \mathcal{E}$ be the event of ``not $\mathcal{E}$''. Then
\begin{eqnarray}
\mathbb{E}[N^{(1)}_{ij}(T)] &=& \sum_{t=1}^T \mathbb{P} \big\{(a_t^{(1)}, a_t^{(2)}) = (i,j)\big\} \nonumber \\
&= & \sum_{t=1}^T \mathbb{P}\{(a_t^{(1)}, a_t^{(2)}) = (i,j),  \mathcal{E}^{{p}}_{ji}(t), \mathcal{E}^{\theta}_{ji}(t)\} \nonumber\\
&& + \sum_{t=1}^T \mathbb{P}\{(a_t^{(1)}, a_t^{(2)}) = (i,j),  \mathcal{E}^{{p}}_{ji}(t), \urcorner \mathcal{E}^{\theta}_{ji}(t)\} \nonumber \\
&& + \sum_{t=1}^T \mathbb{P}\{(a_t^{(1)}, a_t^{(2)}) = (i,j), \urcorner \mathcal{E}^{{p}}_{ji}(t)\}. \nonumber
\end{eqnarray}

The first term is zero, because  $\mathbb{P}\{(a_t^{(1)}, a_t^{(2)}) = (i,j),  \mathcal{E}^{{p}}_{ji}(t), \mathcal{E}^{\theta}_{ji}(t)\} = 0$ for all $t$, due to the fact that $a_t^{(2)} \neq j$ when $\theta^{(2)}_{ji}(t) < 1/2 = \theta^{(2)}_{ii}(t)$.

The second and third terms can be bounded similarly to the analysis of TS in traditional MABs \cite{Agrawal2013AISTATS:TS2}. To see this, we note that when fixing the first candidate as $a_t^{(1)} = i$, the comparison between $i$ and other arms is similar to a traditional MAB problem with expected reward $p_{ji}$ ($1\leq j \leq K$). For the case of $p_{ji} < 1/2$, we only need to care about two differences: first, $p_{ii} = 1/2$ is fixed and known; second, in addition to $(a_t^{(1)}, a_t^{(2)}) = (i,j)$, arm $i$ and arm $j$ could also be compared when $(a_t^{(1)}, a_t^{(2)}) = (j,i)$. By capturing the second difference with $N_{ij}(t-1) = N_{ij}^{(1)}(t-1) + N_{ij}^{(2)}(t-1)$, we can leverage the techniques in \cite{Agrawal2013AISTATS:TS2} to prove our results.

Specifically, the second term can be bounded by using the concentration property of the Thompson samples. Letting $L_{ji}(T) = \frac{\log T}{D(x_{ji} || 1/2)}$,  similar to the proof of Lemma~4 in \cite{Agrawal2013AISTATS:TS2}, we have
\begin{eqnarray} \label{eq:p_leq0p5_sec_term}
&&\sum_{t = 1}^T \mathbb{P}\{(a_t^{(1)}, a_t^{(2)}) = (i,j),  \mathcal{E}^{{p}}_{ji}(t), \urcorner \mathcal{E}^{\theta}_{ji}(t)\} \nonumber \\
& =& \sum_{t = 1}^T \mathbb{P}\{(a_t^{(1)}, a_t^{(2)}) = (i,j),  \mathcal{E}^{{p}}_{ji}(t), \urcorner \mathcal{E}^{\theta}_{ji}(t), N_{ij}(t-1) \leq L_{ji}(T)\}  \nonumber \\
&& +  \sum_{t = 1}^T \mathbb{P}\{(a_t^{(1)}, a_t^{(2)}) = (i,j),  \mathcal{E}^{{p}}_{ji}(t), \urcorner \mathcal{E}^{\theta}_{ji}(t), N_{ij}(t-1) >   L_{ji}(T)\} \nonumber \\
&\leq & L_{ji} (T) + \sum_{t = 1}^T \frac{1}{T} \nonumber \\
&=& \frac{\log T}{D(x_{ji}|| 1/2)} + 1.
\end{eqnarray}

The third term can be bounded similarly to Lemma~3 in \cite{Agrawal2013AISTATS:TS2}. Specifically, let $\tau_n$ be the slot index when $i$ and $j$ are compared for the $n$-th time, including both cases $(a_t^{(1)}, a_t^{(2)}) = (i,j)$ and $(a_t^{(1)}, a_t^{(2)}) = (j,i)$. Let $\tau_0 = 0$. Then, $\bar{p}_{ji}(t)$ is fixed between $\tau_{n} + 1$ and $\tau_{n+1}$, and $\sum_{t = \tau_{n}+1}^{\tau_{n + 1}} \mathds{1}((a_t^{(1)}, a_t^{(2)}) = (i,j)) \leq 1$ (it is 0 if the $(n+1)$-th comparison is implemented in the form of $(a_t^{(1)}, a_t^{(2)}) =(j,i)$). Then
\begin{eqnarray}
&& \sum_{t = 1}^T \mathbb{P}\{(a_t^{(1)}, a_t^{(2)}) = (i,j), \urcorner \mathcal{E}^{{p}}_{ji}(t)\} \nonumber \\
&\leq& \sum_{n=0}^{T-1} \mathbb{E}\bigg[\sum_{t = \tau_n + 1}^{\tau_{n + 1}} \mathds{1}((a_t^{(1)}, a_t^{(2)}) = (i,j)) \cdot \mathds{1}(\urcorner \mathcal{E}^{{p}}_{ji}(t))\bigg] \nonumber \\
&\leq & \sum_{n=0}^{T-1}  \mathbb{E}\bigg[ \mathds{1}(\urcorner \mathcal{E}^{{p}}_{ji}(\tau_n + 1)) \sum_{t = \tau_n + 1}^{\tau_{n + 1}}  \mathds{1}((a_t^{(1)}, a_t^{(2)}) = (i,j))\bigg]  \nonumber \\
&\leq &  \sum_{n=0}^{T-1}  \mathds{P}(\urcorner \mathcal{E}^{{p}}_{ji}(\tau_n + 1))  \nonumber \\
&\leq &  1 + \sum_{n=1}^{T-1} e^{-n D(x_{ji} || p_{ji})} \nonumber \\
&\leq & 1 + \frac{1}{D(x_{ji} || p_{ji})}.
\end{eqnarray}

For any $\epsilon \in (0, 1]$, we choose $x_{ji} \in (p_{ji}, 1/2)$ such that $D(x_{ji}|| 1/2) = D(p_{ji} || 1/2)/(1+\epsilon)$, which also implies $\frac{1}{D(x_{ji} || p_{ji})} = O(\frac{1}{\epsilon^2})$ as shown in  \cite{Agrawal2013AISTATS:TS2}.
The conclusion then follows by combining the bounds for all the above three terms.

\subsection{Proof of Lemma~\ref{thm:cmp_bounds_larger_0p5}} \label{app:proof_cmp_bounds_larger_0p5}

We prove Lemma~\ref{thm:cmp_bounds_larger_0p5} by using the concentration  property of RLCB $l_{ji}(t)$. According to the definition of $N^{(1)}_{ij}(T)$, we have
\begin{eqnarray}
\mathbb{E}[N^{(1)}_{ij}(T)] &=& \sum_{t=1}^T \mathbb{P} \big\{(a_t^{(1)}, a_t^{(2)}) = (i,j)\big\}  \nonumber \\
&=& \sum_{t=1}^T \mathbb{P}\big\{(a_t^{(1)}, a_t^{(2)}) = (i,j),  N_{ij}(t-1) \geq \frac{4\alpha \log T}{\Delta_{ji}^2}\big\} \nonumber \\
&&+ \sum_{t=1}^T \mathbb{P}\big\{(a_t^{(1)}, a_t^{(2)}) = (i,j),   N_{ij}(t-1) < \frac{4\alpha \log T}{\Delta_{ji}^2}\big\}. \nonumber
\end{eqnarray}

For the first term, we note that when $a_t^{(1)} = i$, arm $j$ can be selected as $a_t^{(2)}$ only when the $l_{ji}(t) \leq 1/2$. When $N_{ij}(t-1) \geq \frac{4\alpha \log T}{\Delta_{ji}^2}$, we have $\Delta_{ji} \geq 2\sqrt{\frac{\alpha \log t}{N_{ij}(t-1)}}$. Thus, $l_{ji}(t) + \Delta_{ji} \geq u_{ji}(t)$. Because $p_{ji} > 1/2$, its RLCB satisfies
\begin{eqnarray}
\mathbb{P}\big\{l_{ji}(t)  \leq 1/2, N_{ij}(t-1) \geq \frac{4\alpha \log T}{\Delta_{ji}^2} \big\} \leq \mathbb{P}\big\{l_{ji}(t)
 &\leq& p_{ji} - \Delta_{ji}, N_{ij}(t-1) \geq \frac{4\alpha \log T}{\Delta_{ji}^2} \big\} \nonumber \\
 &\leq& \mathbb{P}\big\{u_{ji}(t)  \leq p_{ji} \big\}. \nonumber
\end{eqnarray}
Using Lemma~\ref{thm:rucb_properties}, we have
\begin{eqnarray}
\sum_{t=1}^T \mathbb{P}\big\{(a_t^{(1)}, a_t^{(2)}) = (i,j),  N_{ij}(t-1) \geq \frac{4\alpha \log T}{\Delta_{ji}^2}\big\} \leq \sum_{t=1}^T \mathbb{P}\big\{u_{ji}(t)  \leq p_{ji} \big\} = O(1). \nonumber
\end{eqnarray}

For the second term, we can bound it as follows:
\begin{eqnarray} \label{eq:regret_greater_0p5_smallN}
&& \sum_{t = 1}^T  \mathbb{P}\{(a_t^{(1)}, a_t^{(2)}) = (i,j),   N_{ij}(t-1) < \frac{4\alpha \log T}{\Delta_{ji}^2}\} \nonumber \\
&=&  \mathbb{E}\bigg[\sum_{t = 1}^T \mathds{1}((a_t^{(1)}, a_t^{(2)}) = (i,j),   N_{ij}(t-1) < \frac{4\alpha \log T}{\Delta_{ji}^2})\bigg] \leq \frac{4\alpha \log T}{\Delta_{ji}^2},
\end{eqnarray}
because $\sum_{t = 1}^T \mathds{1}\big((a_t^{(1)}, a_t^{(2)}) = (i,j),   N_{ij}(t-1) < \frac{4\alpha \log T}{\Delta_{ji}^2}\big) \leq \frac{4\alpha \log T}{\Delta_{ji}^2}$ due to the fact that: when $(a_t^{(1)}, a_t^{(2)}) = (i,j)$ at $t$, $N_{ij}(t-1)$ will be increased by 1, but  $ \mathds{1}\big((a_t^{(1)}, a_t^{(2)}) = (i,j),   N_{ij}(t-1) < \frac{4\alpha \log T}{\Delta_{ji}^2}\big) = 0$  as long as  $N_{ij}(t-1) \geq \frac{4\alpha \log T}{\Delta_{ji}^2}$.

The conclusion then follows by combining the bounds for the  above two terms.

\subsection{Proof of Lemma~\ref{thm:property_score_ucb}} \label{app:proof_property_score_ucb}

Let $i^*$ be the Copeland winner (or any one of them if there are multiple Copeland winners) in the dueling bandit.  We prove Lemma~\ref{thm:property_score_ucb} by analyzing the RUCB $u_{i^* j}(t)$ at $t$.
According to Lemma~\ref{thm:rucb_properties}, we have that for any $j \neq i^*$,
\begin{equation}
\mathbb{P}\{u_{i^* j}(t) < p_{i^* j} \} \leq \big[\frac{\log t}{\log(\alpha + 1/2)} + 1 \big]{t^{-\frac{2\alpha}{\alpha + 1/2}}}.
\end{equation}
Note that $\zeta^* = \frac{1}{K - 1} \sum_{j \neq i} \mathds{1}(p_{i^* j} > 1/2)$. Let $\mathcal{L}_{i^*} = \{j: 1 \leq j \leq K,  p_{i^* j} > 1/2 \}$ be the set of arms that lose to $i^*$. Thus,
\begin{eqnarray}
\hat{\zeta}^*(t) <  \zeta^*   \Rightarrow \exists j \in \mathcal{L}_{i^*}, \text{~such that}~u_{i^* j}(t) < p_{i^* j}.
\end{eqnarray}
Consider all elements in $\mathcal{L}_{i^*}$, we have
\begin{eqnarray}
\mathbb{P} \{\hat{\zeta}^*(t) \geq  \zeta^*\}
&=& 1 - \mathbb{P}\{\hat{\zeta}^*(t) <  \zeta^*\} \nonumber \\
&\geq& 1 - \mathbb{P}\{\exists j \in \mathcal{L}_{i^*}, \text{~s.t.}~u_{i^* j}(t) < p_{i^* j}\} \nonumber \\
&\geq& 1 - |\mathcal{L}_{i^*}| \big[\frac{\log t}{\log(\alpha + 1/2)} + 1 \big]{t^{-\frac{2\alpha}{\alpha + 1/2}}} \nonumber \\
&\geq& 1 -  K \big[\frac{\log t}{\log(\alpha + 1/2)} + 1 \big]{t^{-\frac{2\alpha}{\alpha + 1/2}}}.
\end{eqnarray}

\subsection{Proof of Lemma~\ref{thm:cmp_bounds_eq_0p5}} \label{app:proof_cmp_bounds_eq_0p5}
To bound the number of time-slots when we compare one non-winner arm against itself, we need to investigate the necessary conditions for this event.

Specifically, when the upper bound of the Copeland score $\hat{\zeta}^*(t) \geq \zeta^*$, the event $(a_t^{(1)}, a_t^{(2)}) = (i,i)$ for $i \notin \mathcal{C}^*$ occurs only if: a) there exists at least one $k \in \mathcal{K}$ with $p_{ki} > 1/2$, such that $l_{ki}(t) \leq 1/2$; and b) $\theta_{ki}^{(2)}(t) \leq 1/2$ for all $k$ with $l_{ki}(t) \leq 1/2$. Now we bound $\mathbb{E}[N_{ii}(T)]$ by bounding the probability of these two conditions.

For $k$ with $p_{ki} > 1/2$, we define the following probability
 \begin{equation}
q_{ki}(t) = \mathbb{P}\{\theta_{ki}^{(2)}(t) > 1/2 | \mathcal{H}_{t-1} \}.
\end{equation}
Note that the value of $\hat{\zeta}^*(t)$,  $l_{ki}(t)$, and  $q_{ki}(t)$ depends on the history, and thus is determined by $\mathcal{H}_{t-1}$. Similar to Lemma~1 in \cite{Agrawal2013AISTATS:TS2}, we bound the probability of comparing $i$ against itself (accompanied by $\hat{\zeta}^*(t) \geq \zeta^*$ and $l_{ki}(t) \leq 1/2$, which is different from TS for traditional MABs)  by that of comparing $i$ with $k$.
\begin{lemma} \label{thm:prob_bound}
Given $(i,k)$ with $p_{ki} > 1/2$, we have
\begin{eqnarray} \label{eq:prob_bound}
&& \mathbb{P}\big\{(a_t^{(1)}, a_t^{(2)}) = (i,i),  \hat{\zeta}^*(t) \geq \zeta^*, l_{ki}(t) \leq 1/2| \mathcal{H}_{t-1} \big\}  \nonumber \\
 &\leq& \frac{1 - q_{ki}(t)}{q_{ki}(t)} \mathbb{P}\big\{(a_t^{(1)}, a_t^{(2)}) = (i,k) | \mathcal{H}_{t-1}\big\}.
\end{eqnarray}
\end{lemma}
\begin{proof}
First of all, the value of $\hat{\zeta}^*(t)$ and $l_{ki}(t)$ depends on $\mathcal{H}_{t-1}$. Thus, if $\mathcal{H}_{t-1}$ satisfies that $\hat{\zeta}^*(t) < \zeta^*$ or $l_{ki}(t) > 1/2$, then \eqref{eq:prob_bound}  holds because the left hand side of is zero.

Now we consider $\mathcal{H}_{t-1}$ satisfying $\hat{\zeta}^*(t) \geq \zeta^*$ and $l_{ki}(t) \leq 1/2$. For the left hand side, we have
\begin{eqnarray}\label{eq:prob_bound_left}
&& \mathbb{P}\big\{(a_t^{(1)}, a_t^{(2)}) = (i,i),  \hat{\zeta}^*(t) \geq \zeta^*, l_{ki}(t) \leq 1/2| \mathcal{H}_{t-1} \big\}  \nonumber \\
& = & \mathbb{P}\big\{(a_t^{(1)}, a_t^{(2)}) = (i,i)| \mathcal{H}_{t-1} \big\} \nonumber \\
& \leq  & \mathbb{P}\big\{\theta_{k'i}^{(2)} (t) \leq 1/2, \forall k', \text{s.t.}~ l_{k'i}(t) \leq 1/2  | \mathcal{H}_{t-1} \big\} \nonumber \\
&=& \mathbb{P}\big\{\theta_{ki}^{(2)} (t) \leq 1/2 | \mathcal{H}_{t-1} \big\} \cdot \mathbb{P}\big\{\theta_{k'i}^{(2)} (t) \leq 1/2, \forall k'\neq k, \text{s.t.}~ l_{k'i}(t) \leq 1/2  |\mathcal{H}_{t-1} \big\} \nonumber \\
&=& [1-q_{ki}(t)] \mathbb{P}\big\{\theta_{k'i}^{(2)} (t) \leq 1/2, \forall k'\neq k, \text{s.t.}~ l_{k'i}(t) \leq 1/2  |\mathcal{H}_{t-1} \big\}.
\end{eqnarray}

For the right hand side, we have
\begin{eqnarray}\label{eq:prob_bound_right}
&& \mathbb{P}\big\{(a_t^{(1)}, a_t^{(2)}) = (i,k) | \mathcal{H}_{t-1}\big\}  \nonumber \\
&\geq& \mathbb{P}\big\{\theta_{ki}^{(2)} (t) > 1/2 \geq \theta_{k'i}^{(2)}(t),  \forall k'\neq k, \text{s.t.}~ l_{k'i}(t) \leq 1/2  |\mathcal{H}_{t-1} \big\} \nonumber \\
&=& \mathbb{P}\big\{\theta_{ki}^{(2)} (t) > 1/2 | \mathcal{H}_{t-1} \big\}  \cdot \mathbb{P}\big\{\theta_{k'i}^{(2)}(t) \leq 1/2,  \forall k'\neq k, \text{s.t.}~ l_{k'i}(t) \leq 1/2  |\mathcal{H}_{t-1} \big\} \nonumber \\
&=&  q_{ki}(t) \mathbb{P}\big\{\theta_{k'i}^{(2)} (t) \leq 1/2, \forall k'\neq i~\text{or}~k, \text{s.t.}~ l_{k'i}(t) \leq 1/2  |\mathcal{H}_{t-1} \big\}.
\end{eqnarray}
The conclusion then follows by combining \eqref{eq:prob_bound_left} and \eqref{eq:prob_bound_right}.
\end{proof}

Now we return to the proof of Lemma~\ref{thm:cmp_bounds_eq_0p5}. We divide the probability of $(a_t^{(1)}, a_t^{(2)}) = (i,i)$ into two terms according to the value of $\hat{\zeta}^*(t)$.
\begin{eqnarray}
&& \mathbb{E}[N_{ii}(T)] \nonumber \\
&=& \sum_{t = 1}^T \mathbb{P}\big\{(a_t^{(1)}, a_t^{(2)}) = (i,i)\} \nonumber \\
&\leq & \sum_{t = 1}^T \mathbb{P}\big\{(a_t^{(1)}, a_t^{(2)}) = (i,i), \hat{\zeta}^*(t) \geq \zeta^*, \exists k, p_{ki} > 1/2, l_{ki}(t) \leq 1/2\big\} + \sum_{t=1}^T  \mathbb{P}\big\{\hat{\zeta}^*(t) < \zeta^*\big\} \nonumber \\
&\leq & \sum_{k: p_{ki} > 1/2}  \sum_{t = 1}^T \mathbb{P}\big\{(a_t^{(1)}, a_t^{(2)}) = (i,i), \hat{\zeta}^*(t) \geq \zeta^*, l_{ki}(t) \leq 1/2\big\} + \sum_{t=1}^T K \big[\frac{\log t}{\log(\alpha + 1/2)} + 1 \big]{t^{-\frac{2\alpha}{\alpha + 1/2}}}\nonumber \\
&\leq& \sum_{k: p_{ki} > 1/2}  \sum_{t = 1}^T \mathbb{P}\big\{(a_t^{(1)}, a_t^{(2)}) = (i,i), \hat{\zeta}^*(t) \geq \zeta^*, l_{ki}(t) \leq 1/2\big\} + O(K).
\end{eqnarray}
In the above equation, we have already bounded the second term by using Lemmas~\ref{thm:property_score_ucb} and \ref{thm:rucb_properties}.

Next, we bound the first term by analyzing the bound for each $k$ with $p_{ki} > 1/2$. Let $\tau_n$ be the time-slot index where $k$ and $i$ are compared for the $n$-th time, including both cases $(a_t^{(1)}, a_t^{(2)}) = (i,k)$ and $(a_t^{(1)}, a_t^{(2)}) = (k,i)$, and let $\tau_0 = 0$. Then  by Lemma~\ref{thm:prob_bound}, we have that for each $k$ with $p_{ki} > 1/2$,
\begin{eqnarray}
&& \sum_{t = 1}^T \mathbb{P}\big\{(a_t^{(1)}, a_t^{(2)}) = (i,i), \hat{\zeta}^*(t) \geq \zeta^*, l_{ki}(t) \leq 1/2\big\} \nonumber \\
&=&  \sum_{t = 1}^T \mathbb{E}\bigg[\mathbb{P}\big\{(a_t^{(1)}, a_t^{(2)}) = (i,i), \hat{\zeta}^*(t) \geq \zeta^*, l_{ki}(t) \leq 1/2 | \mathcal{H}_{t-1}\big\}\bigg] \nonumber \\
&\leq & \sum_{t = 1}^T \mathbb{E} \bigg[ \frac{1- q_{ki}(t)}{q_{ki}(t)}\mathbb{P}\big\{(a_t^{(1)}, a_t^{(2)}) = (i,k) | \mathcal{H}_{t-1}\big\}\bigg] \nonumber \\
&\leq & \sum_{t = 1}^T \mathbb{E} \bigg[ \mathbb{E} \bigg[\frac{1- q_{ki}(t)}{q_{ki}(t)}\mathds{1}\big((a_t^{(1)}, a_t^{(2)}) = (i,k)\big) | \mathcal{H}_{t-1}\bigg] \bigg] \nonumber \\
& \overset{(a)}{=} & \sum_{n = 0}^{T-1} \mathbb{E}\bigg[ \frac{1- q_{ki}(\tau_{n} + 1)}{q_{ki}(\tau_{n} + 1)} \sum_{t = \tau_n +1}^{\tau_{n+1}} \mathds{1}\big((a_t^{(1)}, a_t^{(2)}) = (i,k)\big) \bigg]\nonumber \\
&\leq & \sum_{n = 0}^{T-1} \mathbb{E} \big[\frac{1}{q_{ki}(\tau_{n} + 1)} - 1\big]. \nonumber
\end{eqnarray}
The equality~(a) follows from the fact that the distribution of $\theta_{ki}^{(2)}(t)$ only changes after $k$ and $i$ are compared. According to Lemma~2 in \cite{Agrawal2013AISTATS:TS2}, $\mathbb{E} \big[\frac{1}{q_{ki}(\tau_{n} + 1)}\big]$ is bounded as follows:
\begin{eqnarray}
&& \mathbb{E} \big[\frac{1}{q_{ki}(\tau_{n} + 1)}\big] \nonumber \\
& \leq&
\begin{cases}
1 + \frac{3}{\Delta_{ki}}, &\text{for $n < \frac{8}{\Delta_{ki}}$}; \\
1 + \Theta\big(e^{-n\Delta_{ki}^2 /2} + \frac{1}{(n+1) \Delta_{ki}^2 } e^{-n D(1/2 || p_{ki})} + \frac{1}{e^{n\Delta_{ki}^2/4} - 1}\big), &\text{for $n \geq \frac{8}{\Delta_{ki}}$}.
\end{cases} \nonumber
\end{eqnarray}
Thus,
\begin{eqnarray}
&& \sum_{t = 1}^T \mathbb{P}\big\{(a_t^{(1)}, a_t^{(2)}) = (i,i), \hat{\zeta}^*(t) \geq \zeta^*, l_{ki}(t) \leq 1/2\big\} \nonumber \\
&\leq & \frac{24}{\Delta_{ki}^2} + \sum_{n = 0}^{T-1} \Theta\big(e^{-n\Delta_{ki}^2 /2} + \frac{1}{(n+1) \Delta_{ki}^2 } e^{-n D(1/2 || p_{ki})} + \frac{1}{e^{n\Delta_{ki}^2/4} - 1}\big) \nonumber \\
&\leq & \frac{24}{\Delta_{ki}^2} + \Theta\big( \frac{1}{\Delta_{ki}^2} + \frac{1}{\Delta_{ki}^2 D(1/2|| p_{ki})} + \frac{1}{\Delta_{ki}^4} + \frac{1}{\Delta_{ki}^2} \big) \nonumber \\
&=& \Theta\big( \frac{1}{\Delta_{ki}^2} + \frac{1}{\Delta_{ki}^2 D(1/2|| p_{ki})} + \frac{1}{\Delta_{ki}^4} \big). \nonumber
\end{eqnarray}
The conclusion then follows by summing over all $k$ with $p_{ki} > 1/2$.

\section{Regret Bound Refinement} \label{app:regret_condorcet_refined}
Theorem~\ref{thm:regret_copeland_refined} can be proved by combining Lemmas~\ref{thm:cmp_bounds_less_0p5} to \ref{thm:cmp_bounds_eq_0p5} and Lemma~\ref{thm:regret_refined_singlearm}. In this appendix, we present the proof of Lemma~\ref{thm:regret_refined_singlearm}.

\subsection{Proof of Lemma~\ref{thm:regret_refined_singlearm}} \label{app:proof_regret_refined_singlearm}

We first consider the event of $(a_t^{(1)}, a_t^{(2)}) = (i,\sigma_{i(j)})$ with $j > L_C + 1$ in two cases with different values of $\hat{\zeta}^*(t)$. Recall that $\zeta^* = (K - L_C)/(K-1)$  is the maximum normalized Copeland score. Then,
\begin{eqnarray} \label{eq:regret_refined_diff_zeta}
 \mathbb{E}[N_{i\sigma_{i(j)}}^{(1)}(T)] &=& \sum_{t = 1}^T \mathbb{P}\big\{(a_t^{(1)}, a_t^{(2)}) = (i,\sigma_{i(j)})\} \nonumber \\
&\leq & \sum_{t = 1}^T \mathbb{P}\big\{(a_t^{(1)}, a_t^{(2)}) = (i,\sigma_{i(j)}), \hat{\zeta}^*(t) \geq \zeta^*\big\} + \sum_{t=1}^T  \mathbb{P}\big\{\hat{\zeta}^*(t) < \zeta^*\big\} \nonumber\\
&\leq& \sum_{t = 1}^T \mathbb{P}\big\{(a_t^{(1)}, a_t^{(2)}) = (i,\sigma_{i(j)}), \hat{\zeta}^*(t) \geq \zeta^*\big\} + O(K),
\end{eqnarray}
where the second term is bounded by $\sum_{t=1}^T K \big[\frac{\log t}{\log(\alpha + 1/2)} + 1 \big]{t^{-\frac{2\alpha}{\alpha + 1/2}}} = O(K)$ according to Lemmas~\ref{thm:property_score_ucb} and \ref{thm:rucb_properties}.

To bound the first term in \eqref{eq:regret_refined_diff_zeta}, with a slight abuse of notation, we choose two numbers $x_{\sigma_{i(j)},i}$ and $y_{\sigma_{i(j)},i}$ such that $p_{\sigma_{i(j)},i} < x_{\sigma_{i(j)},i} < y_{\sigma_{i(j)},i} < p_{\sigma_{i(L_C+1)},i}$, and define the following events:
\begin{eqnarray}
&&\mathcal{E}^{{p}}_{\sigma_{i(j)},i}(t) = \{\bar{p}_{\sigma_{i(j)},i}(t) < x_{\sigma_{i(j)},i} \},\nonumber\\
&&\mathcal{E}^{\theta}_{\sigma_{i(j)},i}(t) = \{\theta^{(2)}_{\sigma_{i(j)},i}(t) < y_{\sigma_{i(j)},i} \},\nonumber
\end{eqnarray}
where the existence of $x_{\sigma_{i(j)},i}$ and $y_{\sigma_{i(j)},i}$ is guaranteed under Assumption 2.
Then, the first term can be decomposed as
\begin{eqnarray} \label{eq:regret_refined_diff_event}
&& \sum_{t=1}^T \mathbb{P}\big\{(a_t^{(1)}, a_t^{(2)}) = (i,\sigma_{i(j)}), \hat{\zeta}^*(t) \geq \zeta^*\big\} \nonumber\\
& = & \sum_{t=1}^T \mathbb{P}\big\{(a_t^{(1)}, a_t^{(2)}) = (i,\sigma_{i(j)}), \hat{\zeta}^*(t) \geq \zeta^*, \mathcal{E}^{{p}}_{\sigma_{i(j)},i}(t), \mathcal{E}^{\theta}_{\sigma_{i(j)},i}(t) \big\} \nonumber \\
&&+ \sum_{t=1}^T \mathbb{P}\big\{(a_t^{(1)}, a_t^{(2)}) = (i,\sigma_{i(j)}), \hat{\zeta}^*(t) \geq \zeta^*,  \mathcal{E}^{{p}}_{\sigma_{i(j)},i}(t), \urcorner \mathcal{E}^{\theta}_{\sigma_{i(j)},i}(t) \big\} \nonumber \\
&&+ \sum_{t=1}^T \mathbb{P}\big\{(a_t^{(1)}, a_t^{(2)}) = (i,\sigma_{i(j)}), \hat{\zeta}^*(t) \geq \zeta^*,  \urcorner \mathcal{E}^{{p}}_{\sigma_{i(j)},i}(t)\big\} \nonumber \\
&\leq & \sum_{t=1}^T \mathbb{P}\big\{(a_t^{(1)}, a_t^{(2)}) = (i,\sigma_{i(j)}), \hat{\zeta}^*(t) \geq \zeta^*, \mathcal{E}^{{p}}_{\sigma_{i(j)},i}(t), \mathcal{E}^{\theta}_{\sigma_{i(j)},i}(t) \big\} \nonumber \\
&&+ \sum_{t=1}^T \mathbb{P}\big\{(a_t^{(1)}, a_t^{(2)}) = (i,\sigma_{i(j)}), \mathcal{E}^{{p}}_{\sigma_{i(j)},i}(t), \urcorner \mathcal{E}^{\theta}_{\sigma_{i(j)},i}(t) \big\} \nonumber \\
&&+ \sum_{t=1}^T \mathbb{P}\big\{(a_t^{(1)}, a_t^{(2)}) = (i,\sigma_{i(j)}),  \urcorner \mathcal{E}^{{p}}_{\sigma_{i(j)},i}(t)\big\}.
\end{eqnarray}

Now we bound each term in Eq.~\eqref{eq:regret_refined_diff_event} respectively.

\textbf{a) First term:} $\sum_{t=1}^T \mathbb{P}\big\{(a_t^{(1)}, a_t^{(2)}) = (i,\sigma_{i(j)}), \hat{\zeta}^*(t) \geq \zeta^*, \mathcal{E}^{{p}}_{\sigma_{i(j)},i}(t), \mathcal{E}^{\theta}_{\sigma_{i(j)},i}(t) \big\}$

For the first term, we note that when $\hat{\zeta}^*(t) \geq \zeta^*$, the first candidate $a_t^{(1)}$ could be $i$ only when there exists a $j' \leq L_C+1$, such that $l_{\sigma_{i(j')},i}(t) \leq 1/2$. Thus,
\begin{eqnarray}
&&\mathbb{P}\big\{(a_t^{(1)}, a_t^{(2)}) = (i,\sigma_{i(j)}), \hat{\zeta}^*(t) \geq \zeta^*, \mathcal{E}^{{p}}_{\sigma_{i(j)},i}(t), \mathcal{E}^{\theta}_{\sigma_{i(j)},i}(t) \big\} \nonumber \\
&\leq & \sum_{j'=1}^{L_C+1} \mathbb{P}\big\{(a_t^{(1)}, a_t^{(2)}) = (i,\sigma_{i(j)}), \hat{\zeta}^*(t) \geq \zeta^*, l_{\sigma_{i(j')},i}(t) \leq 1/2, \mathcal{E}^{{p}}_{\sigma_{i(j)},i}(t), \mathcal{E}^{\theta}_{\sigma_{i(j)},i}(t) \big\}.\nonumber
\end{eqnarray}
For each $j' \leq L_C+1$, define the following probability:
 \begin{equation}
q^{(i)}_{j'j}(t) = \mathbb{P}\{\theta_{\sigma_{i(j')},i}^{(2)}(t) >  y_{\sigma_{i(j)},i} | \mathcal{H}_{t-1} \}.
\end{equation}
Similar to Lemma~\ref{thm:prob_bound}, we can show that
\begin{eqnarray}
&& \mathbb{P}\big\{(a_t^{(1)}, a_t^{(2)}) = (i,\sigma_{i(j)}), \hat{\zeta}^*(t) \geq \zeta^*, l_{\sigma_{i(j')},i}(t) \leq 1/2, \mathcal{E}^{{p}}_{\sigma_{i(j)},i}(t), \mathcal{E}^{\theta}_{\sigma_{i(j)},i}(t) |\mathcal{H}_{t-1}\big\} \nonumber \\
&\leq & \frac{1-q^{(i)}_{j'j}(t)}{q^{(i)}_{j'j}(t)} \mathbb{P}\big\{(a_t^{(1)}, a_t^{(2)}) = (i,\sigma_{i(j')}),  \mathcal{E}^{{p}}_{\sigma_{i(j)},i}(t), \mathcal{E}^{\theta}_{\sigma_{i(j)},i}(t) |\mathcal{H}_{t-1}\big\}, \nonumber
\end{eqnarray}
and its summation over $t$ can be bounded as
\begin{eqnarray}
&& \sum_{t=1}^T \mathbb{P}\big\{(a_t^{(1)}, a_t^{(2)}) = (i,\sigma_{i(j)}), \hat{\zeta}^*(t) \geq \zeta^*, l_{\sigma_{i(j')},i}(t) \leq 1/2, \mathcal{E}^{{p}}_{\sigma_{i(j)},i}(t), \mathcal{E}^{\theta}_{\sigma_{i(j)},i}(t) \big\}
 =  O(1). \nonumber
\end{eqnarray}
Considering all $j'$ from $1$ to $L_C+1$, we have
\begin{eqnarray} \label{eq:regret_refined_diff_event_term1}
\sum_{t=1}^T\mathbb{P}\big\{(a_t^{(1)}, a_t^{(2)}) = (i,\sigma_{i(j)}), \hat{\zeta}^*(t) \geq \zeta^*, \mathcal{E}^{{p}}_{\sigma_{i(j)},i}(t), \mathcal{E}^{\theta}_{\sigma_{i(j)},i}(t) \big\} = O(L_C+1).
\end{eqnarray}

\textbf{b) Second term:}  $\sum_{t=1}^T \mathbb{P}\big\{(a_t^{(1)}, a_t^{(2)}) = (i,\sigma_{i(j)}), \mathcal{E}^{{p}}_{\sigma_{i(j)},i}(t), \urcorner \mathcal{E}^{\theta}_{\sigma_{i(j)},i}(t) \big\}$

We use the back substitution argument to refine the second term to $O(\log \log T)$. When fixing the first candidate as $a_t^{(1)} = i$, the comparison between $a_t^{(1)}$ and other arms is similar to a traditional MAB. Let $N_i^{(1)}(T) = \sum_{t=1}^T\mathds{1}(a_t^{(1)} = i)$ be the number of time-slots when this type of MAB is played, and let
\begin{eqnarray}
L_{j i}^{\beta}(n) = \frac{\beta\log n}{D(x_{\sigma_{i(j)},i} || y_{\sigma_{i(j)},i})}. \nonumber
\end{eqnarray}
 Then, considering all possible cases of $N_i^{(1)}(T)$ and $N_{i \sigma_{i(j)}}^{(1)}(t-1)$, we have
 {\small
\begin{eqnarray}
&&  \sum_{t=1}^T \mathbb{P}\big\{(a_t^{(1)}, a_t^{(2)}) = (i,\sigma_{i(j)}), \mathcal{E}^{{p}}_{\sigma_{i(j)},i}(t), \urcorner \mathcal{E}^{\theta}_{\sigma_{i(j)},i}(t) \big\} \nonumber \\
&\leq & \sum_{n=0}^{T}\sum_{t=1}^T \mathbb{P}\big\{(a_t^{(1)}, a_t^{(2)}) = (i,\sigma_{i(j)}), \mathcal{E}^{{p}}_{\sigma_{i(j)},i}(t), \urcorner \mathcal{E}^{\theta}_{\sigma_{i(j)},i}(t), N_{i \sigma_{i(j)}}^{(1)}(t-1) \leq  L_{ji}^{\beta}(n),  N_i^{(1)}(T) = n \big\} \nonumber \\
&& + \sum_{n=0}^T \sum_{t=1}^T \mathbb{P}\big\{(a_t^{(1)}, a_t^{(2)}) = (i,\sigma_{i(j)}), \mathcal{E}^{{p}}_{\sigma_{i(j)},i}(t), \urcorner \mathcal{E}^{\theta}_{\sigma_{i(j)},i}(t), N_{i \sigma_{i(j)}}^{(1)}(t-1) >  L_{ji}^{\beta}(n),  N_i^{(1)}(T) = n \big\}. \nonumber
\end{eqnarray}
}
For the first case, note that
{\small
\begin{eqnarray}
&& \sum_{t=1}^T \mathbb{P}\big\{(a_t^{(1)}, a_t^{(2)}) = (i,\sigma_{i(j)}), \mathcal{E}^{{p}}_{\sigma_{i(j)},i}(t), \urcorner \mathcal{E}^{\theta}_{\sigma_{i(j)},i}(t), N_{i \sigma_{i(j)}}^{(1)}(t-1) \leq  L_{ji}^{\beta}(n),  N_i^{(1)}(T) = n \big\} \nonumber \\
& \leq  & \sum_{t=1}^T \mathbb{P}\big\{(a_t^{(1)}, a_t^{(2)}) = (i,\sigma_{i(j)}),  N_{i \sigma_{i(j)}}^{(1)}(t-1) \leq  L_{ji}^{\beta}(n) |  N_i^{(1)}(T) = n \big\} \mathbb{P}\{N_i^{(1)}(T) = n \} \nonumber \\
&\leq & L_{ji}^{\beta}(n) \mathbb{P}\{N_i^{(1)}(T) = n \},
\end{eqnarray}
}
similar to the analysis for Eq.~\eqref{eq:regret_greater_0p5_smallN}.

Then, we have
\begin{eqnarray}
&& \sum_{n=0}^{T}\sum_{t=1}^T \mathbb{P}\big\{(a_t^{(1)}, a_t^{(2)}) = (i,\sigma_{i(j)}), \mathcal{E}^{{p}}_{\sigma_{i(j)},i}(t), \urcorner \mathcal{E}^{\theta}_{\sigma_{i(j)},i}(t), N_{i j_{ik}}^{(1)}(t-1) \leq  L_{ji}^{\beta}(n),  N_i^{(1)}(T) = n \big\} \nonumber \\
& \leq  &  \mathbb{E}[L_{ji}^{\beta}(n)]
 \leq  \frac{\beta \log (\mathbb{E}[N_i^{(1)}(T)])}{D(x_{\sigma_{i(j)},i}|| y_{\sigma_{i(j)},i})}
 \leq   \frac{\beta \big(\log\log T + O(\log K)\big)}{D(x_{\sigma_{i(j)},i}|| y_{\sigma_{i(j)},i})}, \nonumber
\end{eqnarray}
where the last inequality follows from the concavity of the $\log(\cdot)$ function and the fact that $\mathbb{E}[N_i^{(1)}(T)] = O(K \log T)$ as shown when proving Proposition~\ref{thm:regret_copeland}.

For the second case, let $\tau^{(i)}_{m}$ be time-slot where $a_t^{(1)} = i$ for the $m$-th time and $\tau^{(i)}_{0} = 0$. Then
{\small
\begin{eqnarray}
&&  \sum_{t=1}^T \mathbb{P}\big\{(a_t^{(1)}, a_t^{(2)}) = (i,\sigma_{i(j)}), \mathcal{E}^{{p}}_{\sigma_{i(j)},i}(t), \urcorner \mathcal{E}^{\theta}_{\sigma_{i(j)},i}(t), N_{i \sigma_{i(j)}}^{(1)}(t-1) >  L_{ji}^{\beta}(n),  N_i^{(1)}(T) = n \big\} \nonumber \\
&\leq & \mathbb{E}\bigg[\sum_{t=1}^T \mathds{1}\big((a_t^{(1)}, a_t^{(2)}) = (i,\sigma_{i(j)}), \mathcal{E}^{{p}}_{\sigma_{i(j)},i}(t), \urcorner \mathcal{E}^{\theta}_{\sigma_{i(j)},i}(t), N_{i \sigma_{i(j)}}^{(1)}(t-1) >  L_{ji}^{\beta}(n),  N_i^{(1)}(T) = n \big) \bigg] \nonumber \\
& \leq & \sum_{m=0}^n \mathbb{E}\bigg[\sum_{t = \tau_m + 1}^{\tau_{m+1}} \mathds{1}\big((a_t^{(1)}, a_t^{(2)}) = (i,\sigma_{i(j)}), \mathcal{E}^{{p}}_{\sigma_{i(j)},i}(t), \urcorner \mathcal{E}^{\theta}_{\sigma_{i(j)},i}(t), N_{i \sigma_{i(j)}}^{(1)}(t-1) >  L_{ji}^{\beta}(n) \bigg] \nonumber \\
&\overset{(a)}{\leq} & \sum_{m=0}^n \mathbb{E}\bigg[\mathds{1}\big((a_{\tau_{m+1}}^{(1)}, a_{\tau_{m+1}}^{(2)}) = (i,\sigma_{i(j)}), \mathcal{E}^{{p}}_{\sigma_{i(j)},i}(\tau_{m+1}), \urcorner \mathcal{E}^{\theta}_{\sigma_{i(j)},i}(\tau_{m+1}), N_{i \sigma_{i(j)}}^{(1)}(\tau_{m+1}-1) >  L_{ji}^{\beta}(n) \bigg] \nonumber \\
&\leq & \sum_{m=0}^n \mathbb{P}\big\{\mathcal{E}^{{p}}_{\sigma_{i(j)},i}(t), \urcorner \mathcal{E}^{\theta}_{\sigma_{i(j)},i}(t), N_{i \sigma_{i(j)}}^{(1)}(t-1) >  L_{ji}^{\beta}(n) \big\} \nonumber \\
&\overset{(b)}{\leq} & n \cdot \frac{1}{n^{\beta}} = \frac{1}{n^{\beta - 1}},
\end{eqnarray}
}
where $(a)$ holds because $a_t^{(1)} = i$ could only happen at $t = \tau^{(i)}_{m+1}$; $(b)$ is true because, the two sets of samples in D-TS are drawn independently and their distributions only depend on the historic comparison results; thus, given $N_{i \sigma_{i(j)}}^{(1)}(t-1) >  L_{ji}^{\alpha}(n)$, the events  $\mathcal{E}^{{p}}_{\sigma_{i(j)},i}(t)$ and $\mathcal{E}^{\theta}_{\sigma_{i(j)},i}(t)$ are independent of $t = \tau^{(i)}_{m+1}$, and the probability can be bounded according to the concentration property of Thompson samples (in the proof of Lemma~3 in \cite{Agrawal2013AISTATS:TS2}):
\begin{equation}
\mathbb{P}\{\mathcal{E}^{{p}}_{\sigma_{i(j)},i}(t), \urcorner \mathcal{E}^{\theta}_{\sigma_{i(j)},i}(t), N_{i \sigma_{i(j)}}^{(1)}(t-1) >  L_{ji}^{\beta}(n)\} \leq e^{-L_{ji}^{\beta}(n) D(x_{\sigma_{i(j)},i} || y_{\sigma_{i(j)},i})} = \frac{1}{n^{\beta}}. \nonumber
\end{equation}
Then for $\beta > 2$,
\begin{eqnarray}
&& \sum_{n=0}^T \sum_{t=1}^T \mathbb{P}\big\{(a_t^{(1)}, a_t^{(2)}) = (i,\sigma_{i(j)}), \mathcal{E}^{{p}}_{\sigma_{i(j)},i}(t), \urcorner \mathcal{E}^{\theta}_{\sigma_{i(j)},i}(t), N_{i \sigma_{i(j)}}^{(1)}(t-1) >  L_{ji}^{\beta}(n),  N_i^{(1)}(T) = n \big\} \nonumber \nonumber \\
&\leq & \sum_{n=0}^T  \frac{1}{n^{\beta - 1}} = O(1).
\end{eqnarray}

Combining the above two cases, we can bound the second term of Eq.~\eqref{eq:regret_refined_diff_event} as:
\begin{eqnarray} \label{eq:regret_refined_diff_event_term2}
 \sum_{t=1}^T \mathbb{P}\big\{(a_t^{(1)}, a_t^{(2)}) = (i,\sigma_{i(j)}), \mathcal{E}^{{p}}_{\sigma_{i(j)},i}(t), \urcorner \mathcal{E}^{\theta}_{\sigma_{i(j)},i}(t) \big\}
\leq  \frac{\beta \big(\log\log T + O(\log K)\big)}{D(x_{\sigma_{i(j)},i}|| y_{\sigma_{i(j)},i})} + O(1).
\end{eqnarray}

\textbf{c) Third term:} $\sum_{t=1}^T \mathbb{P}\big\{(a_t^{(1)}, a_t^{(2)}) = (i,\sigma_{i(j)}), \urcorner \mathcal{E}^{{p}}_{\sigma_{i(j)},i}(t)\big\}$

For the third term, we can bound it as
\begin{eqnarray} \label{eq:regret_refined_diff_event_term3}
&& \sum_{t=1}^T \mathbb{P}\big\{(a_t^{(1)}, a_t^{(2)}) = (i,\sigma_{i(j)}), \urcorner \mathcal{E}^{{p}}_{\sigma_{i(j)},i}(t)\big\} \nonumber \\
&\leq & \sum_{t=1}^T \mathbb{P}\big\{(a_t^{(1)}, a_t^{(2)}) = (i,\sigma_{i(j)}), \urcorner \mathcal{E}^{{p}}_{\sigma_{i(j)},i}(t)\big\} \nonumber \\
&\leq & \frac{1}{D(x_{\sigma_{i(j)},i} || p_{\sigma_{i(j)},i})} + 1,
\end{eqnarray}
where the last inequality follows from  Lemma~3 in \cite{Agrawal2013AISTATS:TS2}.

Combining the analysis for all above three terms, we can bound Eq.~\eqref{eq:regret_refined_diff_event} by choosing appropriate $x_{\sigma_{i(j)},i}$ and $y_{\sigma_{i(j)},i}$. Specifically,
for any $\epsilon > 0$, similar to \cite{Agrawal2013AISTATS:TS2}, we can choose appropriate $x_{\sigma_{i(j)},i}$ and $y_{\sigma_{i(j)},i}$ such that $D(x_{\sigma_{i(j)},i}|| y_{\sigma_{i(j)},i}) = D(p_{\sigma_{i(j)},i}||p_{\sigma_{i(L_C+1)},i})/(1+\epsilon)^2$ and $\frac{1}{D(x_{\sigma_{i(j)},i} || p_{\sigma_{i(j)},i})} = O(\frac{1}{\epsilon^2})$. The conclusion of Lemma~\ref{thm:regret_refined_singlearm} then follows.

%


\section{Additional Experimental Results} \label{app:add_sim_results}

This appendix presents additional experimental results using both synthetic and real-world data, to further evaluate the performance of the proposed algorithms, in comparison to the state-of-the-art schemes.

Because the simulation complexity is $O(K^2 T)$, we adjust the number of independent experiments to save time, and run 500, 100, and 10 independent experiments  for $K < 10$, $10 \leq K \leq 100$, and $K > 100$, respectively.
In each experiment, the arms are randomly shuffled to prevent algorithms from exploiting special structures of the preference matrix, except for ECW-RMED in the ``Gap'' dataset, where we run experiments for both fixed and shuffled arm orders.

\subsection{Datasets}

\subsubsection{Condorcet Dueling Bandits}

\textbf{Cyclic:} A dataset adopted from \cite{Komiyama2015COLT:DB}, where the preference matrix is given by Table~\ref{tab:synthetic_cyclic}. In this dataset, Arm 1 is the Condorcet winner with $p_{1j} = 0.6$, and the other arms have a cyclic preference relationship with one arm beating another with high probability. Strong transitivity does not hold in this dataset, and the Condorcet winner is not necessary the best arm when comparing all other arms with a fixed arm, i.e., $p_{1i} < \max_{j} p_{ji}$ for $i \neq 1$.

\textbf{StrongBorda:} A 5-armed dueling bandit with a preference matrix in Table~\ref{tab:synthetic_strongborda}. In addition to  a Condorcet winner, there is a strong Borda winner, which is not the Condorcet winner, but beats the other arms with high probability. To validate the correctness of algorithms, we still treat this problem as a Condorcet dueling bandit problem and try to find the Condorcet winner, although a Borda winner may be more appropriate in this case \cite{Jamieson2015AISTAT:SDB}.

\textbf{ArXiv:} A 6-armed dueling bandits with a preference matrix given in Table~\ref{tab:arxiv_6rankers}, which is derived by conducting pairwise interleaving experiments \cite{Yue2011ICML:BTM} based on the search engine of ArXiv.org.

\textbf{Sushi:} A 16-armed dueling bandits with a preference matrix derived by \cite{Komiyama2015COLT:DB,Komiyama2016ICML:CWRMED} from a Sushi preference dataset, where the matrix can be found in the appendix of \cite{Komiyama2016ICML:CWRMED}.

\begin{table*}[thbp]
\small
\begin{center}
  \hspace{-0.5cm}
  \begin{minipage}[b]{.2\textwidth}
        \tabcolsep 4pt \caption{Cyclic}
        \vspace{-0.6cm}
        \begin{center}
        \def\temptablewidth{1.1\textwidth}
        {\rule{\temptablewidth}{1pt}}
        \begin{tabular}{c|cccc}
          & 1 &2 &3 &4  \\
        \hline
        1 & 0.5 &0.6 &0.6 &0.6  \\
        2 & 0.4 &0.5 &0.9 &0.1  \\
        3 & 0.4  &0.1 &0.5 &0.9 \\
        4 & 0.4  &0.9 &0.1 & 0.5
        \end{tabular}
        {\rule{\temptablewidth}{1pt}}
        \end{center}
        \label{tab:synthetic_cyclic}
        \end{minipage}
        \hspace{.4cm}
    \begin{minipage}[b]{.28\textwidth}
        \tabcolsep 4pt \caption{StrongBorda}
       \vspace{-0.5cm}
        \begin{center}
        \def\temptablewidth{1.18\textwidth}
        {\rule{\temptablewidth}{1pt}}
        \begin{tabular}{c|ccccc}
          & 1 &2 &3 &4 &5  \\
        \hline
        1 & 0.5 &0.55 &0.55 &0.55 &0.55  \\
        2 & 0.45 &0.5 &0.95 &0.95 &0.95  \\
        3 & 0.45  &0.05 &0.5 &0.95 & 0.95\\
        4 & 0.45  &0.05 &0.05 & 0.5 & 0.95\\
        5 & 0.45  &0.05 &0.05 & 0.05 & 0.5\\
        \end{tabular}
        {\rule{\temptablewidth}{1pt}}
        \end{center}
        \label{tab:synthetic_strongborda}
        \end{minipage}
        \hspace{0.8cm}
           \begin{minipage}[b]{.4\textwidth}
        \tabcolsep 4pt \caption{ArXiv}
       \vspace{-0.2cm}
        \begin{center}
        \def\temptablewidth{.98\textwidth}
        {\rule{\temptablewidth}{1.0pt}}
        \begin{tabular}{c | c c c c c c}
          & 1 &2 &3 &4 &5 & 6 \\
        \hline
        1 & 0.50 & 0.55 & 0.55 & 0.54 & 0.61 & 0.61  \\
        2 & 0.45 & 0.50 & 0.55 & 0.55 & 0.58 & 0.60  \\
        3 & 0.45 & 0.45 & 0.50 & 0.54 & 0.51 & 0.56 \\
        4 & 0.46 & 0.45 & 0.46 & 0.50 & 0.54 & 0.50 \\
        5 & 0.39 & 0.42 & 0.49 & 0.46 & 0.50 & 0.51\\
        6 & 0.39 & 0.40 & 0.44 & 0.50 & 0.49 & 0.50
        \end{tabular}
        {\rule{\temptablewidth}{1.0pt}}
        \end{center}
        \label{tab:arxiv_6rankers}
        \end{minipage}
 \end{center}
\vspace{-0.5cm}
\end{table*}

\subsubsection{Non-Condorcet Dueling Bandits}

\textbf{Non-Condorcet Cyclic:} A 9-armed dueling bandit with a preference matrix given by Table~\ref{tab:noncondorcet_cyclic_9}. In this dataset, there is a cyclic preference relationship among arms, and the arms can be divided into 3 groups with Copeland scores 6, 4, and 2, respectively. Due to this cyclic symmetry, there are multiple Copeland winners with exactly the same performance.

\textbf{Non-Condorcet StrongBorda:} Similar to the Condorcet dueling bandits, we consider this case where there is a Borda winner different from the Copeland winner. In this non-Condorcet setting, we even assume that when comparing the Copeland winner and the Borda winner, the user prefers the Borda winner to the Copeland winner with high probability. Again, this is a extreme case used to the validate the correctness of algorithms.

\textbf{Gap:} A dataset adopted from \cite{Komiyama2016ICML:CWRMED}, which is a 5-armed dueling bandits with a preference matrix given by Table~\ref{tab:synthetic_gap}. In this dataset, the ratio between the regret bound for ECW-RMED and the regret bound for the optimal CW-RMED algorithm is very large.

\textbf{500-Armed Dueling Bandits:} The 500-armed dueling bandit constructed in \cite{Zoghi2015NIPS:CDB}, where there are three Copeland winners that form a cycle and each has a Copeland score 498, and the other arms have Copeland scores ranging from 0 to 496. We use this dataset to evaluate the scaling behaviors of the algorithms. 

\begin{table*}[thbp]
\small
\begin{center}
\begin{minipage}[b]{.32\textwidth}
        \tabcolsep 4pt \caption{Gap}
        \begin{center}
        \def\temptablewidth{0.91\textwidth}
        {\rule{\temptablewidth}{1.0pt}}
       \begin{tabular}{c | c c c c c  }
          & 1 &2 &3 &4 &5 \\
        \hline
        1 & 0.5 &0.8 &0.8 &0.51 &0.2 \\
        2 & 0.2 &0.5 &0.8 &0.2 &0.8 \\
        3 & 0.2 &0.2 &0.5 &0.8 &0.8 \\
        4 &0.49 &0.8 &0.2 &0.5 &0.2 \\
        5 &0.8 &0.2 &0.2 &0.8 &0.5 \\
        \end{tabular}
        {\rule{\temptablewidth}{1.0pt}}
        \end{center}
        \label{tab:noncondorcet_strongborda}
        \end{minipage}
        \hspace{.5cm}
    \begin{minipage}[b]{.42\textwidth}
        \tabcolsep 4pt \caption{Non-Condorcet StrongBorda}
        \begin{center}
        \def\temptablewidth{0.91\textwidth}
        {\rule{\temptablewidth}{1.0pt}}
       \begin{tabular}{c | c c c c c c }
          & 1 &2 &3 &4 &5 & 6\\
        \hline
        1 & 0.5 & 0.05 & 0.55 & 0.55 & 0.55 & 0.55  \\
        2 & 0.95 & 0.5 & 0.95 & 0.95 & 0.45 & 0.45  \\
        3 & 0.45 & 0.05 & 0.5 & 0.95 & 0.95 & 0.95 \\
        4 & 0.45 & 0.05 & 0.05 & 0.5 & 0.95 & 0.95 \\
        5 & 0.45 & 0.55 & 0.05 & 0.05 & 0.5 & 0.95\\
        6 & 0.45 & 0.55 & 0.05 & 0.05 & 0.05 & 0.5
        \end{tabular}
        {\rule{\temptablewidth}{1.0pt}}
        \end{center}
        \label{tab:synthetic_gap}
        \end{minipage}
    \begin{minipage}[b]{.52\textwidth}
        \tabcolsep 4pt \caption{Non-Condorcet Cyclic}
        \begin{center}
        \def\temptablewidth{0.906\textwidth}
        {\rule{\temptablewidth}{1.0pt}}
       \begin{tabular}{c | c c c c c c c c c}
          & 1 &2 &3 &4 &5 & 6 &7 &8 &9\\
        \hline
        1 & 0.5 & 0.4 & 0.6 & 0.1 & 0.6 & 0.6 & 0.6 & 0.6 & 0.6 \\
        2 & 0.6 & 0.5 & 0.4 & 0.6 & 0.1 & 0.6 & 0.6 & 0.6 & 0.6 \\
        3 & 0.4 & 0.6 & 0.5 & 0.6 & 0.6 & 0.1 & 0.6 & 0.6 & 0.6\\
        4 & 0.9 & 0.4 & 0.4 & 0.5 & 0.1 & 0.9 & 0.6 & 0.6 & 0.4 \\
        5 & 0.4 & 0.9 & 0.4 & 0.9 & 0.5 & 0.1 & 0.4 & 0.6 & 0.6\\
        6 & 0.4 & 0.4 & 0.9 & 0.1 & 0.9 & 0.5 & 0.6 & 0.4 & 0.6\\
        7 & 0.4 & 0.4 & 0.4 & 0.4 & 0.6 & 0.4  &0.5 & 0.1 & 0.9 \\
        8 & 0.4 & 0.4 & 0.4 & 0.4 & 0.4 & 0.6  & 0.9 & 0.5 & 0.1\\
        9 & 0.4 & 0.4 & 0.4 & 0.6 & 0.4 & 0.4   & 0.1 & 0.9 & 0.5 \\
        \end{tabular}
        {\rule{\temptablewidth}{1.0pt}}
        \end{center}
        \label{tab:noncondorcet_cyclic_9}
        \end{minipage}
 \end{center}
\vspace{-0.1cm}
\end{table*}

\subsubsection{MSLR (Condorcet and non-Condorcet)}

For the MSLR dataset, we have evaluated the algorithms in the two 5-arm cases in Section~\ref{sec:sim_results}, where the preference matrices are given by Tables~\ref{tab:mslr_5_condorcet} and \ref{tab:mslr_5_non_condorcet}. In this appendix, we also run experiments for larger scale dueling bandits, $K = 16$ and $32$, consisting of arms randomly selected from the 136 rankers (in the \emph{MSLR\_Informational\_PMat.npz} file on \url{http://bit.ly/nips15data}, \cite{Zoghi2015NIPS:CDB}; to see the asymptotic performance within an acceptable $T$, we eliminate the arms with $|p_{ij} - 1/2| < 0.003$).
The indices of the chosen rankers and their Copeland scores are presented in Table~\ref{tab:mslr_all}, where the arms are indexed from 1.
Note that due to the randomness, our chosen rankers are likely different from those in \cite{Komiyama2015COLT:DB} and \cite{Komiyama2016ICML:CWRMED}, even for the same $K$.

\begin{table*}[thbp]
\small
\begin{center}
  \hspace{-0.5cm}
    \begin{minipage}[b]{.4\textwidth}
        \tabcolsep 4pt \caption{MSLR {\small ($K = 5$, Condorcet)}}
       \vspace{-0.2cm}
        \begin{center}
        \def\temptablewidth{.98\textwidth}
        {\rule{\temptablewidth}{1pt}}
        \begin{tabular}{c|ccccc}
          & 1 &2 &3 &4 &5  \\
        \hline
        1 & 0.500  &  0.535  & 0.613  & 0.757  &0.765 \\
        2 & 0.465 & 0.500    & 0.580  & 0.727  &0.738 \\
        3 & 0.387 & 0.420  & 0.500    & 0.659  &0.669 \\
        4 & 0.243 & 0.273  & 0.341  & 0.500    &0.510 \\
        5 & 0.235 & 0.262  & 0.331  & 0.490  &0.500
        \end{tabular}
        {\rule{\temptablewidth}{1pt}}
        \end{center}
        \label{tab:mslr_5_condorcet}
        \end{minipage}
        \hspace{1cm}
           \begin{minipage}[b]{.4\textwidth}
        \tabcolsep 4pt \caption{MSLR {\small($K = 5$, non-Condorcet)}}
       \vspace{-0.2cm}
         \begin{center}
        \def\temptablewidth{.98\textwidth}
        {\rule{\temptablewidth}{1pt}}
        \begin{tabular}{c|ccccc}
          & 1 &2 &3 &4 &5  \\
        \hline
        1 & 0.500  &0.484  &0.519  &0.529  &0.518 \\
        2 & 0.516 & 0.500  &0.481  &0.530  &0.539 \\
        3 & 0.481 & 0.519&  0.500  &0.504  &0.512 \\
        4 & 0.471 & 0.470&  0.496 & 0.500  &0.503 \\
        5 & 0.482 & 0.461  & 0.488  &0.497  &0.500 \\
        \end{tabular}
        {\rule{\temptablewidth}{1pt}}
        \end{center}
        \label{tab:mslr_5_non_condorcet}
        \end{minipage}
 \end{center}
\vspace{-0.5cm}
\end{table*}

\begin{table}[thbp]
\small
\begin{center}
        \tabcolsep 4pt \caption{MSLR {\small ($K = 16$ and $32$, arms are indexed from 1) } }
       \vspace{-0.2cm}
        \begin{center}
        \def\temptablewidth{.92\textwidth}
        {\rule{\temptablewidth}{1.0pt}}
        \begin{tabular}{c  c c}
        Subset name  & Chosen rankers (Copeland score) & Winners (Copeland score) \\
        \hline
        K = 16, Condorcet & {\tiny \begin{tabular}{@{}c@{}}10(0), 22(6), 36(12), 58(9), 59(10), 66(5), 67(3), 68(4), \\ 77(2), 98(1), 109(11), 112(7), 115(15), 116(13), 117(8), 125(14)\end{tabular} } & 115(15) \\
        \hline
        K = 32, Condorcet & {\tiny \begin{tabular}{@{}c@{}}7(9), 21(26), 24(12), 28(5), 32(13), 35(24), 37(15), 38(19),
                                                         \\43(6), 44(4), 45(10), 48(21), 52(3), 56(28), 61(29), 68(8),
                                                         \\71(27), 73(20), 82(14), 85(25), 87(17), 91(11), 96(1), 99(2),
                                                         \\100(0), 105(7), 112(16), 115(31), 117(18), 121(30), 123(22), 133(23)\end{tabular} } & 115(31)  \\
        \hline
        K = 16, non-Condorcet  & {\tiny \begin{tabular}{@{}c@{}}2(6), 21(11), 42(4), 60(14), 67(5), 70(8), 79(3), 90(12),
                                                         \\ 91(7), 98(0), 99(1), 103(2), 106(14), 109(10), 117(9), 130(14) \end{tabular} } & \small{60(14), 106(14), 130(14)} \\
        \hline
        K = 32, non-Condorcet  & {\tiny \begin{tabular}{@{}c@{}}4(8), 8(4), 12(24), 16(1), 20(0), 24(13), 28(7), 32(14), \\
                                                                36(25), 40(26), 44(6), 48(18), 52(5), 56(28), 60(30), 64(22), \\
                                                                68(10), 72(16), 76(19), 80(21), 84(12), 88(17), 92(9), 96(3), \\
                                                                100(2), 104(11), 108(23), 112(15), 116(28), 120(30), 124(20), 128(29) \end{tabular} } & 60(30), 120(30)
        \end{tabular}
        {\rule{\temptablewidth}{1.0pt}}
        \end{center}
        \label{tab:mslr_all}
 \end{center}
\vspace{-0.5cm}
\end{table}

\subsection{Performance Comparisons}

We first analyze the regret performance of algorithms for Condorcet and non-Condorcet dueling bandits, respectively, and then discuss their robustness with respect to preference matrix and delayed feedback, as well as the impact of RUCB/RLCB elimination.
\subsubsection{Condorcet Dueling Bandits}

Fig.~\ref{fig:regret_condorcet} shows the cumulative regret of all algorithms in Condorcet dueling bandits. From Figs.~\ref{fig:Synthetic_Cyclic} to \ref{fig:MSLR_Informational_32_Condorcet}, we can see that D-TS and D-TS$^+$ achieve similar performance, because there is a unique winner in the system and few ties occur when choosing the first candidate.
Compared to existing algorithms, D-TS and D-TS$^+$ perform much better, except for the StrongBorda dataset, as discussed later.

Compared with the earlier trial RCS, our D-TS and D-TS$^+$ algorithms lead to more extensive utilization of TS and significantly reduce the regret. Although the RCS algorithm achieves better performance than RUCB and CCB in real-world datasets (Figs.~\ref{fig:ArXiv_6_rankers} to \ref{fig:MSLR_Informational_32_Condorcet}) by leveraging TS for selecting the first candidate, the improvement is limited. This is because RCS requires a ``100\%-pass'' when using ``majority voting'' to select the first candidate, i.e., the arm that beats all the other arms with respect to the samples, and randomly picks one if no such an arm exists. Thus, it may miss many opportunities to choose the Condorcet winner as the first candidate. In contrast, using RUCB-based elimination when choosing the first candidate, D-TS/D-TS$^+$ only need a simple ``majority voting'' rule. Thus, under D-TS/D-TS$^+$,  the Condorcet winner can be chosen as the first candidate even it is beaten by (a few) other arms with respect to the samples. In addition, without requiring ``100\%-pass'', D-TS and D-TS$^+$ directly apply to general Copeland dueling bandits. Moreover, by launching another round of sampling for selecting the second candidate, D-TS and D-TS$^+$ further reduce the regret.

Compared with RMED1 and ECW-RMED, D-TS and D-TS$^+$ still perform better. Note that without knowing the existence of a Condorcet winner, ECW-RMED performs slightly worse than RMED1, while achieving similar asymptotic performance. RMED1 has been shown to achieve the optimal asymptotic performance, in both asymptotic order $O(\log T)$ and \emph{coefficients} \cite{Komiyama2015COLT:DB} \footnote{The optimality of RMED1 is shown under another definition of regret \cite{Komiyama2015COLT:DB}, which depends on the probability-gap-to-the-winner (G2W). The trend based on the G2W regret is similar to that in this paper, except for the StrongBorda dataset.}. Interestingly, D-TS and D-TS$^+$ not only approach the similar asymptotic performance as RMED1 and ECW-RMED, but also achieve a smaller \emph{constant term}. This is because, at the beginning stage, when the empirical estimates for $p_{ij}$ deviate from the true value, RMED-type algorithms may temporally trap in exploring the non-winner arms and result in a very large constant term. In contrast, by TS, the best arm always has a positive probability to be explored \cite{Gopalan2014ICML:TS}, and will be identified as the winner much faster. This property makes TS-type algorithms better for the scenarios where the system statistics are not stationary and slowly varying over time.

For the StrongBorda dataset, Fig.~\ref{fig:Synthetic_StrongBorda} shows that D-TS and D-TS$^+$ can escape from the suboptimal comparisons and achieve very good performance compared to existing algorithms, except for (Condorcet-)SAVAGE. Somewhat surprisingly, SAVAGE performs best in this dataset with respect to the Copeland-score-based regret. This is because in this dataset, all arms except for the Condorcet winner (arm 1) are beaten by the Borda winner (arm 2) with a preference probability close to 1. Thus, arms 3 to 5 can be eliminated without causing too much loss by comparing with arm 2. This property in fact leads to a very small lower regret bound according to \cite{Komiyama2016ICML:CWRMED}. As an ``explore-then-exploit'' algorithm and with the awareness of the existence of a Condorcet winner, SAVAGE can eliminate the suboptimal arms quickly and achieve low regret. ECW-RMED does not perform very well in this dataset as its regret bound is much higher than the optimal lower bound. Note that these results do not conflict with the optimality of RMED1 in Condorcet dueling bandits with respect to the regret defined in \cite{Komiyama2015COLT:DB}. We can see that SAVAGE approaches the same asymptotic performance to RMED1 if we compare them based probability-gap-to-the-winner regret (for other datasets, the trend and relative relationship is similar for both definitions of regret, and the results are omitted here).

\begin{figure}[thbp]
\begin{center}
\subfigure[Cyclic]{\includegraphics[angle = 0,width = 0.49\linewidth]{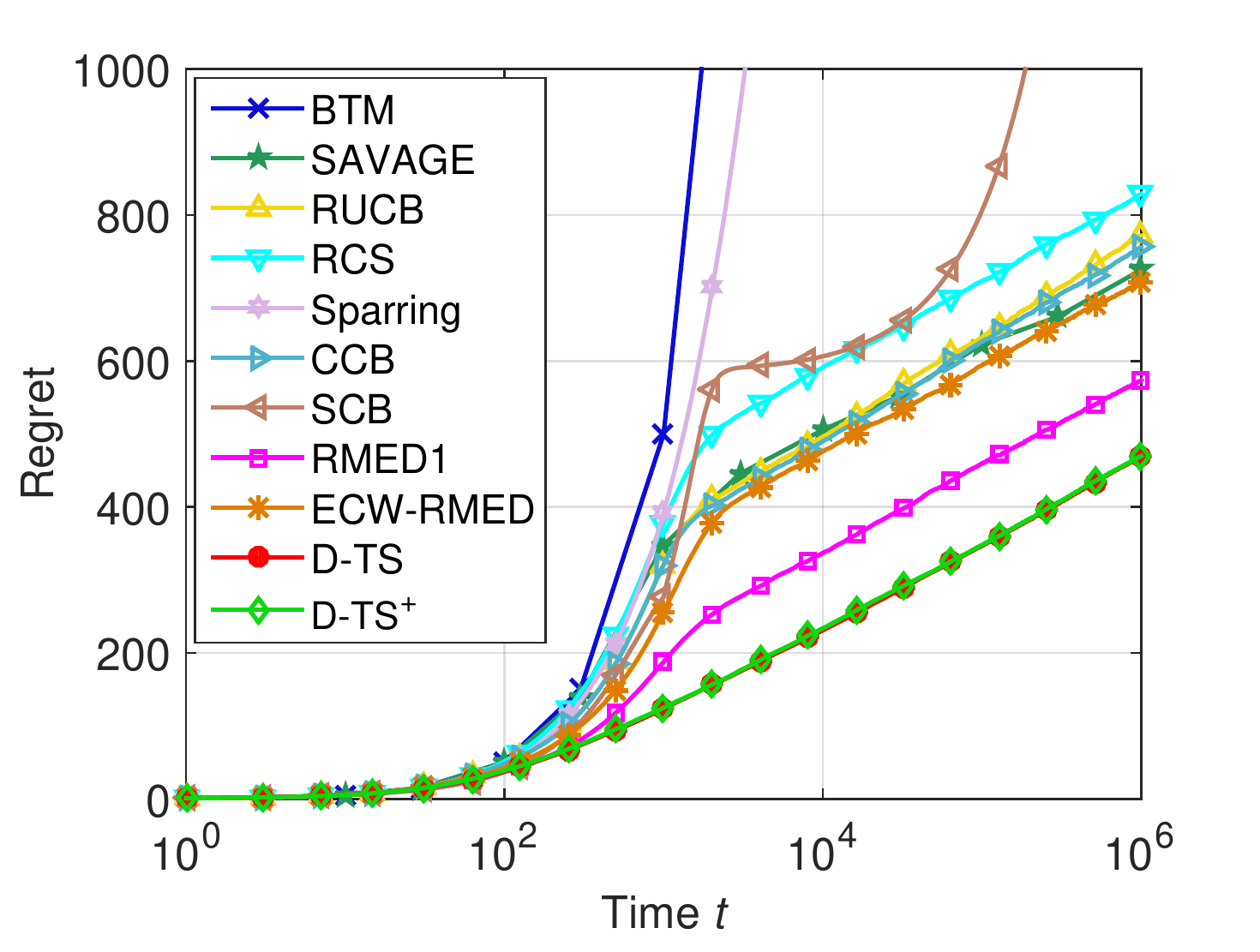}
\label{fig:Synthetic_Cyclic}}
\subfigure[StrongBorda]{\includegraphics[angle = 0,width = 0.49\linewidth]{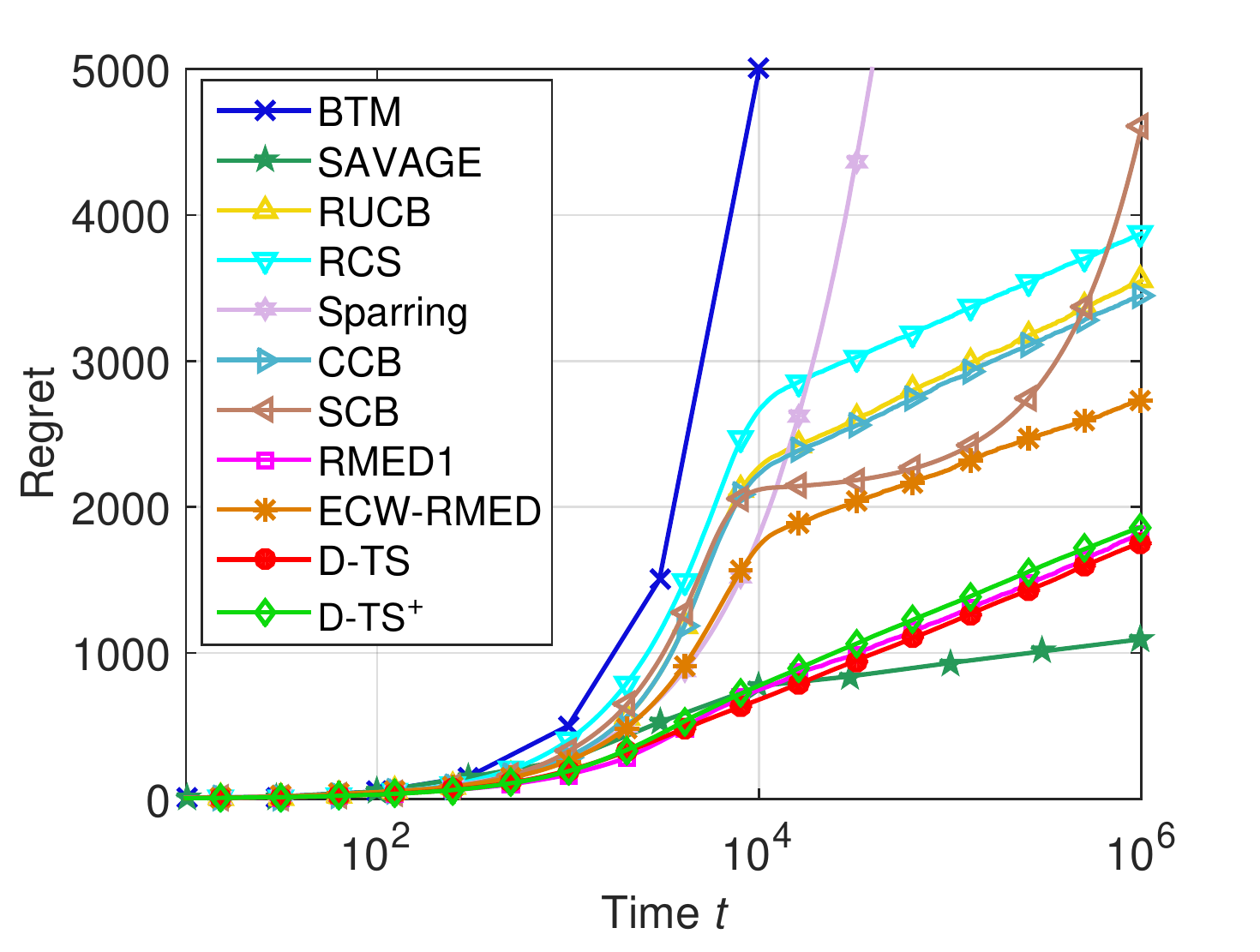}
\label{fig:Synthetic_StrongBorda}}
\subfigure[ArXiv]{\includegraphics[angle = 0,width = 0.49\linewidth]{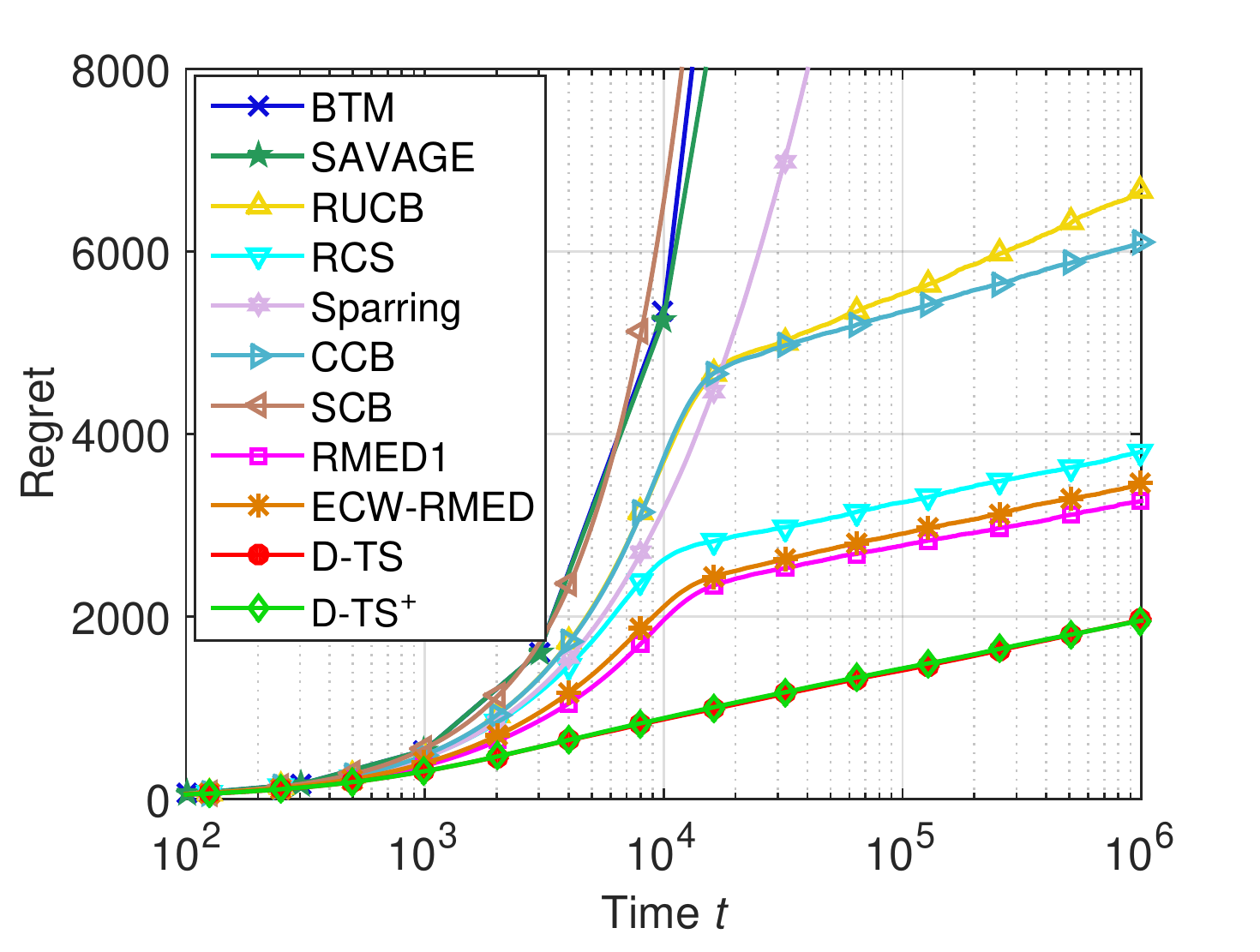}
\label{fig:ArXiv_6_rankers}}
\subfigure[Sushi]{\includegraphics[angle = 0,width = 0.49\linewidth]{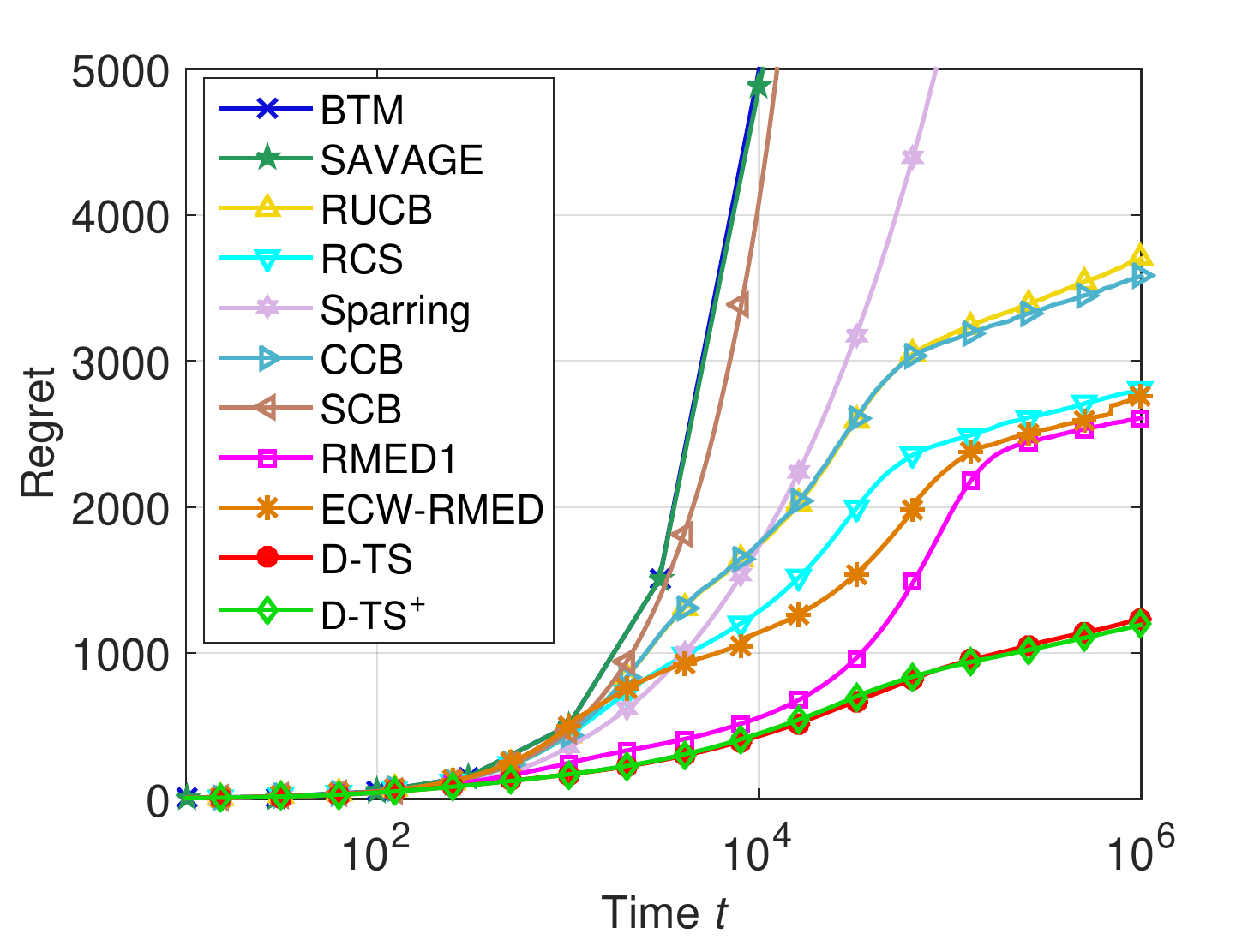}
\label{fig:Sushi}}
\subfigure[MSLR ($K = 16$, Condorcet)]{\includegraphics[angle = 0,width = 0.49\linewidth]{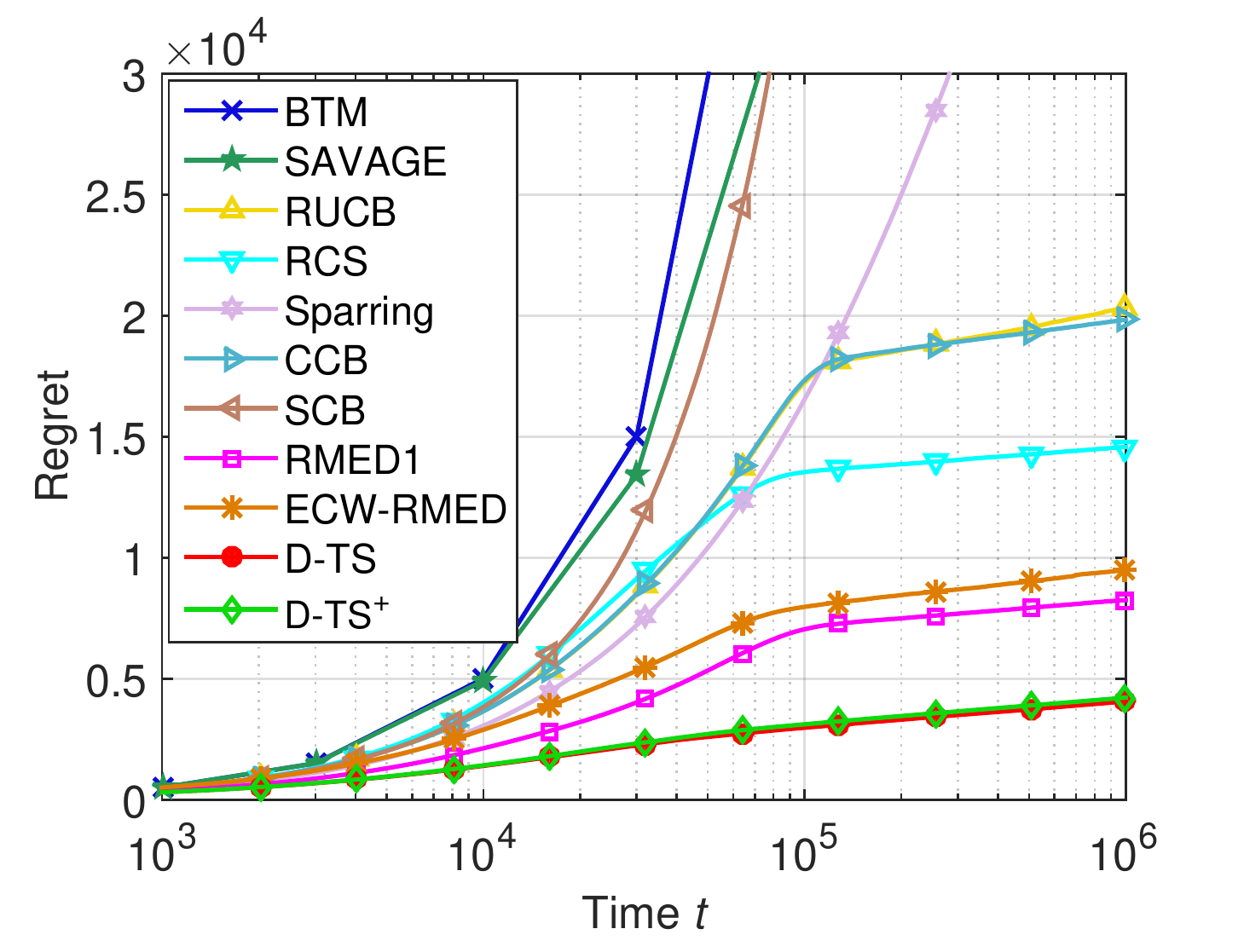}
\label{fig:MSLR_Informational_16_Condorcet}}
\subfigure[MSLR ($K = 32$, Condorcet)]{\includegraphics[angle = 0,width = 0.49\linewidth]{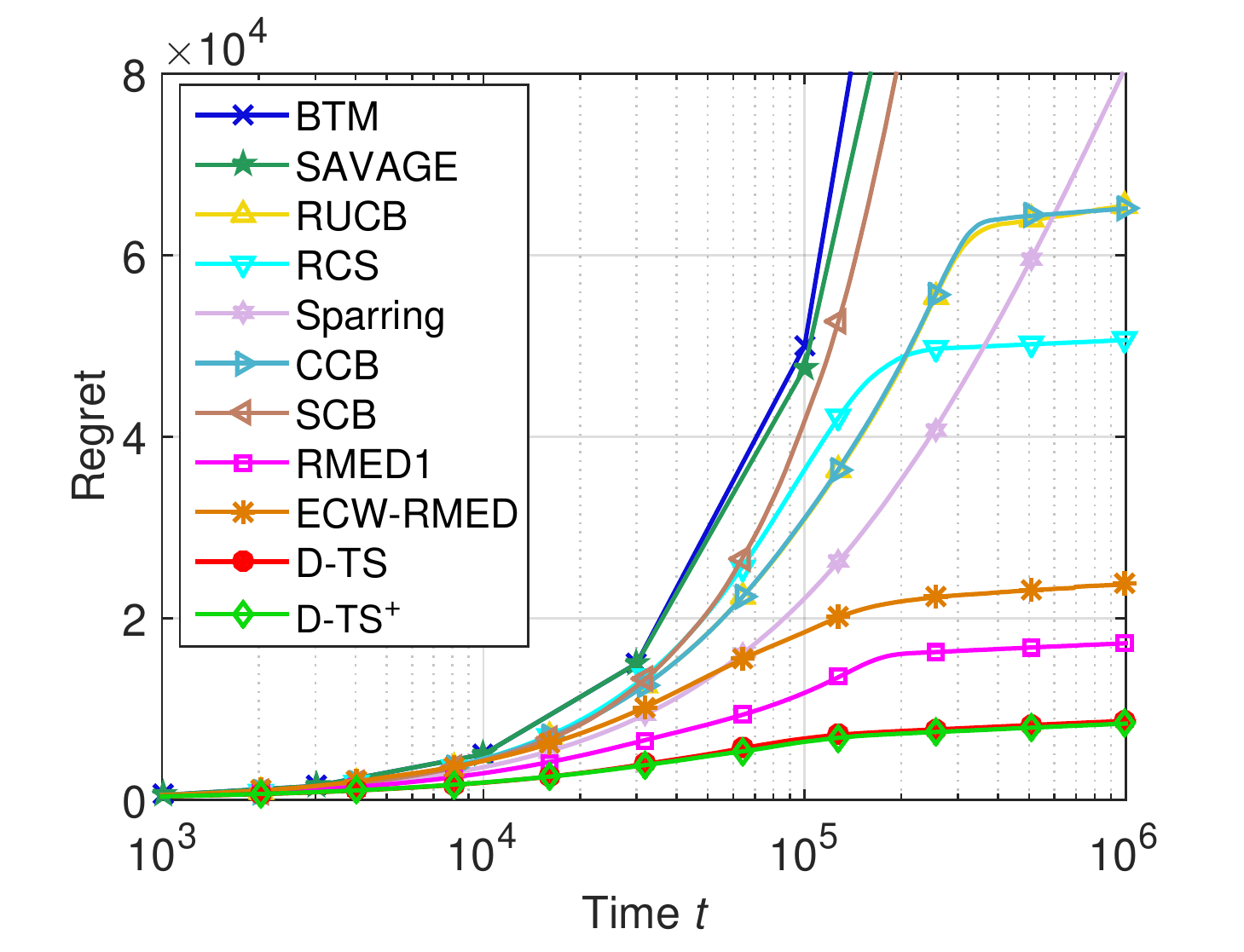}
\label{fig:MSLR_Informational_32_Condorcet}}
\caption{Cumulative regret in Condorcet dueling bandits.}
\label{fig:regret_condorcet}
\end{center}
\end{figure}


\subsubsection{Non-Condorcet Dueling Bandits}

Fig.~\ref{fig:regret_noncondorcet} presents the cumulative regret in non-Condorcet dueling bandits. Note that one simulation in the large scale 500-armed dueling bandit takes very long time, and we only compare CCB, SCB, ECW-RMED, and D-TS, for this dataset.

From Fig.~\ref{fig:regret_noncondorcet}, we can see that the regret of algorithms dedicated to Condorcet dueling bandits, including BTM, RUCB, RCS, and RMED1, grows rapidly, because these algorithms keep exploring for the Condorcet winner, which does not exist in these datasets. The Sparring algorithm, whose regret is not theoretically guaranteed, also results in a very large regret in non-Condorcet dueling bandits. The SAVAGE algorithm requires to explore all pairs, and usually leads to large regret.

Compared D-TS and D-TS$^+$ with the UCB-type algorithms for general Copeland dueling bandits,  CCB and SCB, we can see that D-TS and D-TS$^+$ perform much better. In particular, as shown in Fig.~\ref{fig:Synthetic_nonCondorcet_500} for the 500-armed scenario, D-TS still achieves much lower regret than SCB, the scalable version of CCB for large scale systems. 

\begin{figure}[thbp]
\begin{center}
\subfigure[Non-Condorcet Cyclic]{\includegraphics[angle = 0,width = 0.49\linewidth]{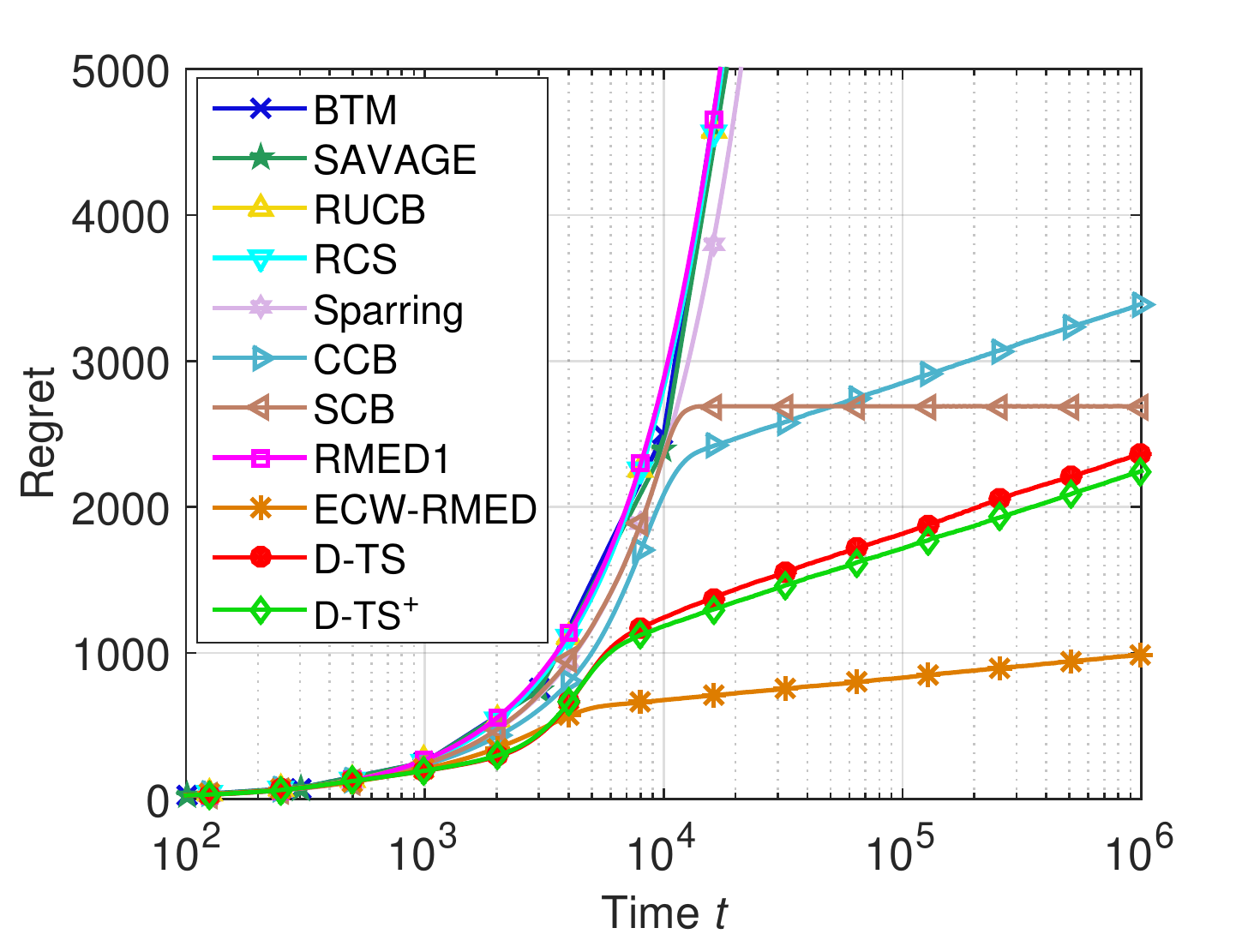}
\label{fig:Synthetic_nonCondorcet_Cyclic_9}}
\subfigure[Non-Condorcet StrongBorda]{\includegraphics[angle = 0,width = 0.49\linewidth]{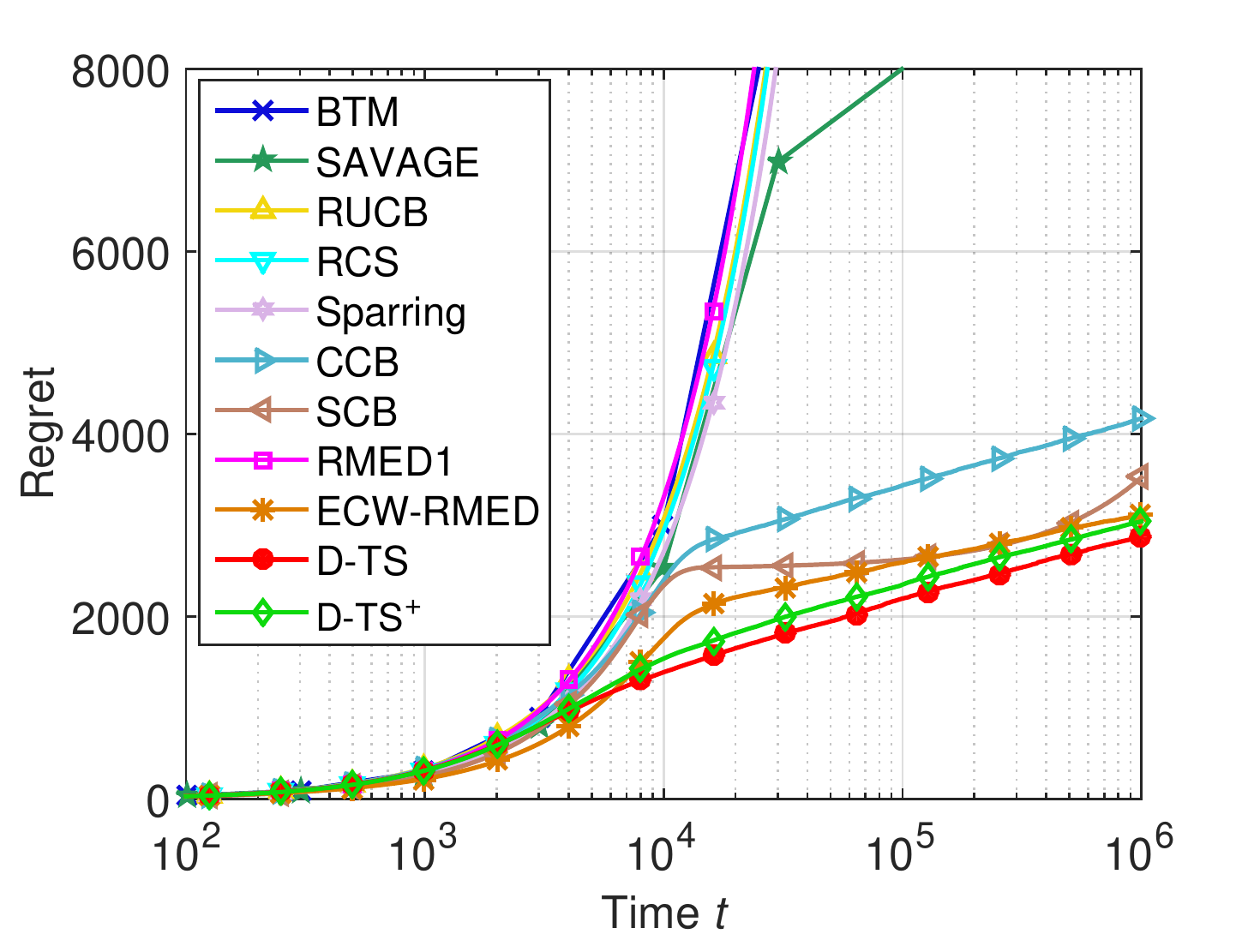}
\label{fig:Synthetic_nonCondorcet_StrongBorda}}
\subfigure[Gap]{\includegraphics[angle = 0,width = 0.49\linewidth]{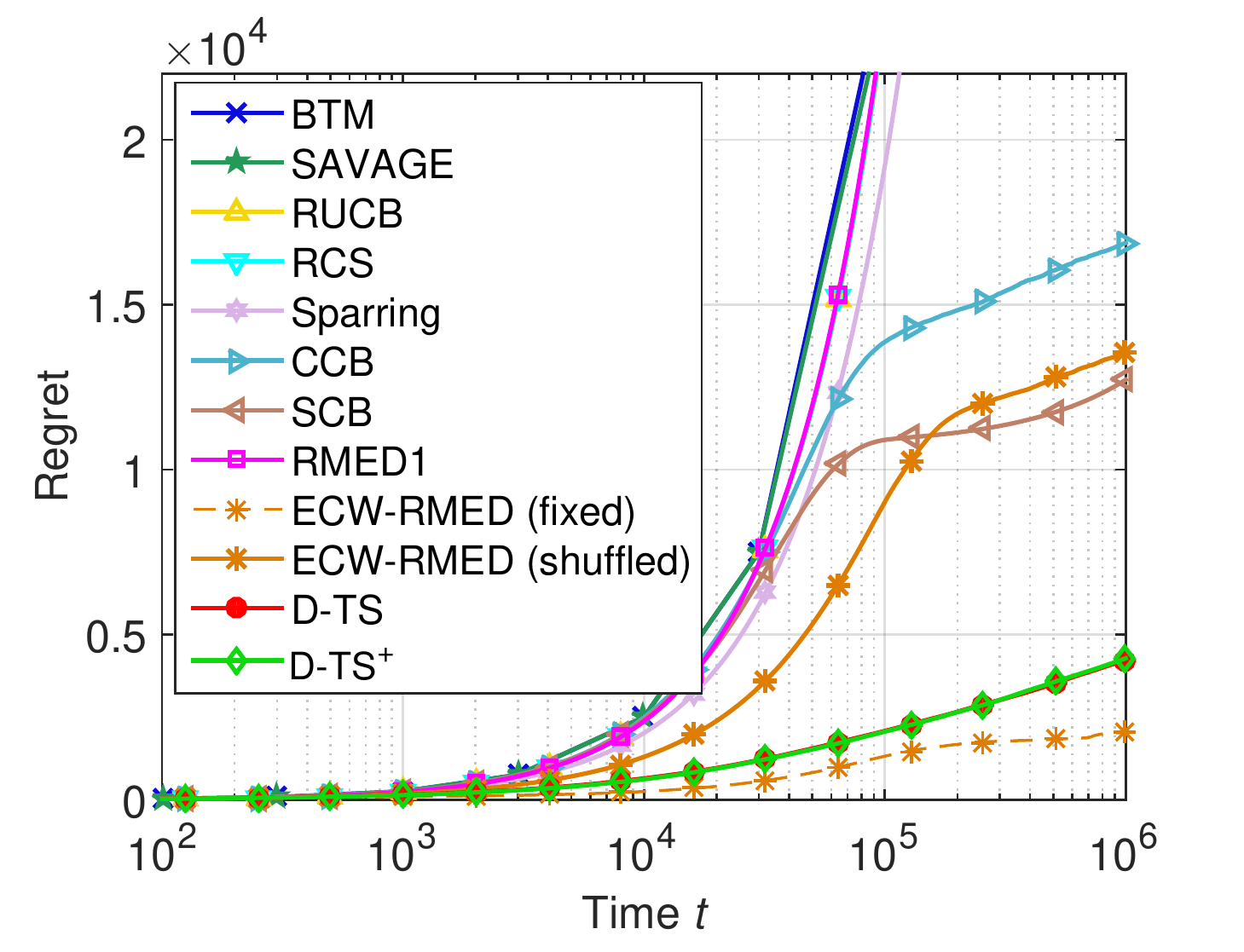}
\label{fig:Synthetic_Gap}}
\subfigure[500-Armed non-Condorcet Dueling Bandits]{\includegraphics[angle = 0,width = 0.49\linewidth]{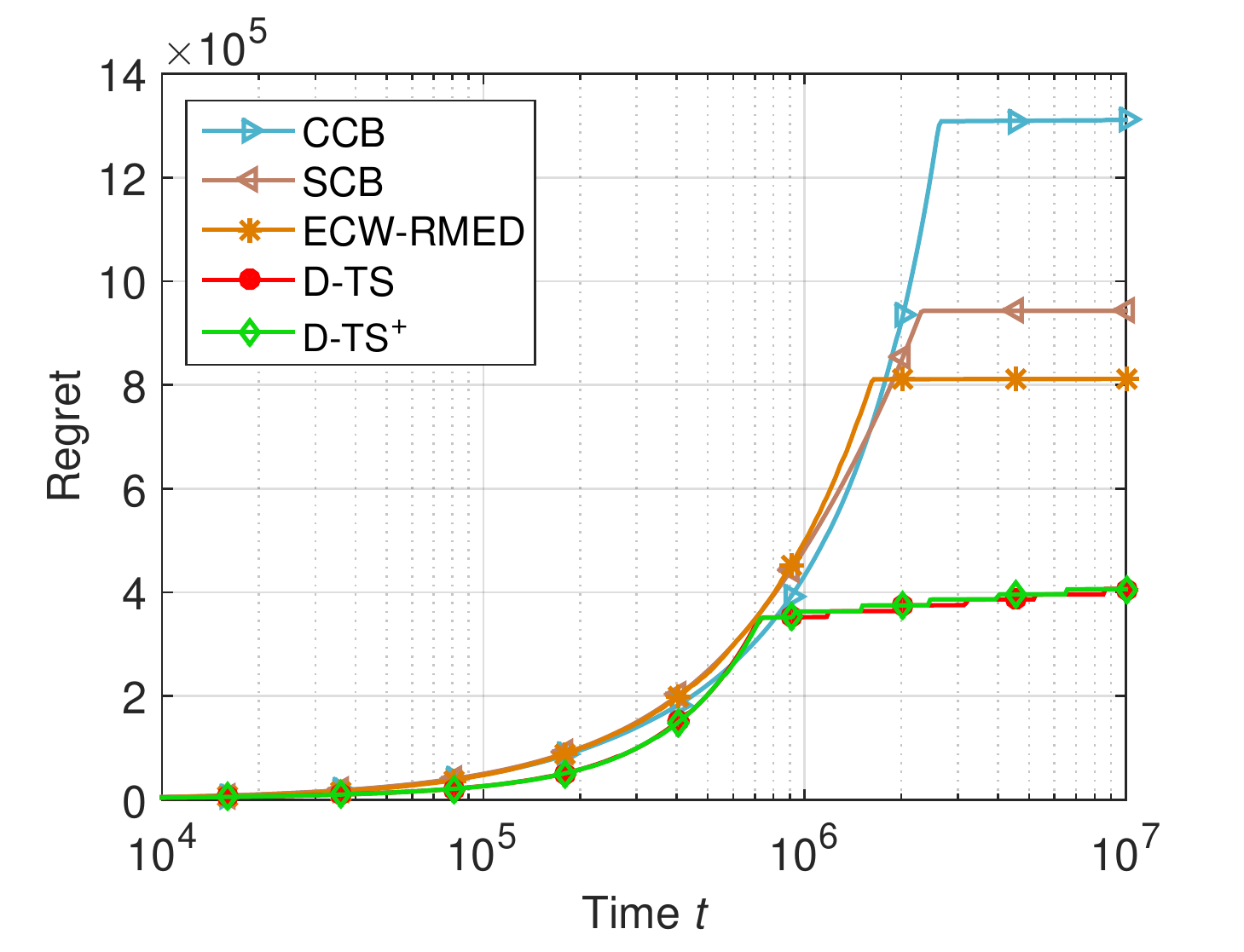}
\label{fig:Synthetic_nonCondorcet_500}}
\subfigure[MSLR ($K = 16$, non-Condorcet)]{\includegraphics[angle = 0,width = 0.49\linewidth]{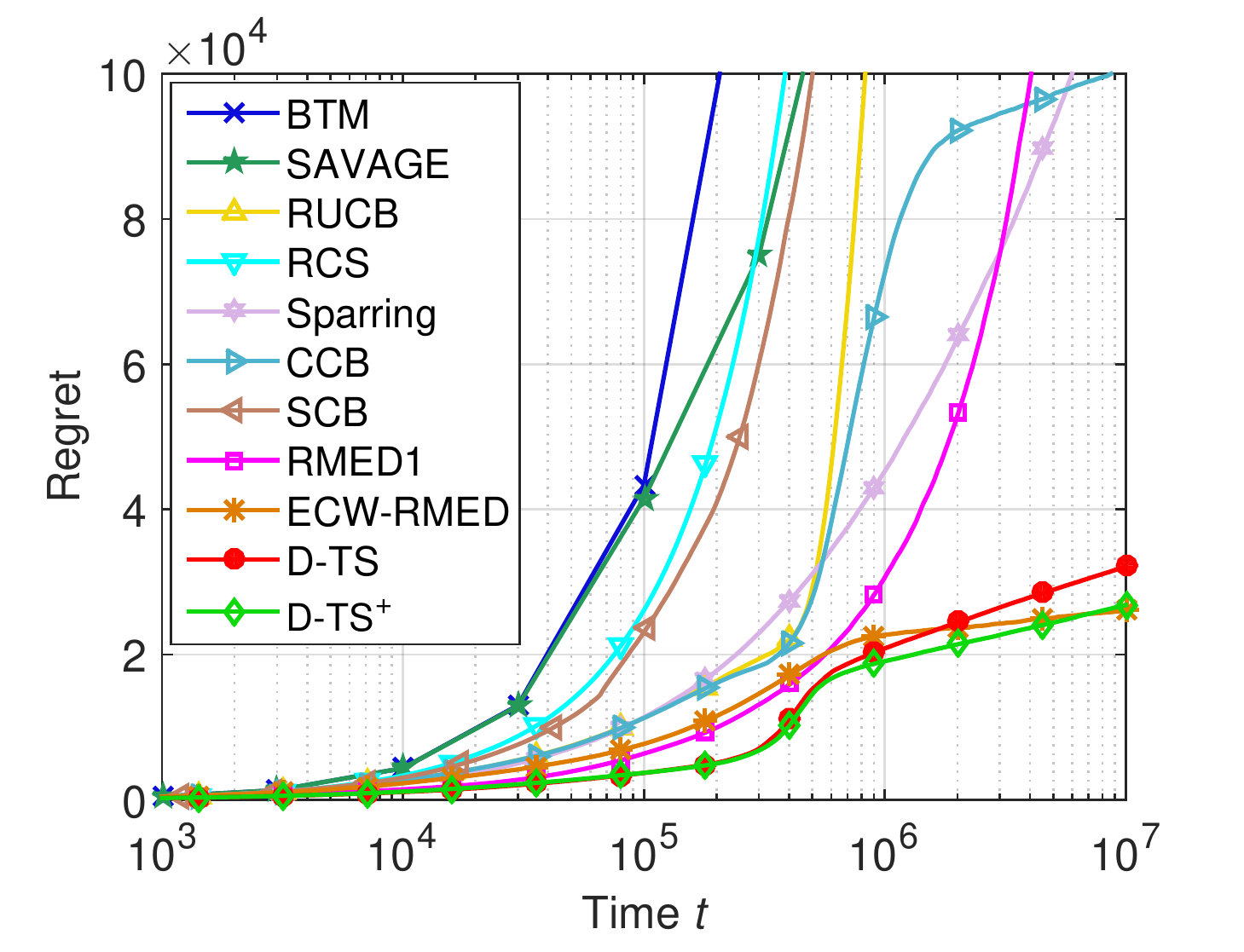}
\label{fig:MSLR_Informational_16_non_Condorcet}}
\subfigure[MSLR ($K = 32$, non-Condorcet)]{\includegraphics[angle = 0,width = 0.49\linewidth]{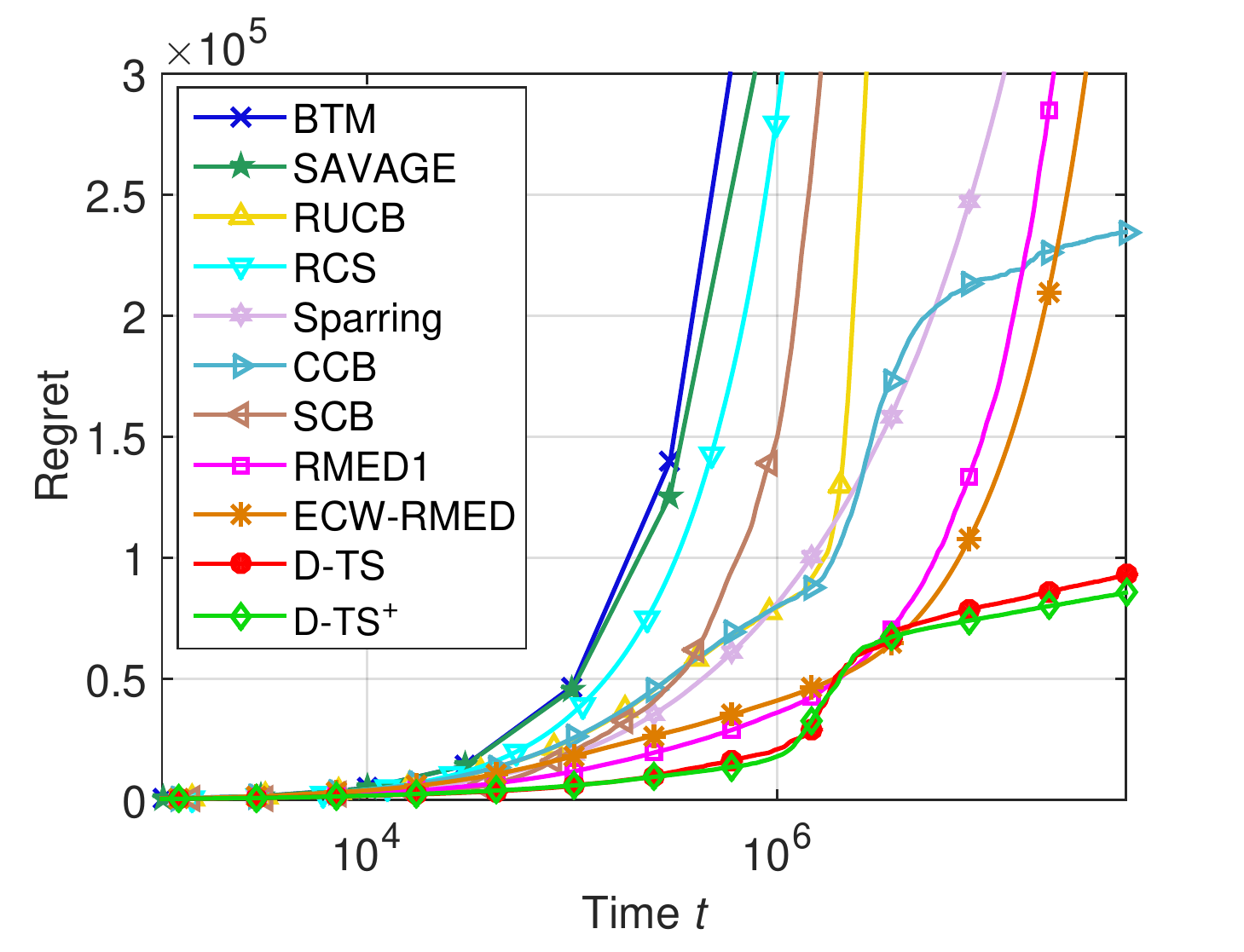}
\label{fig:MSLR_Informational_32_non_Condorcet}}
\caption{Cumulative regret in non-Condorcet dueling bandits.}
\label{fig:regret_noncondorcet}
\end{center}
\end{figure}

Compared with the recently developed algorithm ECW-RMED, D-TS and D-TS+ usually converge faster to the asymptotic regime and thus can achieve smaller regret for small to relatively large $T$. In theory, the optimal CW-RMED achieves the best asymptotic performance. ECW-RMED is its efficient but approximate implementation, with a larger coefficient in terms of asymptotic performance. Therefore, in the asymptotic regime, ECW-RMED could outperform D-TS$^+$ in certain scenarios. In practice, however, we notice that D-TS$^+$ could perform better than ECW-RMED, even for relatively large $T$, e.g., $T= 5 \times 10^6$.
This is because ECW-RMED estimates the required number of comparisons based on the empirical preference probability. Then ECW-RMED may temporally trap in suboptimal comparisons at the beginning stage when the empirical preference probability is likely to deviate from its true value.
In contrast, D-TS and D-TS$^+$ make decisions in a random manner, and the winner(s) has a positive probability to be explored even when the empirical estimates deviate from the true values.
Thus, D-TS and D-TS$^+$ usually converge to the asymptotic regime more quickly than ECW-RMED. Because of this, D-TS and D-TS$^+$ may outperform ECW-RMED even for a relatively large $T$, especially when the number of arms is larger. For example, as shown in Fig.~\ref{fig:MSLR_Informational_16_non_Condorcet}, ECW-RMED performs worse than D-TS$^+$ when $t \leq 5\times10^6$, although it has better asymptotic performance. This situation becomes more serious when the number of arms increases, as we can see from Fig.~\ref{fig:MSLR_Informational_32_non_Condorcet}.

Our D-TS algorithm performs worse than ECW-RMED when there are multiple winners with similar performance, e.g., for the non-Condorcet Cyclic dataset shown in Fig.~\ref{fig:Synthetic_nonCondorcet_Cyclic_9}. One main reason is that with a random tie-breaking rule, D-TS randomly explores all potential winners in each individual sample path \footnote{Another reason is that D-TS explores the superiors for each arm sequentially and results in a regret higher than the lower bound in \cite{Komiyama2016ICML:CWRMED}.}, as shown in Fig.~\ref{fig:diversity_noncondorcet}. Thus, the regret of D-TS scales with the number of winners $|\mathcal{C}^*|$. By carefully breaking the ties, D-TS$^+$ can reduce the regret in many practical scenarios, as shown in Fig.~\ref{fig:MSLR_Informational_5_non_Condorcet}, Fig.~\ref{fig:MSLR_Informational_16_non_Condorcet}, and Fig.~\ref{fig:MSLR_Informational_32_non_Condorcet}. However, the improvement is limited especially when the winners have very similar or exactly the same performance, as shown in Fig.~\ref{fig:Synthetic_nonCondorcet_Cyclic_9}. From the perspective of regret optimization, this may be a disadvantage of TS, where based on randomly sampled belief, all winners have a positive probability to be explored \cite{Gopalan2014ICML:TS}. However, from the diversity perspective, this may be desirable in the application scenarios such as restaurant recommendation, where users may not want to stick to a single winner.

\begin{figure}[thbp]
\begin{center}
\subfigure[Non-Condorcet Cyclic]{\includegraphics[angle = 0,width = 0.49\linewidth]{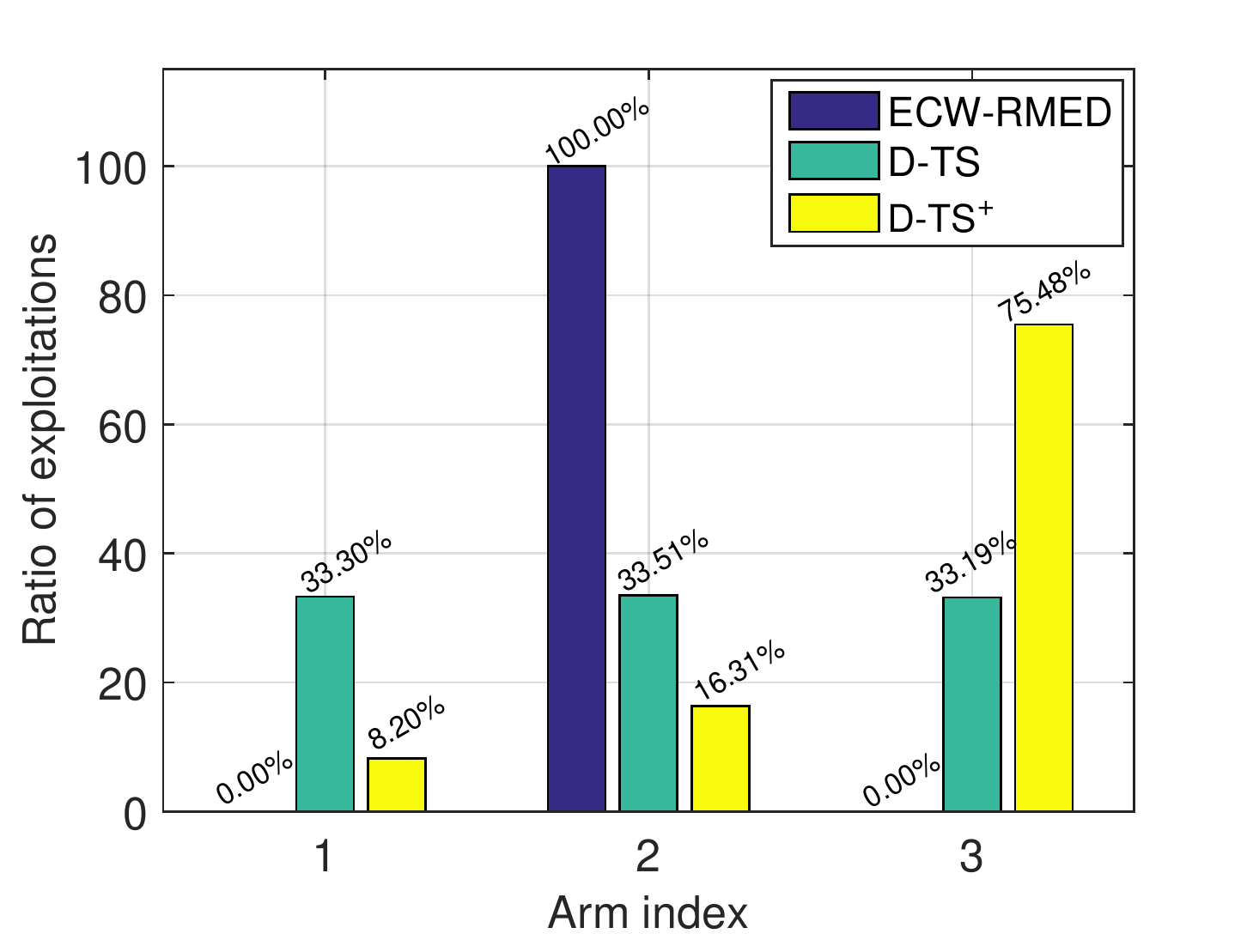}
\label{fig:diversity_Synthetic_nonCondorcet_Cyclic_9}}
\subfigure[MSLR ($K = 16$, non-Condorcet)]{\includegraphics[angle = 0,width = 0.49\linewidth]{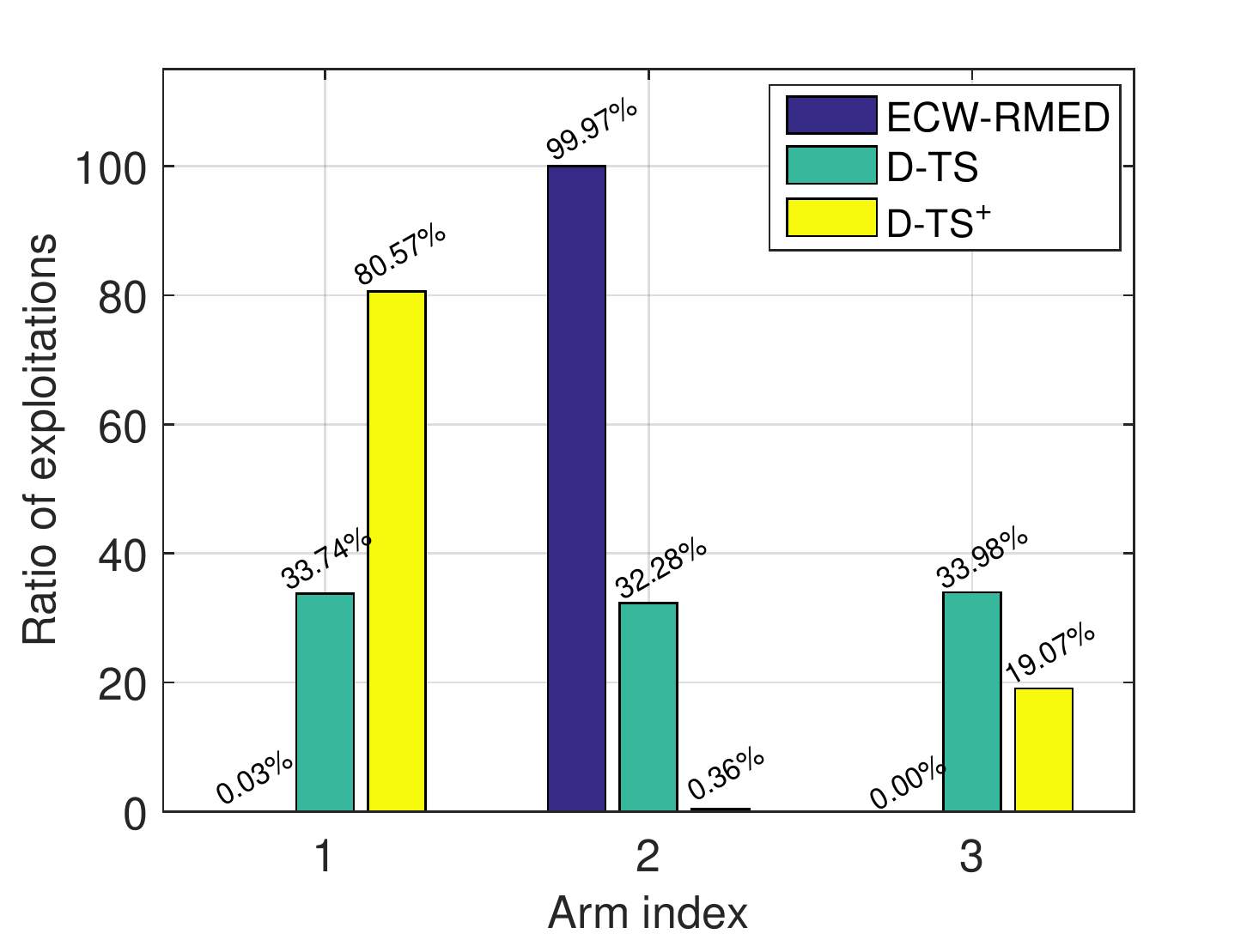}
\label{fig:diversity_MSLR_Informational_16_non_Condorcet}}
\caption{Diversity of exploitations: distribution of exploitations over different Copeland winners in one sample path. This is calculated when an arm is compared against itself. }
\label{fig:diversity_noncondorcet}
\end{center}
\end{figure}

\subsubsection{Robustness} \label{app:robustness}

We study the robustness of ECW-RMED, D-TS, and D-TS$^+$, with respect to the preference matrix and delayed feedback.

\textbf{Influence of preference matrix:} We have seen from Fig.~\ref{fig:regret_deviation} that, when some preference probabilities for different arms are close to 1/2 (5-armed non-Condorcet MSLR dataset), the regret of ECW-RMED fluctuates significantly and has a very large standard deviation, while D-TS and D-TS$^+$ have much smaller regret deviation.

The robustness of D-TS and D-TS$^+$ to the preference matrix can also be seen from the results
in the Gap dataset. Recall that in this dataset, the regret bound of ECW-RMED is much larger than the optimal lower bound. By experiments, we find that the performance of ECW-RMED is significantly affected by the order of arms (and essentially the preference matrix). For example, when the order of arms is fixed as that in Table~\ref{tab:synthetic_gap}, ECW-RMED performs very well. Specifically, as shown by the dashed line in Fig.~\ref{fig:Synthetic_Gap}, the regret of ECW-RMED is very small and its asymptotic behavior is even better than its asymptotic bound (similar results can be found in \cite{Komiyama2016ICML:CWRMED}). However, when the arms are randomly shuffled in each experiment, ECW-RMED achieves much larger regret that is consistent with its asymptotic bound.
In contrast, D-TS and D-TS$^+$ do not depend on the order of arms and perform much better than ECW-RMED on average in this dataset.\\

\textbf{Influence of delayed feedback:} In practice, it may be difficult and costly to process each individual comparison result immediately. Typically, feedback will be batched and provided periodically, say every $d$ time-slots. We evaluate the algorithm performance with respect to the feedback delay. As we can see from Fig.~\ref{fig:regret_vs_delay}, as the feedback delay increases, the regret of ECW-RMED increases much faster than D-TS and D-TS$^+$. In particular, even in the non-Condorcet Cyclic dataset with multiple winners (Fig.~\ref{fig:regret_vs_delay_Synthetic_nonCondorcet_Cyclic}), the regret of ECW-RMED becomes larger than D-TS and D-TS$^+$ when the feedback delay is larger than about 300 time-slots.

\begin{figure}[thbp]
\begin{center}
\subfigure[MSLR ($K = 5$, Condorcet)]{\includegraphics[angle = 0,width = 0.49\linewidth]{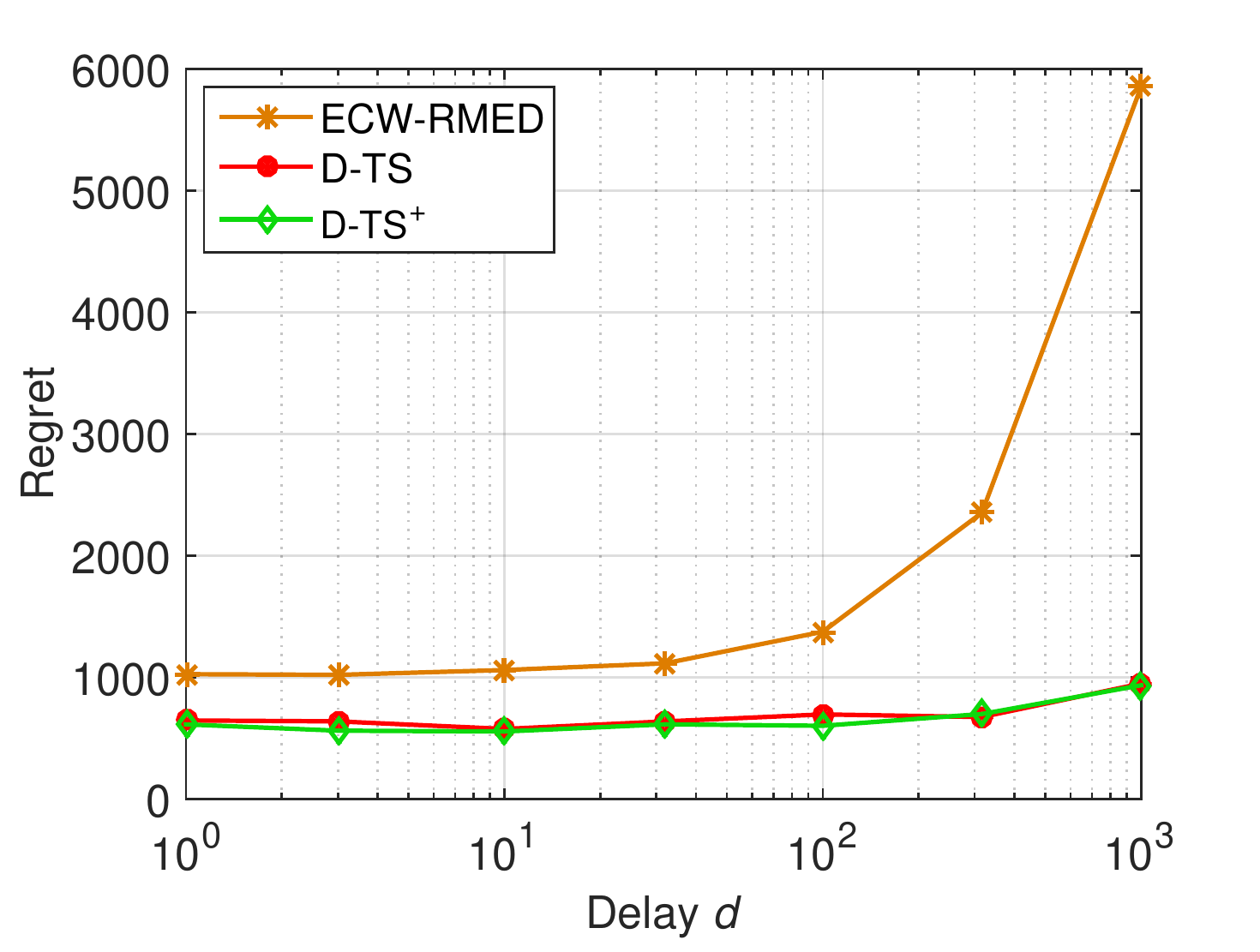}
\label{fig:regret_vs_delay_MSLR_Informational_5_Condorcet}}
\subfigure[Non-Condorcet Cyclic]{\includegraphics[angle = 0,width = 0.49\linewidth]{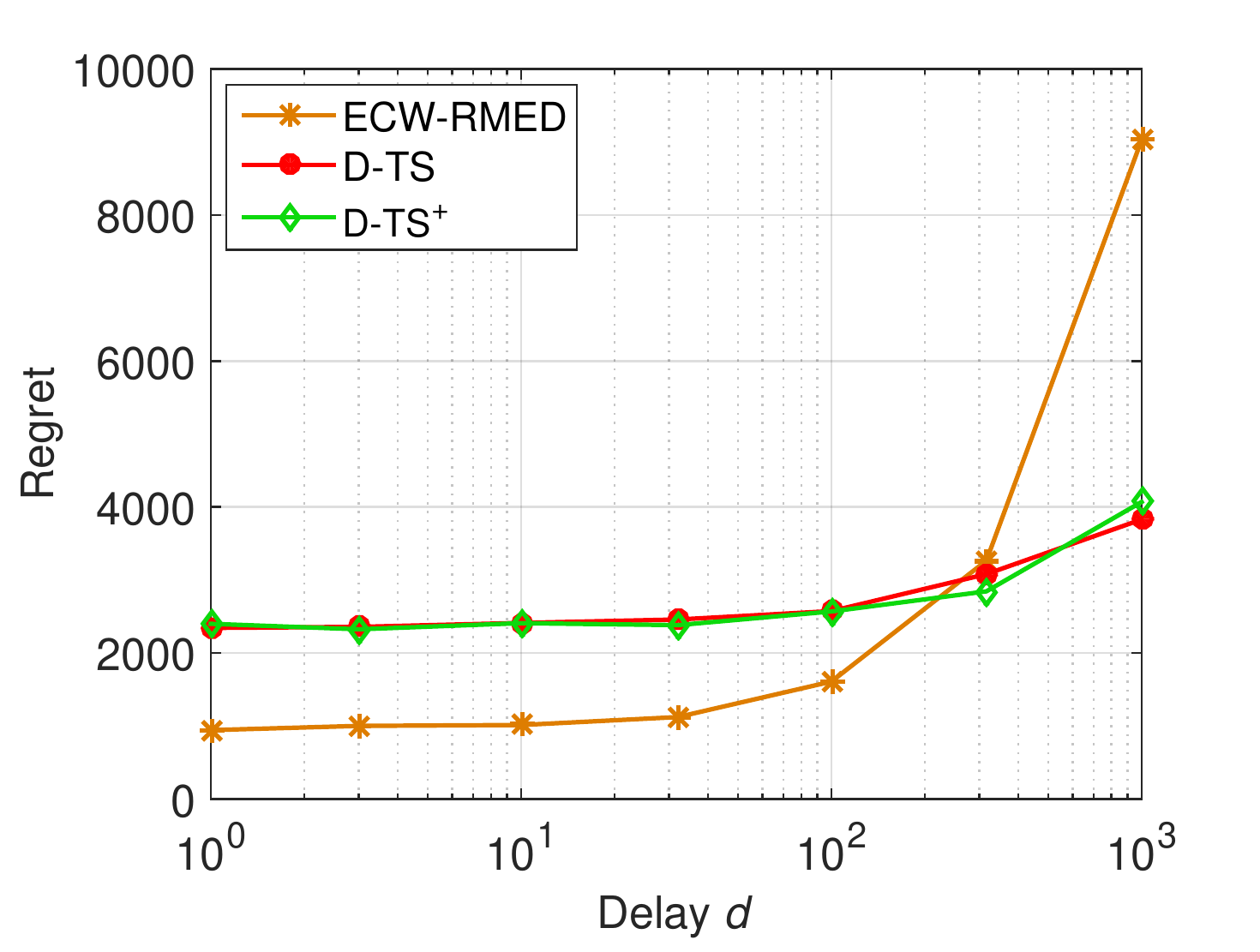}
\label{fig:regret_vs_delay_Synthetic_nonCondorcet_Cyclic}}
\caption{Influence of feedback delay: the regret when the feedback is batched and provided every $d$ time-slots ($T = 10^6$).  }
\label{fig:regret_vs_delay}
\end{center}
\end{figure}

\subsubsection{Impact of RUCB/RLCB Elimination} \label{app:rcb_elimination}
In this section, we illustrate the necessity of RUCB/RLCB elimination in D-TS and D-TS$^+$. A pure Thompson Sampling algorithms without the auxiliary RUCB/RLCB elimination step seem to be more elegant and may be more efficient (without the limitation of RUCB/RLCB). However, we notify that the RUCB/RLCB elimination is necessary to guarantee sublinear or logarithmic regret in general settings, especially in non-Condorcet dueling bandits.

Specifically, we consider the following ``pure D-TS'' algorithm, which is similar to D-TS in Algorithm~\ref{alg:dts_Copeland}, except that the RUCB/RCLB elimination step is ignored, i.e., the candidates $a^{(1)}$ and $a^{(2)}$ are selected from all arms according to the samples $\theta_{ij}^{(1)}$ and $\theta_{i a^{(1)}}^{(2)}$. As shown in Fig~\ref{fig:regret_pureDTS}, pure D-TS may result in large regret in certain scenarios. For example, in the (Condorcet) StrongBorda dataset (Table~\ref{tab:synthetic_strongborda}), the Borda winner (Arm 2)  beats all the other arms with high probability except for the Condorcet winner (Arm 1), and hence, it is easier for the Borda winner to get more votes at the beginning and to be chosen as the first candidate. Thus, pure D-TS achieves higher regret in this case as shown in Fig~\ref{fig:regret_Synthetic_StrongBorda_pureDTS}. In non-Condorcet dueling bandits, without RLCB elimination, pure D-TS could achieve linear regret  if a Copeland winner is beaten by a non-winner arm. For example, as shown in Fig.~\ref{fig:regret_Synthetic_nonCondorcet_Cyclic_9_pureDTS}, the algorithm fails to converge to comparing the Copeland winners against themselves, as the non-winner arm will have higher samples with  high probability at the second round. By introducing RLCB elimination, D-TS/D-TS$^+$ can avoid trapping in theses suboptimal comparisons and achieve much better performance.

We also point out without the limitation of RUCB/RLCB, pure D-TS may achieve in certain practical scenarios, e.g., the 5-arm MSLR data sets as shown in Figs~\ref{fig:regret_MSLR_Informational_5_Condorcet_pureDTS} and \ref{fig:regret_MSLR_Informational_5_non_Condorcet_pureDTS}. In fact, the RLCB elimination can be ignored in Condorcet dueling bandits, because none of the other arms can beat the Condorcet winner. It is an interesting direction to identify the conditions under which pure D-TS can perform better and obtain its theoretical performance under these conditions.

\begin{figure}[thbp]
\begin{center}
\subfigure[Condorcet StrongBorda]{\includegraphics[angle = 0,width = 0.49\linewidth]{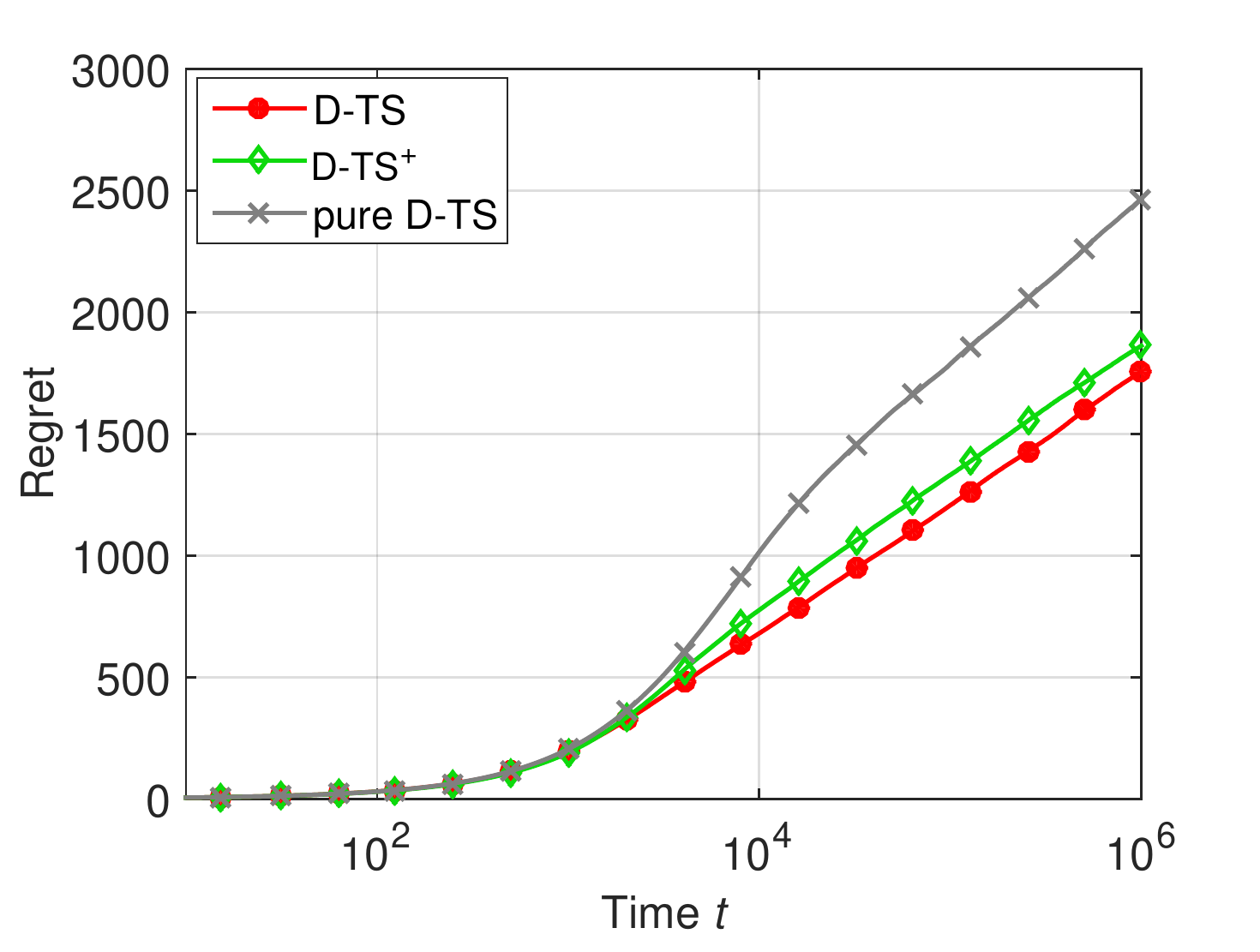}
\label{fig:regret_Synthetic_StrongBorda_pureDTS}}
\subfigure[Non-Condorcet Cyclic]{\includegraphics[angle = 0,width = 0.49\linewidth]{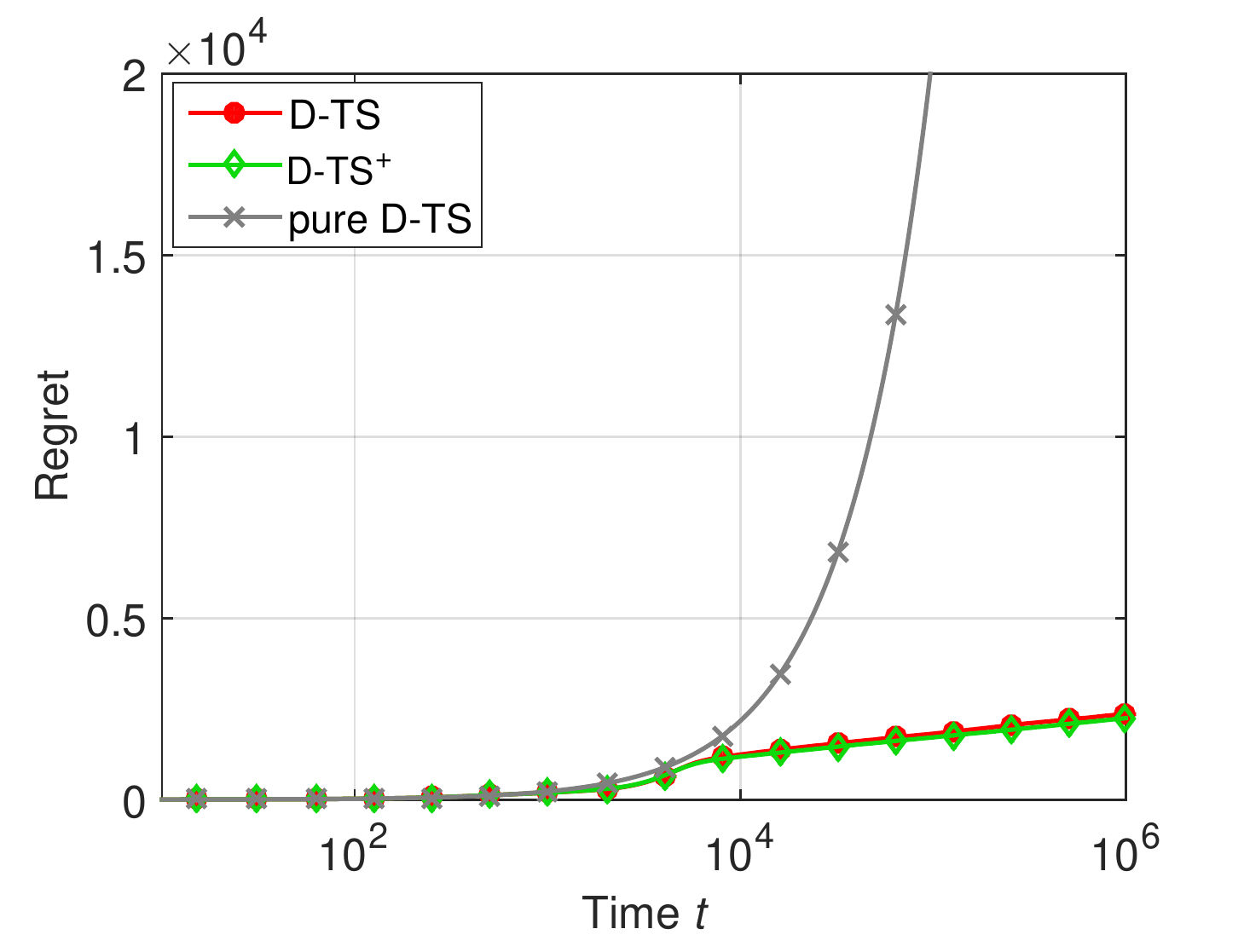}
\label{fig:regret_Synthetic_nonCondorcet_Cyclic_9_pureDTS}}
\subfigure[MSLR ($K = 5$, Condorcet)]{\includegraphics[angle = 0,width = 0.49\linewidth]{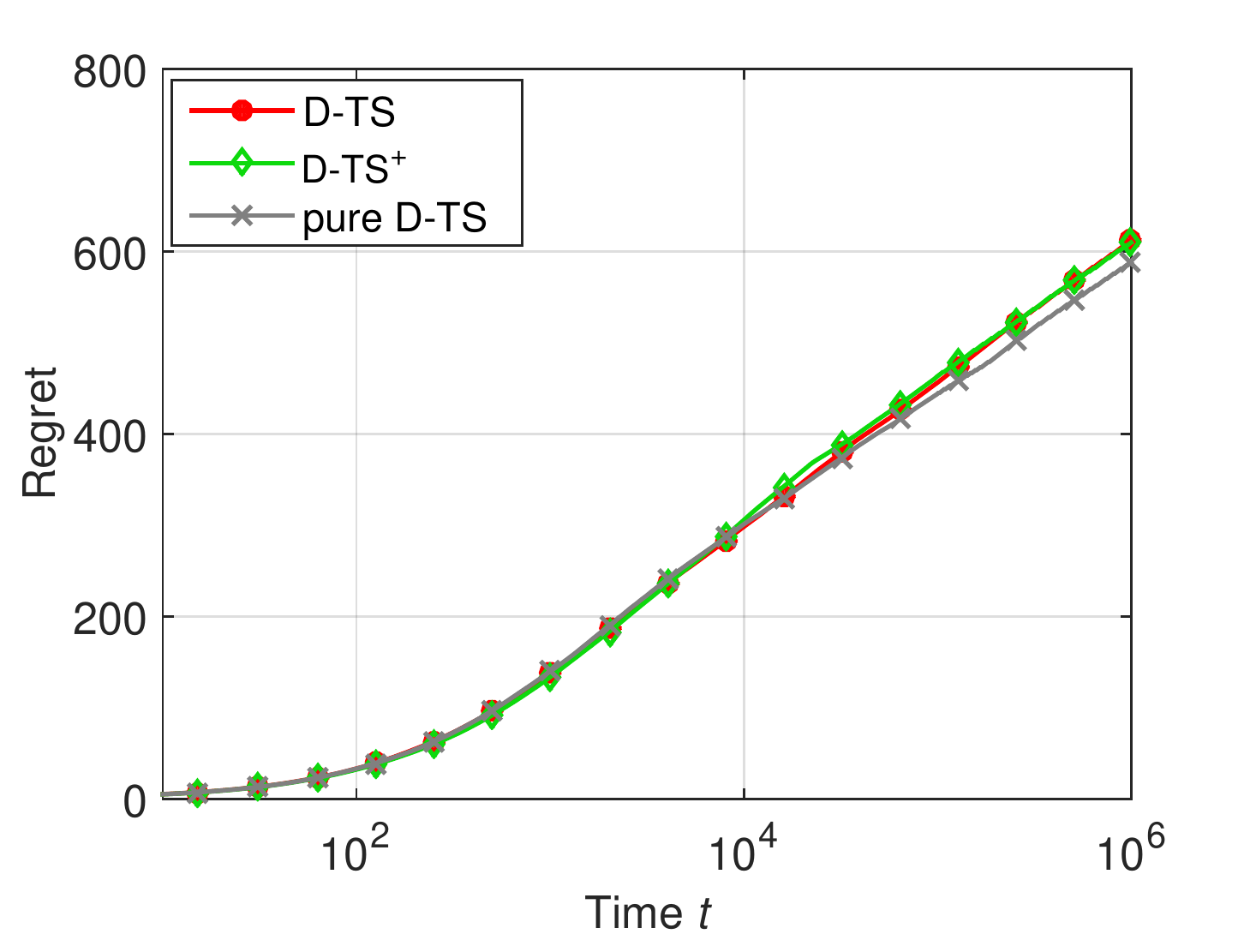}
\label{fig:regret_MSLR_Informational_5_Condorcet_pureDTS}}
\subfigure[MSLR ($K = 5$, non-Condorcet)]{\includegraphics[angle = 0,width = 0.49\linewidth]{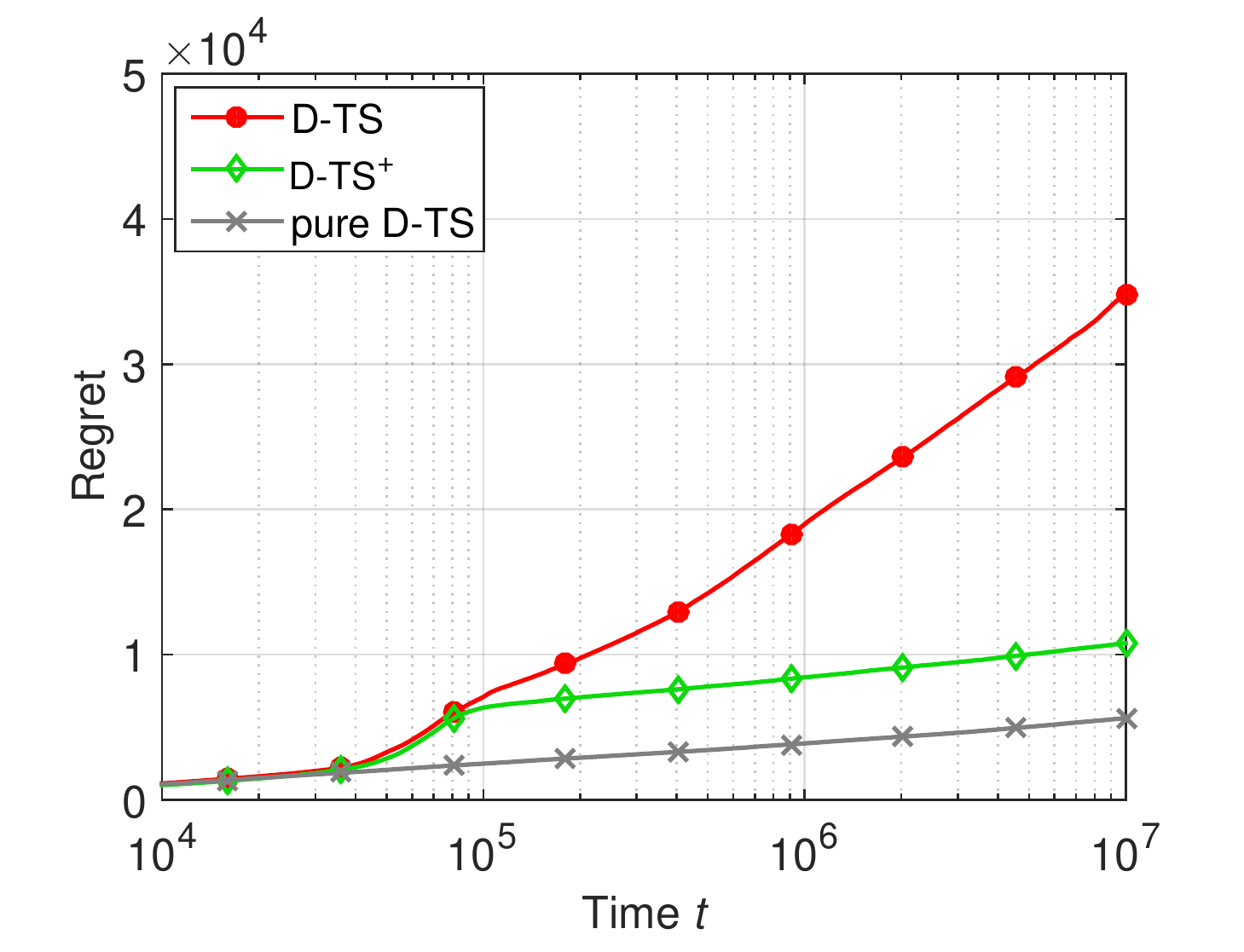}
\label{fig:regret_MSLR_Informational_5_non_Condorcet_pureDTS}}
\caption{Impact of RUCB/RLCB elimination: ``pure D-TS'' is similar to D-TS, except that the RUCB/RLCB elimination is ignored and the candidates are selected among all arms according to the Thompson samples only.  }
\label{fig:regret_pureDTS}
\end{center}
\end{figure}

\subsubsection{Summary of Experimental Results}
We summarize the performance evaluation results based on synthetic and real-world data:
\begin{itemize}
\item In Condorcet dueling bandits, D-TS and D-TS$^+$ achieve similar performance, and both perform much better than existing algorithms that work for unknown/infinite time-horizon settings.  This benefits from the double sampling structure that we proposed for D-TS and D-TS$^+$.

\item In non-Condorcet dueling bandits, D-TS and D-TS$^+$ performs much better than UCB-type algorithms, CCB and SCB. Again, this benefits from the double sampling structure of D-TS and D-TS$^+$, assisted by RUCB/RLCB elimination that guarantees sublinear or logarithmic regret in general settings.

\item Compared with ECW-RMED, D-TS and D-TS$^+$ achieve much better performance in Condorcet dueling bandits, better or close-to performance in non-Condorcet dueling bandits, especially when $T$ is small to relatively large. Furthermore, D-TS and D-TS$^+$  are also much more robust with respect to the preference matrix and delayed feedback.
\end{itemize}

In practice, we may not  know in advance whether we have a Condorcet and non-Condorcet dueling bandit. We may have in practice a time-varying system and delayed feedback. Overall,  good performance, and robustness of D-TS and D-TS$^+$ make
 them strong candidates in practice.

\end{document}